\documentclass{article}

\usepackage[preprint]{neurips_2026}

\usepackage{amsmath}
\usepackage{amssymb}
\usepackage{mathtools}
\usepackage{amsthm}
\usepackage{hyperref}       

\usepackage{algorithmic,algorithm}

\usepackage{soul}
\usepackage{microtype}
\usepackage{booktabs}
\usepackage{caption}

\usepackage{threeparttable}
\usepackage{dsfont}
\usepackage{enumerate}
\usepackage{amsmath,amsthm,amssymb}
\usepackage{multirow}
\usepackage{longtable}
\usepackage{wrapfig}

\usepackage{framed}
\usepackage{color}
\definecolor{shadecolor}{rgb}{0.92,0.92,0.92}

\definecolor{mygray}{gray}{.9}

\newcommand{\nullhyp}{\mathfrak{H}_0}
\newcommand{\althyp}{\mathfrak{H}_1}

\usepackage[capitalize,noabbrev]{cleveref}

\theoremstyle{plain}
\newtheorem{theorem}{Theorem}[section]
\newtheorem{proposition}[theorem]{Proposition}
\newtheorem{lemma}[theorem]{Lemma}
\newtheorem{corollary}[theorem]{Corollary}
\theoremstyle{definition}
\newtheorem{definition}[theorem]{Definition}
\newtheorem{assumption}[theorem]{Assumption}
\theoremstyle{remark}
\newtheorem{remark}[theorem]{Remark}
\usepackage{xcolor}
\usepackage{colortbl}

\usepackage[utf8]{inputenc} 
\usepackage[T1]{fontenc}    
\usepackage{url}            
\usepackage{booktabs}       
\usepackage{amsfonts}       
\usepackage{nicefrac}       
\usepackage{microtype}      
\usepackage{xcolor}         
\usepackage{booktabs, multirow, multicol, makecell}  
\usepackage{pifont}
\usepackage{enumitem}

\definecolor{darkblue}{rgb}{0, 0, 0.5}
\definecolor{lightblue}{RGB}{220,230,255}
\definecolor{selfblue}{RGB}{65,105,225} 
\definecolor{Gray}{gray}{0.93}

\usepackage{etoc}
\etocdepthtag.toc{mtchapter}
\etocsettagdepth{mtchapter}{subsection}
\etocsettagdepth{mtappendix}{none}

\usepackage{subfigure}
\usepackage{threeparttable}

\hypersetup{colorlinks=true, citecolor=darkblue, linkcolor=darkblue, urlcolor=darkblue}

\title{Micro-Defects Expose Macro-Fakes: Detecting AI-Generated Images via Local Distributional Shifts}

\author{%
  Boxuan Zhang$^{1}$ \quad
  Jianing Zhu$^{2}$ \quad
  Qifan Wang$^{3}$ \quad
  Jiang Liu$^{4}$ \quad
  Ruixiang Tang$^{1}$\thanks{Corresponding author.}\\[2pt]
  $^{1}$Rutgers University \quad
  $^{2}$The University of Texas at Austin \quad
  $^{3}$Meta AI\quad
  $^{4}$Advanced Micro Devices\\[2pt]
  \texttt{\{bz362, rt836\}@scarletmail.rutgers.edu}\\[4pt]
}

\begin{document}

\maketitle
\vspace{-12mm}
\begin{center}
    \textbf{Project Page:} \url{https://zbox1005.github.io/MDMF-project/}
\end{center}
\vspace{2mm}

\begin{abstract}
  Recent generative models can produce images that appear highly realistic, raising challenges in distinguishing real and AI-generated images.
  Yet existing detectors based on pre-trained feature extractors tend to over-rely on global semantics, limiting sensitivity to the critical micro-defects.
  In this work, we propose \emph{Micro-Defects expose Macro-Fakes} (MDMF), a local distribution-aware detection framework that amplifies micro-scale statistical irregularities into macro-level distributional discrepancies. To avoid localized forensic cues being diluted by plain aggregation, we introduce a learnable \emph{Patch Forensic Signature} that projects semantic patch embeddings into a compact forensic latent space. 
  We then use \emph{Maximum Mean Discrepancy} (MMD) to quantify distributional discrepancies between generated and real images. Our theory-grounded analysis shows that patch-wise modeling yields provably larger discrepancies when localized forensic signals are present in generated images, enabling more reliable separation from real images. Extensive experiments demonstrate that MDMF consistently outperforms baseline detectors across multiple benchmarks, validating its general effectiveness.
\end{abstract}

\section{Introduction}
Deep generative models have made rapid advances in recent years~\citep{ho2020denoising,saharia2022photorealistic,podell2023sdxl,lipman2022flow}, with diffusion-based architectures enabling the synthesis of highly realistic images from natural language descriptions.
Such advances now power widely used platforms, including Stable Diffusion~\citep{rombach2022high}, DALL·E~\citep{ramesh2022hierarchical}, and Midjourney.
While this progress has accelerated creative high-quality content generation, it also raises significant concerns regarding misinformation~\citep{zhou2023synthetic}, deepfakes~\citep{heidari2024deepfake}, and digital forgery~\citep{somepalli2023diffusion}.
As modern generative models continue to improve in visual fidelity, reliably distinguishing AI-generated images from natural images becomes increasingly challenging and essential, motivating increasing interest in AI-generated image detection~\citep{zhu2023genimage,chen2024drct}.

\vspace{2mm}
Previous studies have achieved promising progress by exploiting artifacts left by generative processes \citep{wang2023dire,chen2024drct,ojha2023towards,zhang2025detecting}. 
Most approaches adopt an image-level paradigm and treat detection as global classification, either learning discriminative features with supervision~\citep{chen2024drct,liu2024forgery} or measuring deviations in frozen representation spaces \citep{ojha2023towards,he2024rigid}. 
However, as modern diffusion models increasingly leave \emph{sparse} and \emph{localized} forensic traces \citep{wang2024detecting,wang2025diffdoctor}, detectors built upon pre-trained representations can over-rely on global semantics, which reduces sensitivity to the micro-scale defects that are most diagnostic of generation. 
Several recent works have explored patch modeling to capture finer-grained cues \citep{zhong2023patchcraft,liu2024forgery,choi2025training}.
Nevertheless, when localized evidence is still summarized by plain aggregation, subtle forensic cues can remain diluted and the decision may continue to be driven by semantics rather than generation-induced irregularities. 
This naturally motivates a fundamental research question: 

\begin{quote}
    \begin{center}
    \textit{Can we learn representations that amplify micro-scale statistical irregularities into stable macro-level distributional discrepancies for AI-generated image detection?}    
    \end{center}
\end{quote}

\begin{wrapfigure}{r}{0.60\linewidth}
\centering
\vspace{-12pt}
\includegraphics[width=\linewidth]{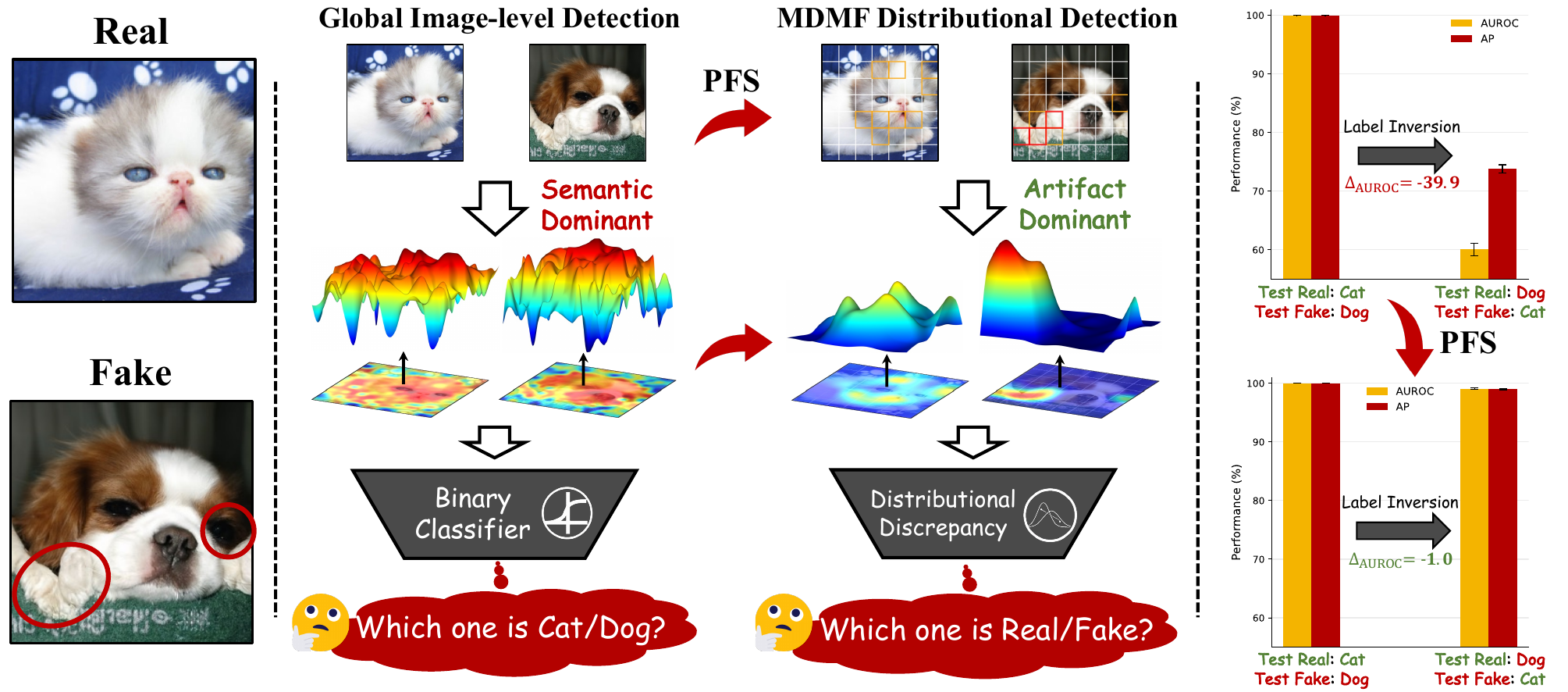}
  \caption{\small{
  \textbf{Intuition behind \emph{Patch Forensic Signature} (PFS).} 
  \textbf{\emph{Left:}} a real cat and a generated dog with plausible localized irregularities (highlighted). 
  \textbf{\emph{Middle:}} global image-level detection aggregates a \textbf{semantic-dominant} representation, inadvertently reducing real/fake detection into semantic recognition (e.g., “cat vs.\ dog”).
  PFS maps patch-wise representations into an \textbf{artifact-dominant} forensic space, making subtle generation-induced statistical deviations more salient. MDMF thus leverages their distributional discrepancy to answer “real vs.\ fake”.
  \textbf{\emph{Right:}} 
  Label inversion stress test. Global detection suffers a sharp performance drop under inverted labels, whereas PFS remains stable, indicating PFS shifting the decision from semantics to artifacts. (see Section~\ref{sec:motivation})
}}
\label{fig:motivation}
\vspace{-8pt}
\end{wrapfigure}

In this paper, we propose a distributional detection perspective grounded in localized forensic evidence.
Concretely, instead of representing an image with a single global feature vector, we decompose it into local regions and analyze the statistics of their features.
This perspective is well matched to modern generators, whose artifacts often manifest as \emph{localized} statistical shifts that 
are easily suppressed by uniform aggregation into global representations.
To operationalize this idea, we introduce the \textbf{\emph{Patch Forensic Signature} (PFS)}, a learnable patch-level representation tailored for forensic analysis. 
PFS reparameterizes semantic patch embeddings into a dedicated forensic space that deemphasizes semantic content while preserving and amplifying subtle statistical irregularities introduced by the generative process (as illustrated in Figure~\ref{fig:motivation} and discussed in Section~\ref{sec:motivation}).
\vspace{2mm}

Based on the Patch Forensic Signature, we propose \textbf{\emph{Micro-Defects expose Macro-Fakes} (MDMF)}, a distributional detection framework that transforms sparse, localized forensic artifacts into reliable image-level signals.
Specifically, MDMF employs \emph{Maximum Mean Discrepancy} (MMD) \cite{gretton2012kernel,liu2020learning} to quantify distributional discrepancy between patch-level PFS representations of test images and those of reference real images (see Section~\ref{sec:method}).
The theoretical analysis proves that patch-wise PFS modeling provably amplifies localized defects compared to global aggregation, while the resulting empirical MMD exhibits a positive separation between real and generated images under finite samples (see Section~\ref{sec:theory}). 
This analysis provides a principled explanation for why aggregating localized evidence at the distribution level leads to reliable separation, even when individual artifacts are weak.
\vspace{.3mm}


We conduct extensive experiments to evaluate the effectiveness and generalization of MDMF. Our evaluation covers widely used benchmarks, including ImageNet \cite{deng2009imagenet}, LSUN-Bedroom \cite{yu2015lsun}, GenImage \cite{zhu2023genimage}, the in-the-wild WildRF \cite{cavia2024wildrf}, and the recent LDMFakeDetect \cite{rajan2025staypositive}.
Across them, MDMF consistently achieves strong and stable detection performance, demonstrating robustness to diverse generative architectures and training paradigms. To further stress-test the method, we conduct case studies on OpenSora-generated videos \cite{zheng2024open}, where many existing detectors degrade substantially while MDMF still identifies stable forensic signals.
We summarize our contributions as follows:

\begin{itemize}[leftmargin=*]
    \item We introduce a new perspective for AI-generated image detection, modeling images as collections of localized visual evidence and revealing that modern generative artifacts manifest as subtle statistical deviations rather than global inconsistencies. (Section~\ref{sec:motivation})

    \item We propose the \textbf{\emph{Patch Forensic Signature} (PFS)}, a learnable forensic representation that reparameterizes semantic embeddings into a latent space designed to suppress semantic invariances while preserving and amplifying generative artifacts. (Section~\ref{sec:method})

    \item We develop \textbf{\emph{Micro-Defects expose Macro-Fakes} (MDMF)}, a distributional detection framework that aggregates localized forensic evidence through MMD, with theoretical analysis establishing provable separation between real and generated images. Experiments across diverse benchmarks show the effectiveness and generalization of MDMF. (Sections~\ref{sec:theory} and~\ref{sec:exp})
\end{itemize}

\section{Micro-Defects Expose Macro-Fakes.}

\noindent\textbf{Preliminary.}
Let $\mathbb{P}$ denote the distribution of real images defined on an image space $\mathcal{X} \subset \mathbb{R}^{H \times W \times C}$, where $H$, $W$, and $C$ denote the image height, width, and number of channels. Given i.i.d.\ samples $S_{\mathbb{P}} = \{x_n\}_{n=1}^N$ drawn from $\mathbb{P}$, the goal of AI-generated image detection is to determine whether a test image $\tilde{y}$ originates from $\mathbb{P}$ or from an alternative distribution $\mathbb{Q}$ introduced by generative models. 

\subsection{Motivation}
Recent advances in generative modeling have substantially reduced perceptually salient artifacts. As a result, discrepancies between real and generated images increasingly appear as sparse, localized deviations rather than global inconsistencies \citep{wang2024detecting,wang2025diffdoctor}. We refer to this regime as \emph{Local Distributional Shifts}.
Most existing approaches adopt an image-level paradigm and cast detection as global classification~\citep{ojha2023towards,chen2024drct,tan2024rethinking}.
However, these global representations are often dominated by semantic content, which can bias real/fake decisions toward semantic correlations rather than the localized forensic deviations that are most diagnostic of the generative process.

We conceptually and empirically analyze this limitation. \textit{Conceptually}, Figure~\ref{fig:overview}(a) provides a mechanistic view where semantic content and generation artifacts jointly contribute to an image. Global detectors typically compress the image into a single representation before predicting real/fake, which is often shaped primarily by semantics. As a result, the detector is biased toward semantic correlations rather than the forensic evidence for real/fake detection.
\textit{Empirically}, we validate this semantic bias using a label inversion toy experiment in Figure~\ref{fig:motivation}. We train a global image-level real/fake classifier on a confounded split with \emph{real cats} and \emph{generated dogs}, and evaluate it on the inverted split with \emph{real dogs} and \emph{generated cats}. The global classifier exhibits a sharp performance drop under label inversion, indicating its heavy reliance on semantic cues instead of artifact evidence.

To mitigate the semantic dominance, we seek a representation that weakens the influence of global semantics while retaining artifact-related cues. 
A natural step is to decompose an image into local patches and operate on the resulting patch representations. 
As illustrated in Figure~\ref{fig:overview}(b), the patch-wise formulation avoids collapsing the image into a globally pooled feature, which weakens the semantic shortcut that can confound real/fake prediction under global aggregation.
However, generative artifact patterns are diverse and difficult to model explicitly, and patch embeddings from standard visual backbones are still influenced by semantics.
This motivates us to learn a patch-wise representation that suppresses semantic dominance while preserving statistical deviations from the generation.

\subsection{The Patch Forensic Signature}
\label{sec:motivation}
We introduce the \textbf{\emph{Patch Forensic Signature} (PFS)}, a learnable representation that reparameterizes semantic patch embeddings into a dedicated forensic space.
At a high level, PFS suppresses semantic variation and accentuates generation-induced statistical deviations, yielding signatures that align more closely with artifact-driven evidence.
We next formalize PFS by first defining the extracted patch signature field and then specifying the learnable projection.

\textbf{Patch Signature Field.}
Let $x \in \mathbb{R}^{H \times W \times C}$ be an input image. We leverage a pre-trained self-supervised vision backbone (e.g., DINOv2~\citep{oquab2024dinov}) to decompose the image into a grid of $K$ non-overlapping patch tokens:
\begin{equation}
\mathbf{E}(x) = \{\mathbf{e}_{i}(x) \in \mathbb{R}^{D}\}_{i=1}^{K},
\end{equation}
where $D$ is the embedding dimension. 
While patch-wise modeling weakens semantic shortcuts under global aggregation, $\mathbf{e}_i(x)$ remains largely semantics-oriented, and thus generative statistical cues are still not salient in this space.
We then introduce a learnable reparameterization into the \emph{forensic space}, defined as a compact latent space where semantic variation is deemphasized while patch-wise statistical deviations become more separable under the detection objective.
\begin{definition}(\textbf{Patch Forensic Signature (PFS)}.)
\label{def:patch_signature}
\textit{Given a patch embedding $\mathbf{e}_i(x)$, we define a learnable projection function ${\phi}_{\theta}: \mathbb{R}^{D} \rightarrow \mathbb{R}^{d}$, parameterized by a lightweight Multilayer Perceptron (MLP), to map semantic embeddings into a compact forensic latent space. We refer to the mapped representation as the Patch Forensic Signature (PFS):}
\begin{equation}\label{eqn:patch_signature}
\mathbf{z}_i(x)={\phi}_{\theta}(\mathbf{e}_i(x)) \in \mathbb{R}^{d}.
\end{equation}
\end{definition}
Consequently, the image $x$ is represented by a set of signature vectors
$\mathbf{Z}_\theta(x) = [\mathbf{z}_1(x), \ldots, \mathbf{z}_K(x)]^\top \in \mathbb{R}^{K \times d}$.
Our later experiments and analysis will show that, under a suitable learning objective
(e.g., Eq.~\ref{eq:objective}), this mapping plays a central role by learning to
reparameterize patch-level representations into a dedicated forensic space that
deemphasizes semantic variation while preserving and amplifying subtle
statistical irregularities introduced by the generative process.

\begin{figure*}[t!]
  \centering
  \includegraphics[width=0.98\linewidth]{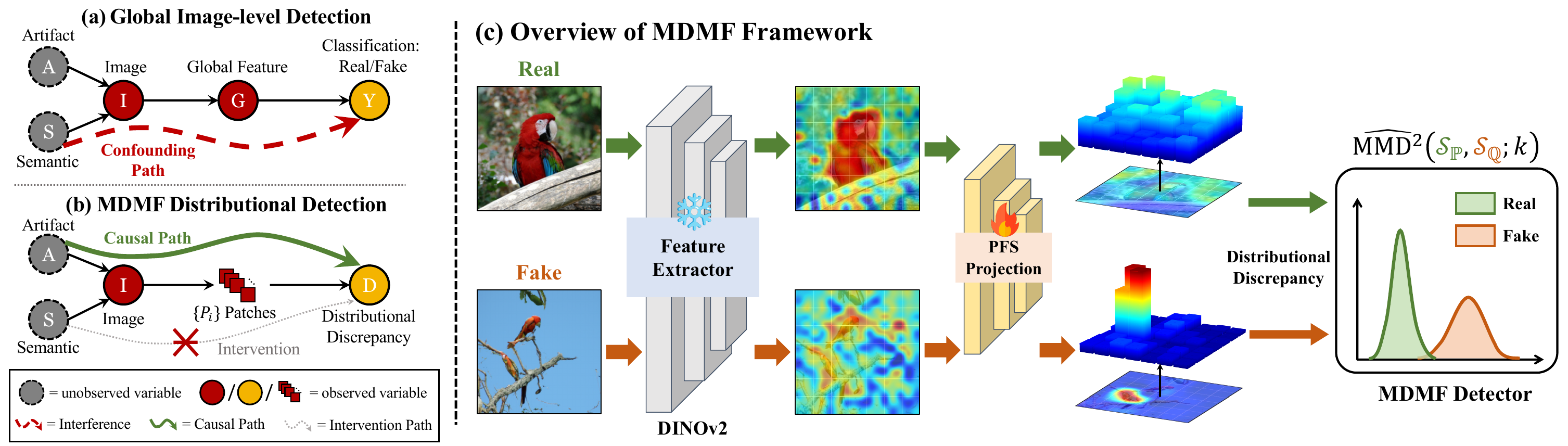}
  \caption{\small{
  \textbf{Motivation and Overview of the MDMF framework.} 
\textbf{\emph{(a)}} Global image-level detection compresses an image into a single feature for real/fake classification, where semantic factors can dominate the decision through a confounding path.
\textbf{\emph{(b)}} MDMF instead operates on patches and bases its prediction on distributional discrepancy, suppressing semantic interference and aligning the decision with artifact-related signals.
\textbf{\emph{(c)}} Given real and generated images, a frozen DINOv2 extracts patch representations, which are mapped by the PFS into a forensic space. MDMF then measures the discrepancy between PFS distributions to produce the final score.
}}
  \label{fig:overview}
  \vspace{-3.5mm}
\end{figure*}

\subsection{Exploring PFS for Detecting AI-Generated Images}
\label{sec:method}

PFS provides patch-wise signatures that emphasize artifact-related statistical cues, yet the resulting evidence remains spatially sparse even in the PFS space. A plain image-level pooling over PFS signatures can still average out these localized cues, making reliable detection difficult for highly realistic generations. 
This motivates a distributional perspective, where we compare the \emph{distributions} of patch signatures between real and generated images to emphasize subtle statistical irregularities.
To operationalize this idea, we adopt the kernel two-sample testing framework via \emph{Maximum Mean Discrepancy} (MMD) \citep{gretton2012kernel}. 
MMD quantifies distributional discrepancy through kernel mean embeddings in a 
\emph{reproducing kernel Hilbert space} 
(RKHS), where small but systematic deviations across local observations can accumulate into a stable image-level signal \citep{liu2020learning}. 
Building on PFS and MMD, we establish the \textbf{\emph{Micro-Defects expose Macro-Fakes} (MDMF)} framework, which transforms sparse patch-level forensic cues into reliable detection scores, as shown in Figure~\ref{fig:overview} (c).

\textbf{MMD Formulation.}
Consider two arbitrary sets of images $\mathcal{S}_{\mathbb{P}}=\{x_{n}\}_{n=1}^{N} \sim \mathbb{P}$ and $\mathcal{S}_{\mathbb{Q}}=\{y_{m}\}_{m=1}^{N} \sim \mathbb{Q}$.
To measure the distance between distributions $\mathbb{P}$ and $\mathbb{Q}$, we employ an unbiased U-statistic estimator for the squared MMD,
\begin{gather}
\label{eq:MMDempirical}
\widehat{\text{MMD}}^2_u(\mathcal{S}_{\mathbb{P}}, \mathcal{S}_{\mathbb{Q}}; k)
:= \frac{1}{N (N-1)} \sum_{i \ne j} H_{ij}
,\end{gather}

where $k$ denotes the kernel of a RKHS and $H_{ij} := k(x_i, x_j)+ k(y_i, y_j)- k(x_i, y_j)- k(y_i, x_j)$. The similar $\widehat{\text{MMD}}^2_b := \frac{1}{N^2} \sum_{ij} H_{ij}$ is the squared MMD
between the empirical distributions of $\mathcal{S}_{\mathbb{P}}$ and $\mathcal{S}_{\mathbb{Q}}$ \citep{liu2020learning}. According to the null hypothesis testing framework \citep{gretton2012kernel},
under the null hypothesis $\nullhyp: \mathbb{P} = \mathbb{Q}$, $\widehat{\text{MMD}}^2_u(\cdot)$ should be close to zero, while strictly positive under the alternative hypothesis $\althyp: \mathbb{P} \neq \mathbb{Q}$. Leveraging this, we design the following optimization and detection protocols.

\textbf{Optimization Protocol.} 
We construct $\mathcal{S}_\mathbb{P}^{tr}$ by aggregating real training images and $\mathcal{S}_\mathbb{Q}^{tr}$ from generated training images. 
We ideally expect to correctly reject $\nullhyp$ and derive $\althyp: \mathbb{P} \neq \mathbb{Q}$, i.e., $\mathcal{S}^{tr}_{\mathbb{P}}$ and $\mathcal{S}^{tr}_{\mathbb{Q}}$ come from different distributions. To enhance discriminative power, we utilize a deep Gaussian kernel \citep{liu2020learning} with bandwidth $\gamma$ for MMD:
\begin{equation}
\label{eq:deep_kernel}
k_{\omega}(x, y) = \exp\left( - \frac{\| \mathbf{Z}_\theta(x) - \mathbf{Z}_\theta(y) \|^2_2}{2\gamma^2} \right),
\end{equation}
Simply maximizing $\widehat{\text{MMD}}^2_u$ can be problematic if the variance of the statistic also increases, leading to unstable gradients. Following the test power maximization principle \citep{gretton2012kernel}, we optimize the parameters $\omega=\{\theta, \gamma\}$, namely the projection weights in $\phi_\theta$ and kernel bandwidth, to maximize the \textbf{regularized test power criterion} with variance $\hat{\sigma}_{H_1}^2 = \frac{4}{N^3} \sum_{i=1}^{N} \left( \sum_{j=1}^{N} H_{ij} \right)^2 - \frac{4}{N^4} \left( \sum_{i=1}^{N} \sum_{j=1}^{N} H_{ij} \right)^2$,
\begin{equation}
\label{eq:objective}
\max_{\omega} \; J_{\lambda}(\mathcal{S}_{\mathbb{P}}^{tr}, \mathcal{S}_{\mathbb{Q}}^{tr}; k_\omega) 
= \frac{\widehat{\text{MMD}}^2_u(\mathcal{S}_{\mathbb{P}}^{tr}, \mathcal{S}_{\mathbb{Q}}^{tr}; k_\omega)}{\sqrt{\hat{\sigma}_{H_1}^2 + \lambda}},
\end{equation}

\begin{algorithm}[t]
\caption{Training MDMF}
\label{alg:mdmf_train}
\small
\begin{algorithmic}[1]
\STATE \textbf{Input:} Training real images $\mathcal{S}^{tr}_{\mathbb{P}}$, generated images $\mathcal{S}^{tr}_{\mathbb{Q}}$; projection head $\phi_\theta$; deep kernel $k_\omega$; regularization $\lambda$; learning rate $\eta$
\STATE Initialize $\omega \leftarrow \{\theta_0$, $\gamma_0$\}
\FOR{$t = 1,2,\dots,T$}
    \STATE Sample mini-batches $\{x_b\}_{b=1}^{B}\sim\mathcal{S}^{tr}_{\mathbb{P}}$ and $\{y_b\}_{b=1}^{B}\sim\mathcal{S}^{tr}_{\mathbb{Q}}$
    \STATE Form PFS vectors $\mathbf{Z}_\theta(x_b) \leftarrow [\mathbf{z}_1(x_b), \ldots, \mathbf{z}_K(x_b)]^\top$ , $\mathbf{Z}_\theta(y_b) \leftarrow [\mathbf{z}_1(y_b), \ldots, \mathbf{z}_K(y_b)]^\top$ 
    \STATE Compute unbiased MMD 
    $
    M(\omega) \gets \widehat{\mathrm{MMD}}_u^2(\mathcal{S}_{\mathbb{P}}^{tr},\mathcal{S}_{\mathbb{Q}}^{tr};k_\omega)
    $ using Eqn. \ref{eq:MMDempirical} and estimate variance $\hat{\sigma}_{H_1}^2$
    \STATE Optimize test-power objective $J_\lambda(\omega) \gets \frac{M(\omega)}{\sqrt{\hat{\sigma}_{H_1}^2+\lambda}}$ using Eqn. \ref{eq:objective}, $\omega \gets \omega + \eta \nabla_{\mathrm{Adam}} J_\lambda(\omega)$
\ENDFOR
\STATE \textbf{Output:} Trained projection head $\phi_{\theta^*}$ and kernel $k_{\omega^*}$
\end{algorithmic}
\end{algorithm}
\begin{algorithm}[t]
\caption{Detecting Images with MDMF}
\label{alg:mdmf_test}
\small
\begin{algorithmic}[1]
\STATE \textbf{Input:} Reference real images $\mathcal{S}^{re}_{\mathbb{P}}$; test images $\mathcal{S}^{te}$; trained $\phi_{\theta^*}$; kernel $k_{\omega^*}$; threshold $\tau$
\STATE Build reference PFS vector $\mathbf{Z}_\theta(x)$ from $x \sim \mathcal{S}^{re}_{\mathbb{P}}$
\FOR{$\tilde{y}\in\mathcal{S}^{te}$}
    \STATE $\mathbf{Z}_\theta(\tilde{y}) \leftarrow [\mathbf{z}_1(\tilde{y}), \ldots, \mathbf{z}_K(\tilde{y})]^\top$
    \STATE $S_{\text{MDMF}}(\tilde{y}) \gets \widehat{\mathrm{MMD}}_{b}^{2}\!\left(\mathcal{S}^{re}_{\mathbb{P}}, \{\tilde{y}\}; k_{\omega^*}\right)$ using Eqn. \ref{eqn:score}
    \STATE $f(\tilde{y}) \gets \mathbb{I}\!\left(S_{\text{MDMF}}(\tilde{y})>\tau\right)$ using Eqn. \ref{eqn:decision}
\ENDFOR
\STATE \textbf{Output:} Predictions $\{f(\tilde{y})\}$ 
\end{algorithmic}
\end{algorithm}

\textbf{Detection Protocol.} 
With the learned parameters $\omega^*$, we apply MMD with the biased estimator to detect individual test images by quantifying their PFS distributional deviation from a set of reference images, following recent works \citep{zhang2024detecting, zhang2025physics} that
demonstrate MMD's effectiveness in single-sample detection. 
Given a set of reference images $\{x_r\}_{r=1}^R$ and a test $\tilde{y}$, we compute the MDMF score:
\begin{align}
\label{eqn:score}
S_{\text{MDMF}}(\tilde{y}) = \widehat{\text{MMD}}^2_b(\mathcal{S}_{\text{ref}}, \{\tilde{y}\}; k_{\omega^*})  = \frac{1}{R^2}\sum_{r, r'=1}^{R}
\cdot k_{\omega^*}(x^{(r)}, x^{(r')}) 
+ k_{\omega^*}(\tilde{y}, \tilde{y}) 
- \frac{2}{R} \sum_{r=1}^{R} k_{\omega^*}(x^{(r)}, \tilde{y}). 
\end{align}
Hence, we can formalize the detection model $f(\cdot)$ to determine whether a given input $\tilde{y}$ is generated:
\begin{equation}\label{eqn:decision}
f(\tilde{y}) = 
\begin{cases}
\text{Generated}, & \text{if } S_{\text{MDMF}}(\tilde{y}) > \tau, \\
\text{Real}, & \text{otherwise},
\end{cases}
\end{equation}
Algorithm~\ref{alg:mdmf_train} and \ref{alg:mdmf_test} summarize the training and testing pipelines of MDMF. 
While our method performs detection by measuring distributional discrepancies via MMD, its effectiveness fundamentally relies on PFS extracting artifact-sensitive patch-level evidence that is often weakened in global image representations (see theoretical analysis in Section \ref{sec:theory} and detailed empirical analysis in Section \ref{sec:exp}).

\subsection{Theoretical Analysis}
\label{sec:theory}
In this section, we provide theoretical justification for MDMF's detection mechanism. First, we show that PFS amplifies sparse localized deviations that tend to be diluted in global image-level detection (Propositions \ref{prop:pfs_mean_shift_second_order} and \ref{prop:patch_advantage}). Second, we establish that MMD on PFS converts this amplified shift into reliable real/fake separation (Proposition \ref{prop:population_mmd} and Theorem \ref{thm:detection}). We first introduce the assumptions.

\begin{assumption}
\label{ass:gaussian_input}
Real images $\{x_n\}_{n=1}^N$ are i.i.d. sampled from distribution $\mathbb{P}$, and generated images
$\{y_m\}_{m=1}^N$ are i.i.d. sampled from distribution $\mathbb{Q}$.
Given any real or generated image, we extract $K$ non-overlapping patch embeddings
$\{\mathbf e_i\}_{i=1}^K\subset\mathbb R^D$ using a fixed pre-trained encoder (e.g., DINOv2).
Each patch embedding follows a $\sigma_e$-sub-Gaussian distribution \citep{wainwright2019high} in $\mathbb R^D$.
\end{assumption}

\begin{assumption}[\textbf{Sparse Defect Model}]
\label{ass:sign_sparse_defect}
For a generated image $y$, we assume each patch embedding:
\begin{equation}
\mathbf e_i(y) = \mathbf u_i + a_i s_i \boldsymbol{\mu}_{\mathrm{defect}},
\end{equation}
where $\mathbf u_i\sim \mathcal{SG}(\mathbf 0,\sigma_e^2\mathbf I_D)$, 
$a_i\sim \mathrm{Bernoulli}(\rho)$ indicates whether the patch is defective, and
$s_i\in\{+1,-1\}$ is an independent Rademacher variable with $\mathcal P(s_i=+1)=\mathcal P(s_i=-1)=1/2$.
Hence $\mathbb E[\mathbf e_i(y)]=\mathbf 0$ but defective patches elevate second-order energy.
For real images, $\mathbf e_i(x)=\mathbf u_i$.
\end{assumption}
Assumption \ref{ass:gaussian_input} follows common practice in representation analysis works \citep{wang2024embedding,zhang2025physics,zhang2024if}, while Assumption \ref{ass:sign_sparse_defect} aligns with sparse-artifact observations in generated images
~\citep{wang2024detecting,wang2025diffdoctor}. 
Under these assumptions, we then establish PFS amplifies localized defects into a detectable distributional shift.

\begin{proposition}
\label{prop:pfs_mean_shift_second_order}
Assume $\phi_\theta$ is twice differentiable at $\mathbf 0$ with Hessian 
$\mathbf H_\phi(\mathbf 0)\in\mathbb R^{d\times D\times D}$.
Let
$
\Delta_{\mathrm{PFS}}
:=
\mathbb E_{\mathbb Q}\!\left[\phi_\theta(\mathbf e_i(y))\right]
-
\mathbb E_{\mathbb P}\!\left[\phi_\theta(\mathbf e_i(x))\right].
$
Then the leading-order PFS mean shift satisfies
\begin{equation}
\label{eq:pfs_second_order_shift_final}
\Delta_{\mathrm{PFS}}
\;\approx\;
\frac{\rho}{2}\,\mathcal Q(\boldsymbol{\mu}_{\mathrm{defect}}),
\end{equation}
where $\mathcal Q(\boldsymbol{\mu})\in\mathbb R^d$ denotes the Hessian-induced quadratic form of
$\phi_\theta$ evaluated along direction $\boldsymbol{\mu}$, i.e.,
$\big[\mathcal Q(\boldsymbol{\mu})\big]_\ell
=
\boldsymbol{\mu}^\top \nabla^2 \phi_{\theta,\ell}(\mathbf 0)\boldsymbol{\mu}$, for $\ell=1,\ldots,d$.
If $\mathcal Q(\boldsymbol{\mu}_{\mathrm{defect}})\neq 0$,
$\|\Delta_{\mathrm{PFS}}\|_2>0$ for any $\rho>0$.
\end{proposition}

\begin{proposition}
\label{prop:patch_advantage}
Under Assumption~\ref{ass:sign_sparse_defect} and Proposition~\ref{prop:pfs_mean_shift_second_order}, we define the global-pooled leading order shift as $\Delta_{\mathrm{global}} := 
\mathbb E_{\mathbb Q}\!\left[\phi_\theta(\bar{\mathbf e}(y))\right] - \mathbb E_{\mathbb P}\!\left[\phi_\theta(\bar{\mathbf e}(x))\right]$, where $\bar{\mathbf e}(x)=\frac{1}{K}\sum_{i=1}^K \mathbf e_i(x)$.
Then the leading-order shifts satisfy:
\begin{equation}
\label{eq:patch_global_ratio}
\|\Delta_{\mathrm{PFS}}\|_2
\approx
K\,\|\Delta_{\mathrm{global}}\|_2 
>
\|\Delta_{\mathrm{global}}\|_2,
\end{equation}
\end{proposition}

Notably, Proposition~\ref{prop:patch_advantage} does not imply unbounded gains as the number of patches increases.
When finite-sample estimation and patch-resolution effects are taken into account, the patch advantage admits an optimal granularity, as observed in Section~\ref{sec:exp:abla} and analyzed in Appendix~\ref{sec:cor_patch_tradeoff}.
We quantify how the amplified PFS shift manifests as a measurable population MMD gap between $\mathbb{P}$ and $\mathbb{Q}$.

\begin{proposition}
\label{prop:population_mmd}
Let $k_\omega(\cdot,\cdot)$ be a Gaussian kernel where $\omega$ denotes the set of projection weights $\theta$ and kernel bandwidth $\gamma$.
Under Proposition~\ref{prop:pfs_mean_shift_second_order} and a Gaussian surrogate in PFS space, the population $\widehat{\mathrm{MMD}}^2$ between $\mathbb{P}$ and $\mathbb{Q}$ satisfies:
\begin{align}
\label{eq:population_mmd_new}
\widehat{\mathrm{MMD}}^2(\mathbb{P},\mathbb{Q}; k_\omega) =
2\left(\frac{\gamma^2}{\gamma^2+2\sigma_z^2}\right)^{\frac{Kd}{2}}
\left[
1 - \exp\!\left(
-\frac{K\|\boldsymbol{\Delta}_{\mathrm{PFS}}\|_2^2}{2(\gamma^2+2\sigma_z^2)}
\right)
\right],
\end{align}
where $\sigma_z^2$ denotes the isotropic proxy variance of the Gaussian surrogate in PFS space. 
$\widehat{\mathrm{MMD}}^2(\mathbb{P},\mathbb{Q}; k_\omega)$ is strictly positive for
$\|\boldsymbol{\Delta}_{\mathrm{PFS}}\|_2>0$ and is monotonically increasing. 
\end{proposition}

Building on Proposition \ref{prop:population_mmd}, we derive the finite-sample concentration guarantees for detection.

\begin{theorem}
\label{thm:detection}
Let $S_r=\{x_i\}_{i=1}^{M}\stackrel{i.i.d}{\sim}\mathbb{P}$ be a reference set of real images
and $S_t=\{y_j\}_{j=1}^N$ be test images, let $\lambda=\gamma^2+2\sigma_z^2$. For any $\delta\in(0,1)$, with probability at least $1-\delta$, the following holds:

\textbf{(Case I: Real test image).}
If $S_t\stackrel{i.i.d}{\sim}\mathbb{P}$,
\begin{equation}
\label{eq:real_case_bound}
\widehat{\mathrm{MMD}}_{u}^{2}(S_r,S_t)
\le
\underbrace{C_1(\sigma_z,\gamma)\sqrt{\left(\frac{1}{M}+\frac{1}{N}\right)\log\frac{2}{\delta}}}_{\text{Finite-sample fluctuation}}.
\end{equation}
\textbf{(Case II: Generated test image).}
If $S_t\stackrel{i.i.d}{\sim}\mathbb{Q}$,
\begin{align}
\label{eq:fake_case_bound}
\widehat{\mathrm{MMD}}_{u}^{2}(S_r,S_t)
\ge 
\underbrace{
2\left(\frac{\gamma^2}{\lambda}\right)^{\frac{Kd}{2}}
\left[
1-\exp\!\left(-\frac{K\|\boldsymbol{\Delta}_{\mathrm{PFS}}\|_2^2}{2\lambda}\right)
\right]}_{\text{Artifact-induced signature shift}}
-
\underbrace{C_2(\sigma_z,\gamma)\sqrt{\left(\frac{1}{M}+\frac{1}{N}\right)\log\frac{2}{\delta}}}_{\text{Finite-sample error}}.
\end{align}
\end{theorem}
\textbf{Interpretation.} Theorem~\ref{thm:detection} establishes that the empirical MMD concentrates around its population value with deviation scaling as $O(\sqrt{1/M + 1/N})$.
For real test images, the population MMD vanishes and values reflect only finite-sample fluctuations.
For generated images, Proposition~\ref{prop:population_mmd} guarantees a positive gap scaling with $\|\boldsymbol{\Delta}_{\mathrm{PFS}}\|_2^2$.
When this separation dominates, real images yield smaller MMD scores than generated ones, justifying reliable detection for AI-generated images. 

\begin{table*}[t!]
\setlength{\tabcolsep}{3pt} 
\caption{\small{Detection performance ($\%$) on ImageNet benchmark. 
We mainly compare training-based methods.} 
}
\label{compar_imagenet}
\renewcommand\arraystretch{1.1}
\resizebox{\textwidth}{!}{%
\begin{tabular}{@{}llcccccccccccccccccccc@{}}
\toprule

 & &
  \multicolumn{2}{c}{ADM} &
  \multicolumn{2}{c}{ADMG} &
  \multicolumn{2}{c}{LDM} &
  \multicolumn{2}{c}{DiT} &
  \multicolumn{2}{c}{BigGAN} &
  \multicolumn{2}{c}{GigaGAN} &
  \multicolumn{2}{c}{StyleGAN XL} &
  \multicolumn{2}{c}{RQ-Transformer} &
  \multicolumn{2}{c}{Mask GIT} &
  \multicolumn{2}{c}{\multirow{-1}{*}{Average}} \\ \cmidrule(l){3-4} \cmidrule(l){5-6}\cmidrule(l){7-8}  \cmidrule(l){9-10}\cmidrule(l){11-12}\cmidrule(l){13-14}\cmidrule(l){15-16}\cmidrule(l){17-18}\cmidrule(l){19-20}
\multirow{-3}{*}{Methods}  &
\multirow{-3}{*}{Venue}  &
  AUROC &
  AP &
  AUROC&
  AP &
  AUROC&
  AP &
  AUROC&
  AP &
  AUROC&
  AP &
  AUROC&
  AP &
  AUROC&
  AP &
  AUROC&
  AP &
  AUROC&
  AP &
  AUROC ($\uparrow$)&
  AP ($\uparrow$) \\ \midrule
\rowcolor{gray!10}
CNNspot~\citep{wang2020cnn} & CVPR'20 &62.25 &63.13 &63.28 &62.27 &63.16 &64.81 &62.85 &61.16 &85.71 &84.93 &74.85 &71.45 &68.41 &68.67 &61.83 &62.91 &60.98 &61.69 &67.04 &66.78 \\
\rowcolor{gray!10}
Ojha~\citep{ojha2023towards} & CVPR'23 &83.37 &82.95 &79.60 &78.15 &80.35 &79.71 &82.93 &81.72 &93.07 &92.77 &87.45 &84.88 &85.36 &83.15 &85.19 &84.22 &90.82 &90.71 &85.35 &84.25\\
\rowcolor{gray!10}
DIRE~\citep{wang2023dire} & ICCV'23 &51.82 &50.29 &53.14 &52.96 &52.83 &51.84 &54.67 &55.10 &51.62 &50.83 &50.70 &50.27 &50.95 &51.36 &55.95 &54.83 &52.58 &52.10 &52.70 &52.18 \\
\rowcolor{gray!10}
NPR~\citep{tan2024rethinking} & CVPR'24 &85.68 &80.86 &84.34 &79.79 &91.98 &86.96 &86.15 &81.26 &89.73 &84.46 &82.21 &78.20 &84.13 &78.73 &80.21 &73.21 &89.61 &84.15 &86.00 &80.84 \\
\rowcolor{gray!10}
PatchCraft~\citep{zhong2023patchcraft} & --- &81.83 &79.65 &70.88 &69.36 &68.47 &65.19 &75.38 &73.29 &99.85 &99.26 &98.55 &97.91 &96.33 &96.25 &91.28 &91.47 &92.56 &92.17 &86.13 &84.95 \\
\rowcolor{gray!10}
DRCT~\citep{chen2024drct} & ICML'24 &90.26 &90.07 &85.74 &83.85 &90.24 &89.88 &\underline{88.27} &89.06 &95.87 &94.99 &86.89 &86.12 &89.11 &88.39 &92.38 &92.41 &94.44 &94.47 &90.36 &89.92\\
\rowcolor{gray!10}
FatFormer~\citep{liu2024forgery} & CVPR'24 & 91.77 & 90.36 &83.58 &83.17 &92.58 &92.06 &86.93 &85.14 &98.76 &98.47 &97.65 &98.02 &97.64 &97.57 &96.55 &95.96 &97.65 &97.27 &93.68 &93.11\\
\rowcolor{gray!10}
\rowcolor{gray!10}
LOTA~\citep{wang2025lota} & ICCV'25 &66.84 &65.73 &67.18 &66.61 &80.68 &88.33 &74.85 &84.49 &77.95 &78.06 &78.96 &87.92 &73.55 &83.99 &67.41 &76.16 &82.34 &90.19 &74.42 &80.16 \\
\rowcolor{gray!10}
C2P-CLIP~\citep{tan2025c2pclip} & AAAI'25 &72.12 &77.88 &69.07 &75.10 &90.06 &95.72 &48.68 &74.04 &99.84 &99.88 &85.82 &94.19 &94.39 &97.69 &82.27 &91.60 &98.97 &99.60 &82.36 &89.52 \\
\rowcolor{gray!10}
SAFE~\citep{li2025safe} & KDD'25 &65.51 &59.52 &64.78 &59.12 &91.41 &94.36 &87.42 &\underline{92.64} &93.07 &92.11 &90.80 &94.57 &90.11 &93.85 &88.84 &92.28 &94.41 &96.53 &85.15 &86.11 \\
\rowcolor{gray!10}
AIDE~\citep{yan2025aide} & ICLR'25 &90.32 &90.96 &86.96 &\underline{88.08} &90.44 &94.95 &78.77 &87.97 &99.62 &99.65 &96.46 &98.26 &97.62 &\underline{98.85} &\underline{98.19} &\underline{99.10} &\underline{99.50} &\underline{99.75} &93.10 &95.29 \\
\rowcolor{gray!10}
Effort~\citep{yan2025effort} & ICML'25 &88.28 &89.96 &83.74 &85.89 &\underline{94.15} &\underline{97.30} &84.14 &92.47 &\textbf{99.96} &\textbf{99.96} &94.46 &97.56 &94.24 &97.52 &95.52 &97.93 &\textbf{99.84} &\textbf{99.93} &92.70 &\underline{95.39}\\
F-ConV~\citep{zhang2025detecting} & NeurIPS'25 &\textbf{92.74} &\underline{91.65} &\underline{88.51} &87.67 &88.87 &88.47 &85.94 &84.88 &98.94 &98.98 &98.14 &\underline{98.72} &\underline{98.52} &98.38 &96.79 &96.33 &95.52 &95.38 &\underline{93.77} &93.38 \\
\rowcolor{gray!30}
\textbf{MDMF} & --- & \underline{92.56} & \textbf{93.57} & \textbf{88.86} & \textbf{90.16} & \textbf{94.63} & \textbf{97.35} & \textbf{88.89} & \textbf{94.48} & \underline{99.93} & \underline{99.94} & \textbf{98.99} & \textbf{99.52} & \textbf{98.76} & \textbf{99.41} & \textbf{98.84} & \textbf{99.46} & 99.40 & 99.72 & \textbf{95.65} & \textbf{97.07} \\
 \bottomrule
\end{tabular}
}
\label{tab:imagenet}
\vspace{-3.5mm}
\end{table*}

\begin{figure*}[t]
  \centering
  \begin{minipage}{0.48\linewidth}
    \centering
    \label{fig:sora:visual}
    \includegraphics[width=\linewidth]{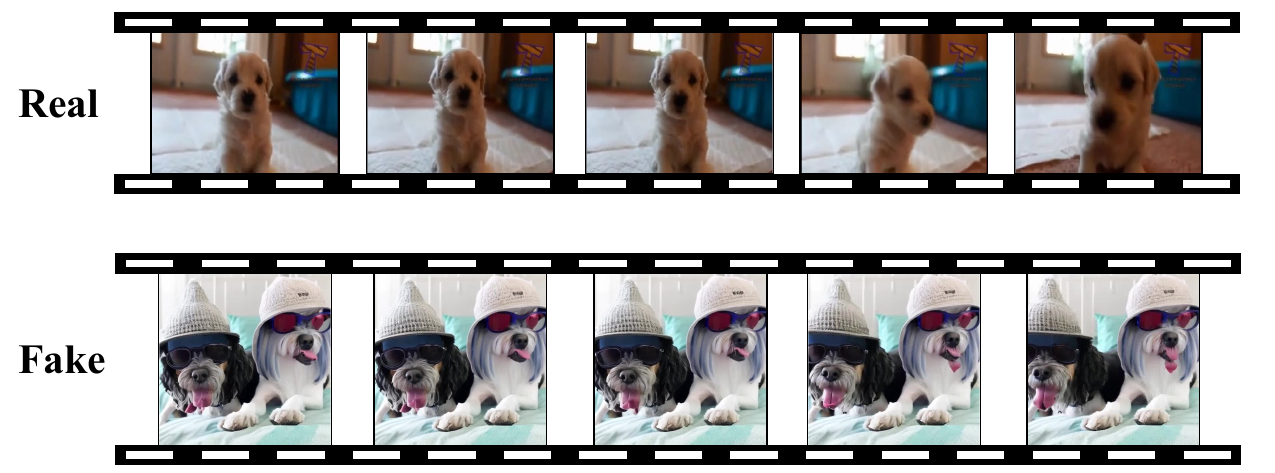}
    {\small (a) Examples of \textit{Real} and \textit{Fake} videos}
  \end{minipage}\hfill
  \begin{minipage}{0.48\linewidth}
    \centering
    \label{fig:sora:result}
    \includegraphics[width=\linewidth]{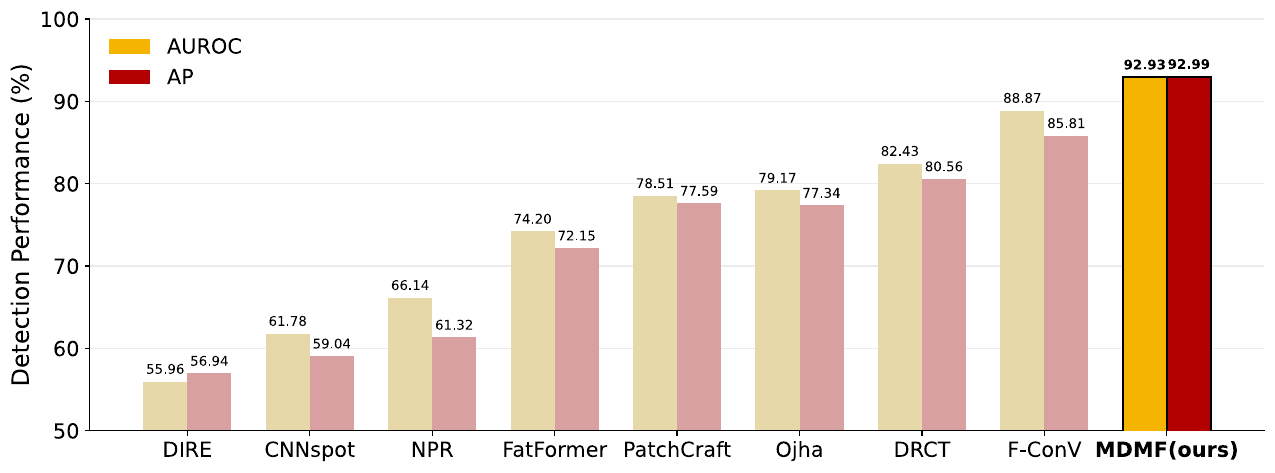}
    {\small (b) Detection Performance on OpenSora}
  \end{minipage}
  \caption{\small{Examples visualization and performance comparison on OpenSora.}}
  \label{fig:sora}
  \vspace{-3.5mm}
\end{figure*}

\section{Experiments}
\label{sec:exp}

\subsection{Experimental Setup}
We provide detailed experimental setups in Appendix \ref{app:sec:addexpset}.
\paragraph{Datasets.} Following previous works \citep{wang2020cnn, zhang2025detecting}, we evaluate our MDMF on the following benchmarks: \textbf{ImageNet} \citep{deng2009imagenet}, \textbf{LSUN-Bedroom} \citep{yu2015lsun}, \textbf{GenImage} \citep{zhu2023genimage}, in-the-wild \textbf{WildRF}~\citep{cavia2024wildrf}, and \textbf{LDMFakeDetect}~\citep{rajan2025staypositive}. 
To further assess generalization to generators beyond image benchmarks, we additionally conduct a case study on videos generated by \textbf{OpenSora} \citep{zheng2024open}.
Specifically, we sample 3,275 generated videos and extract 10 frames per video, resulting in 32,750 frames and treat them as generated images. For real data, we sample the same number of natural videos and frames on MSR-VTT \citep{xu2016msr}.

\paragraph{Baselines and Evaluation Metrics.} We compare our MDMF with the following training-based detection baselines in the main experiments: CNNspot \citep{wang2020cnn}, Ojha \citep{ojha2023towards}, DIRE \citep{wang2023dire}, PatchCraft \citep{zhong2023patchcraft}, NPR \citep{tan2024rethinking}, DRCT \citep{chen2024drct}, FatFormer \citep{liu2024forgery}, LOTA \citep{wang2025lota}, C2P-CLIP \citep{tan2025c2pclip}, SAFE \citep{li2025safe}, AIDE \citep{yan2025aide}, Effort \citep{yan2025effort}, F-ConV \citep{zhang2025detecting}. 
Following \citep{zhang2025detecting}, we adopt the following metrics: \ding{172} average precision (AP); \ding{173} area under the receiver operating characteristic curve (AUROC); \ding{174} classification accuracy (ACC).

\paragraph{Implementation Details.} Following previous studies \citep{ojha2023towards,liu2024forgery}, we apply random cropping and random horizontal flipping at training, while center cropping at testing, both with no other augmentations. To balance detection performance and efficiency, we adopt DINOv2 ViT-L/14 \citep{oquab2024dinov} to extract patch embeddings and pool the patch size to $W=32$ for PFS computation in main experiments. The projection $\phi_\theta$ and kernel bandwidth $\gamma$ are jointly trained during optimization. 

\subsection{Main Results}

\noindent\textbf{Detection performance comparison with baselines.} Table~\ref{tab:imagenet} reports detection performance on the ImageNet benchmark across nine generative models spanning diffusion, GANs, and transformers. MDMF demonstrates consistently strong performance across all evaluated generators, indicating robust generalization under diverse generative mechanisms. Notably, MDMF shows particularly strong performance on recent diffusion-based models, which are known to produce highly realistic images with sparse and localized artifacts that challenge existing detectors. 
These results validate that our PFS distributional modeling effectively captures the subtle, localized forensic signals characteristic of modern generative paradigms. 
Beyond diffusion models, MDMF also maintains competitive performance on earlier generative paradigms. This consistent behavior further demonstrates that MDMF
effectively captures generator-agnostic forensic artifacts and amplifies micro-scale defects into robust macro-level detection signals
across both emerging and conventional generative models.

\noindent\textbf{Case Study on OpenSora-Generated Content.}
We further evaluate MDMF on a challenging case study using frames sampled from OpenSora~\citep{zheng2024open}, a recent video generation model that is not seen during training. 
Figure~\ref{fig:sora}(a) shows the advanced diffusion-generated videos with strong temporal consistency introduced by OpenSora, resulting in frames that are globally coherent and largely free of perceptual artifacts.
As illustrated in Figure~\ref{fig:sora}(b), while several competitive baselines exhibit notable performance degradation, MDMF still maintains robust detection performance on OpenSora-generated frames. 
This contrast indicates that the distributional modeling of MDMF captures localized forensic signatures that persist even under substantial domain shifts, enabling effective generalization to emerging video generation paradigms that are unseen during training.

\begin{figure}[t!]
  \centering
  \subfigure[\small Diff. Patch Size]{
  \label{abla:1}
  \begin{minipage}[t!]{0.23\linewidth}
        \centering
	\includegraphics[width=\textwidth]{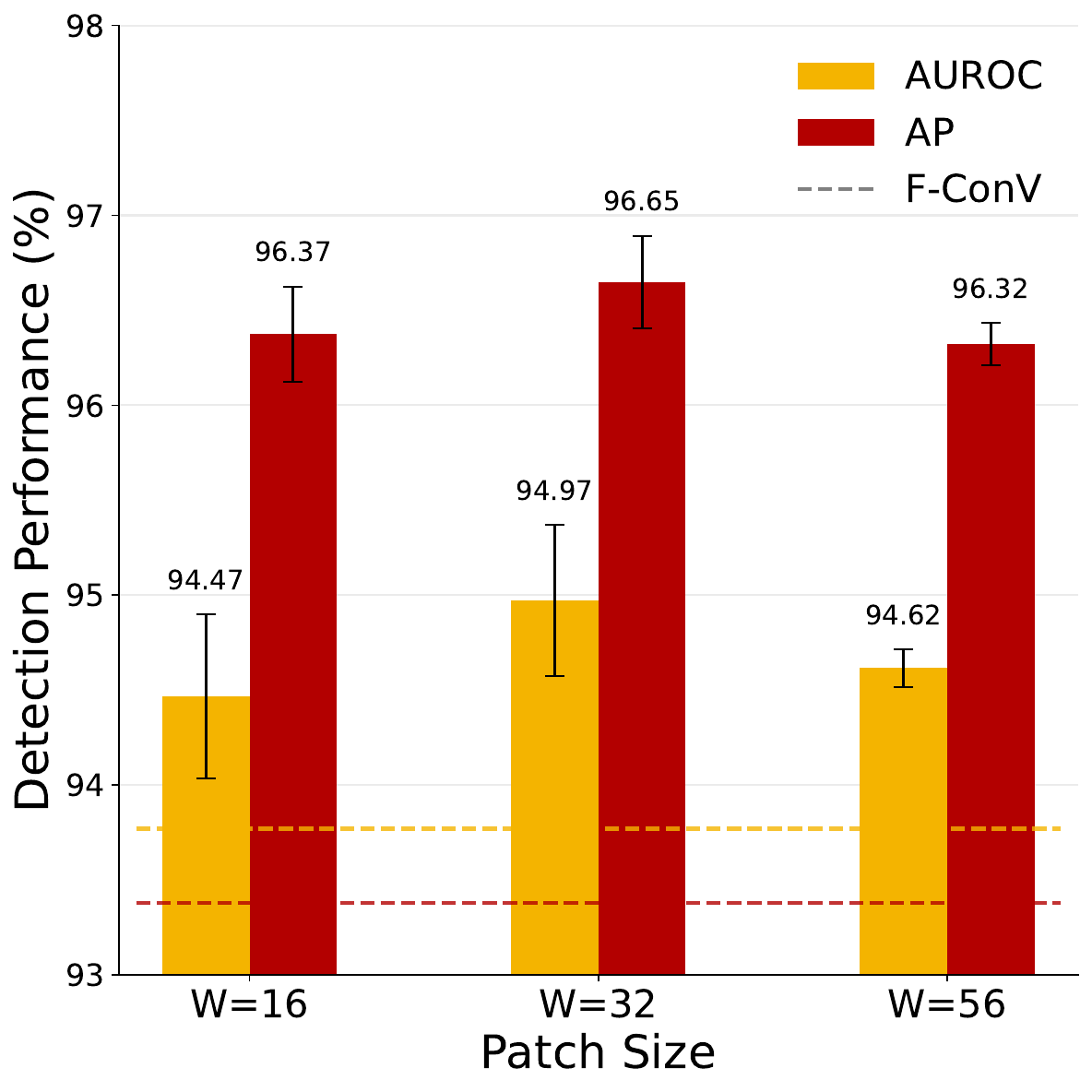}
  \end{minipage}}
  \subfigure[\small Diff. Backbone]{
  \label{abla:2}
  \begin{minipage}[t!]{0.23\linewidth}
        \centering
	\includegraphics[width=\textwidth]{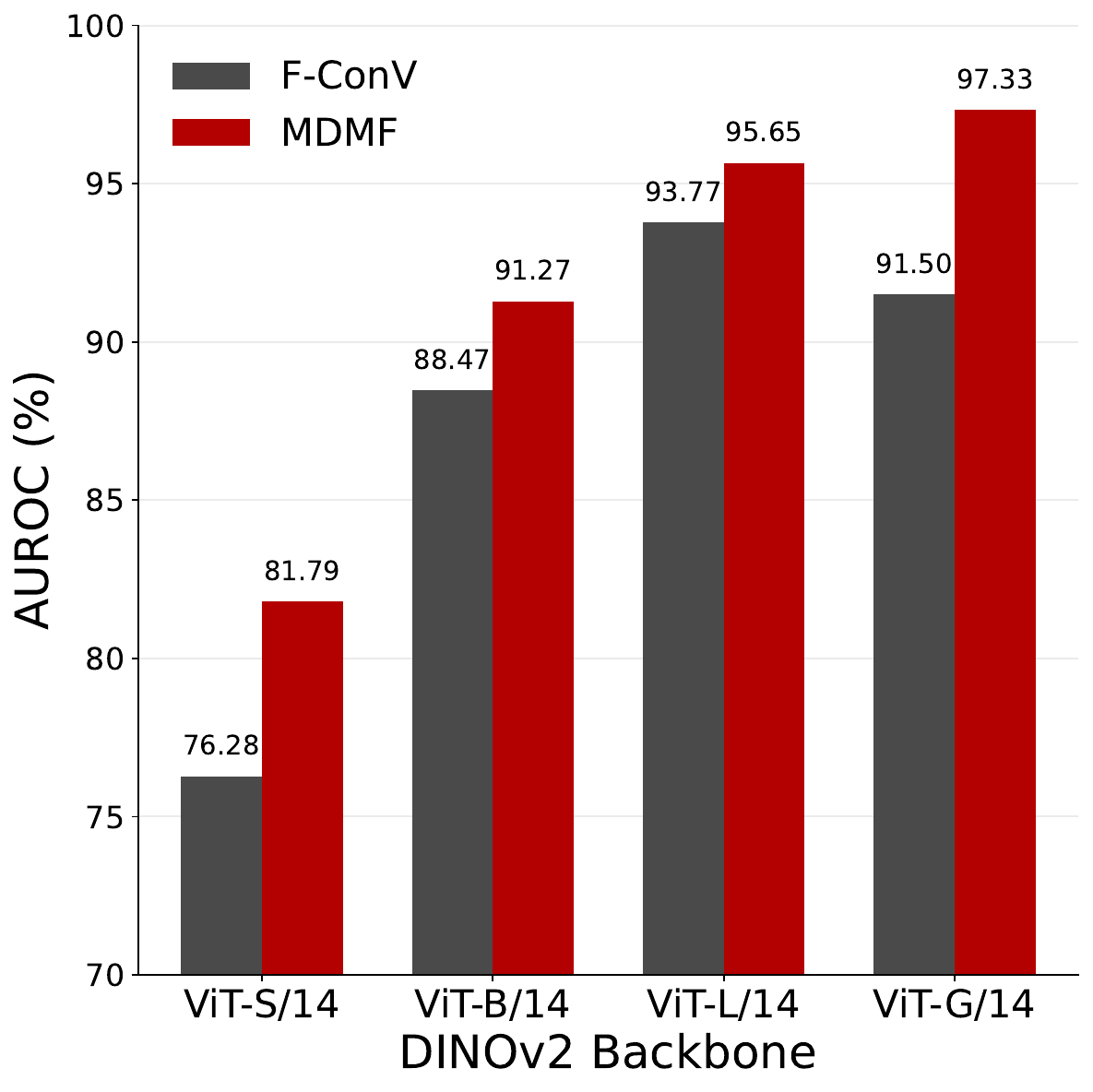}
  \end{minipage}}
  \subfigure[\small Diff. Perturbation]{
  \label{abla:3}
  \begin{minipage}[t!]{0.23\linewidth}
        \centering
	\includegraphics[width=\textwidth]{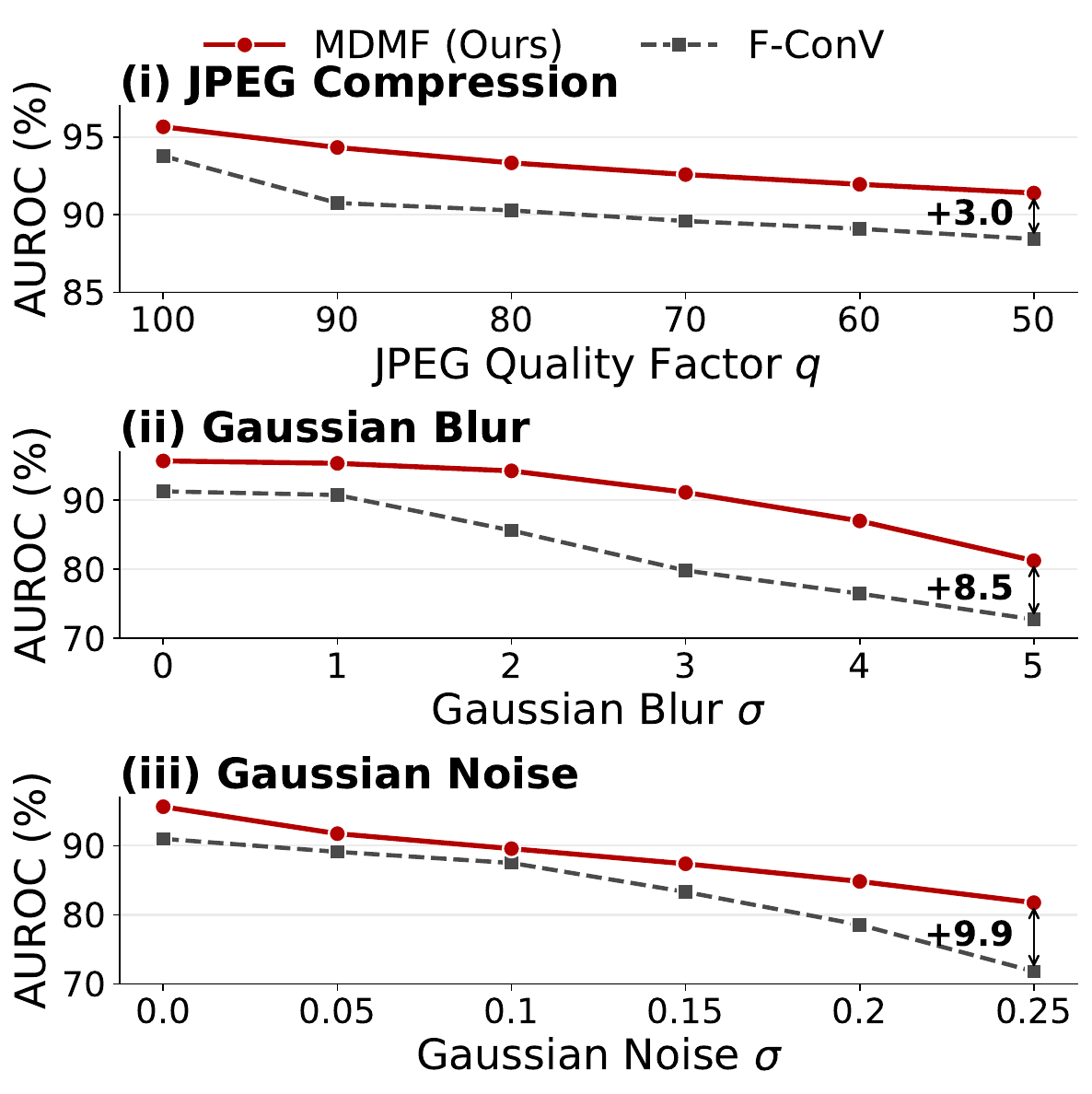}
  \end{minipage}}
  \subfigure[\small Comp. Patch Voting]{
  \label{abla:4}
  \begin{minipage}[t!]{0.23\linewidth}
        \centering
	\includegraphics[width=\textwidth]{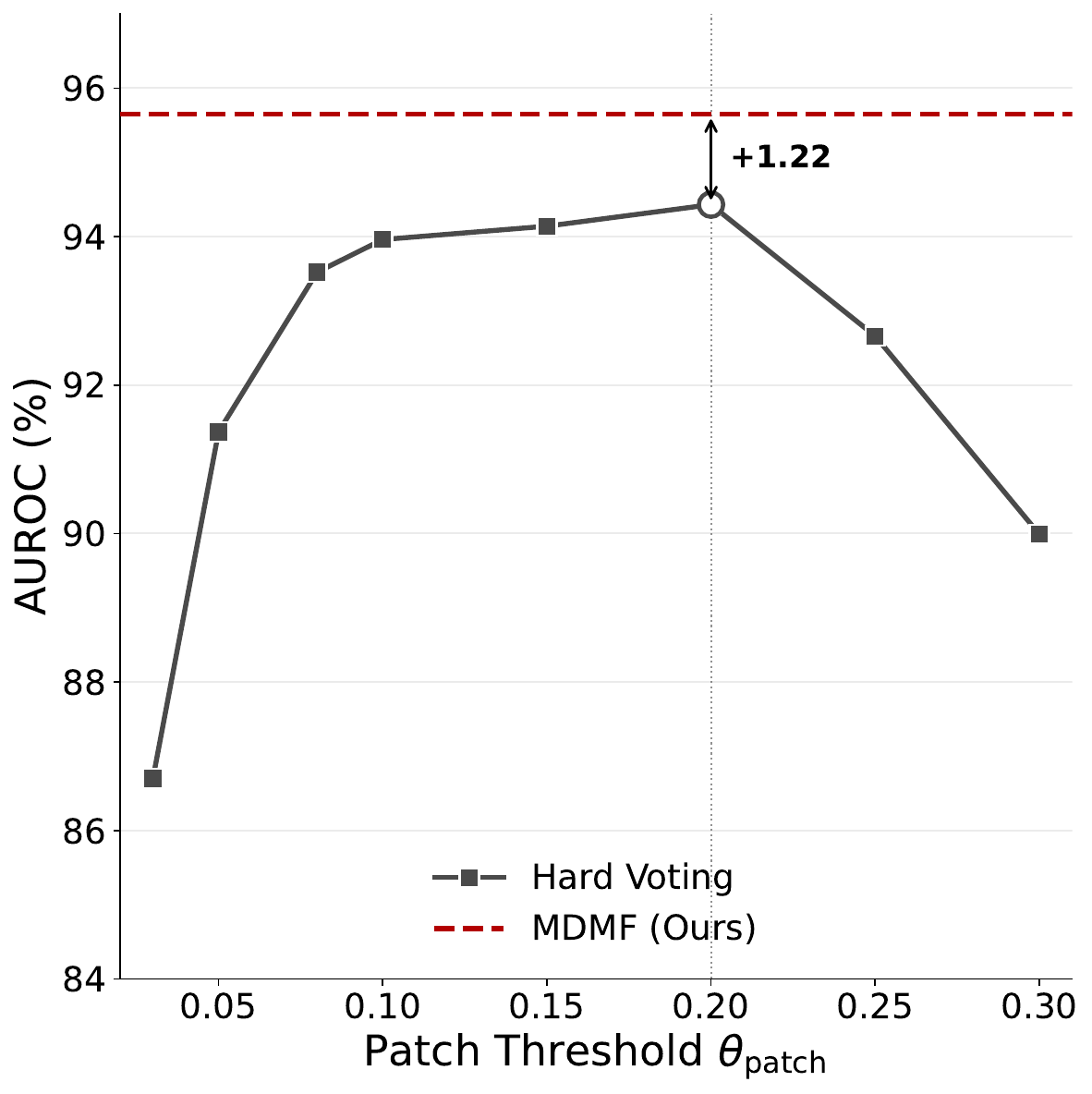}
  \end{minipage}}
\caption{\small{\textbf{Further analysis.}
(a) Sensitivity to patch size $W$;
(b) Robustness to DINOv2 backbone variants;
(c) Robustness to post-processing perturbations;
(d) Comparison with patch-level hard voting under varying $\theta_{\mathrm{patch}}$.}}
  \label{fig:abla}
  \vspace{-4mm}
\end{figure}

\vspace{-1mm}
\subsection{Ablation and Further Analysis}
\label{sec:exp:abla}

We provide detailed results and discussions in Appendix \ref{app:sec:addexpres}.

\begin{wraptable}{r}{0.45\linewidth}
    \centering
    \captionsetup{width=\linewidth}
    \vspace{-14pt}
    \caption{\small{\textbf{Ablation of key components on ImageNet.} Variants without MMD are trained with a BCE objective, while PFS modeling without MMD adds a lightweight attention head for aggregation.
    }}
    \vspace{-2pt}
    \footnotesize
    \resizebox{\linewidth}{!}{\begin{tabular}{c c c c}
    \toprule
    \multirow{2}{*}{PFS Modeling} & \multirow{2}{*}{MMD Optimization} & \multicolumn{2}{c}{Average} \\
    \cmidrule(l){3-4}
    & & AUROC ($\uparrow$) & AP ($\uparrow$) \\
    \midrule
     \multicolumn{4}{c}{Global Pooling}\\
     \rowcolor{gray!20}
    \ding{55} & \ding{55} & 90.14 & 93.33 \\
    \rowcolor{gray!20}
    \ding{55} & $\checkmark$ & 86.53 & 92.18 \\
    \midrule
     \multicolumn{4}{c}{PFS Modeling}\\
     \rowcolor{pink!20}
    $\checkmark$ & \ding{55} & 93.22 & 95.34 \\
     \rowcolor{pink!20}
    $\checkmark$& $\checkmark$ & \textbf{95.65} & \textbf{97.07} \\
    \bottomrule
    \end{tabular}}
    \label{tab:abla}
    \vspace{-6pt}
\end{wraptable}
\vspace{-1mm}
\paragraph{Ablation of core components in MDMF.}
Table~\ref{tab:abla} analyzes the contribution of PFS modeling and MMD optimization in MDMF. For variants without MMD, we train binary classifiers with a standard BCE loss. In particular, the PFS w/o MMD variant uses a lightweight attention head to pool patch-wise PFS into an image-level score (details in Appendix~\ref{app:sec:addexpset}). 
Notably, even without MMD optimization, attention-aggregated PFS still achieves competitive performance, indicating that the forensic reparameterization already suppresses semantic dominance and highlights generation-related cues.
In contrast, the effect of MMD is dependent on the underlying representation. When applied to global pooling, MMD fails to yield performance gains, whereas combining with PFS leads to a clear improvement. This demonstrates MMD serves as a complementary amplifier when patch-wise forensic evidence is preserved in the PFS space, rather than when it is diluted by the semantics-dominant global representations.

\noindent\textbf{Effect of patch granularity.}
To evaluate the effect of patch granularity, we vary the patch size $W$ while keeping other settings fixed. Figure~\ref{fig:abla}(a) shows a non-monotonic dependence on $W$. While always above baselines, performance improves as $W$ increases from small values and degrades when $W$ becomes overly large. This trend supports Proposition~\ref{prop:patch_advantage}, which predicts a finite optimal granularity rather than a monotonic behavior. 
Coarse partitions ($W=56$) provide insufficient spatial resolution to capture sparse local shifts, while fine partitions ($W=16$) weaken the forensic evidence within each patch and introduce higher variance in distributional comparison. This indicates that PFS benefits from an intermediate granularity that balances localized sensitivity with reliable estimation.

\begin{figure*}[t!]
  \small
  \centering
  \includegraphics[width=0.98\linewidth]{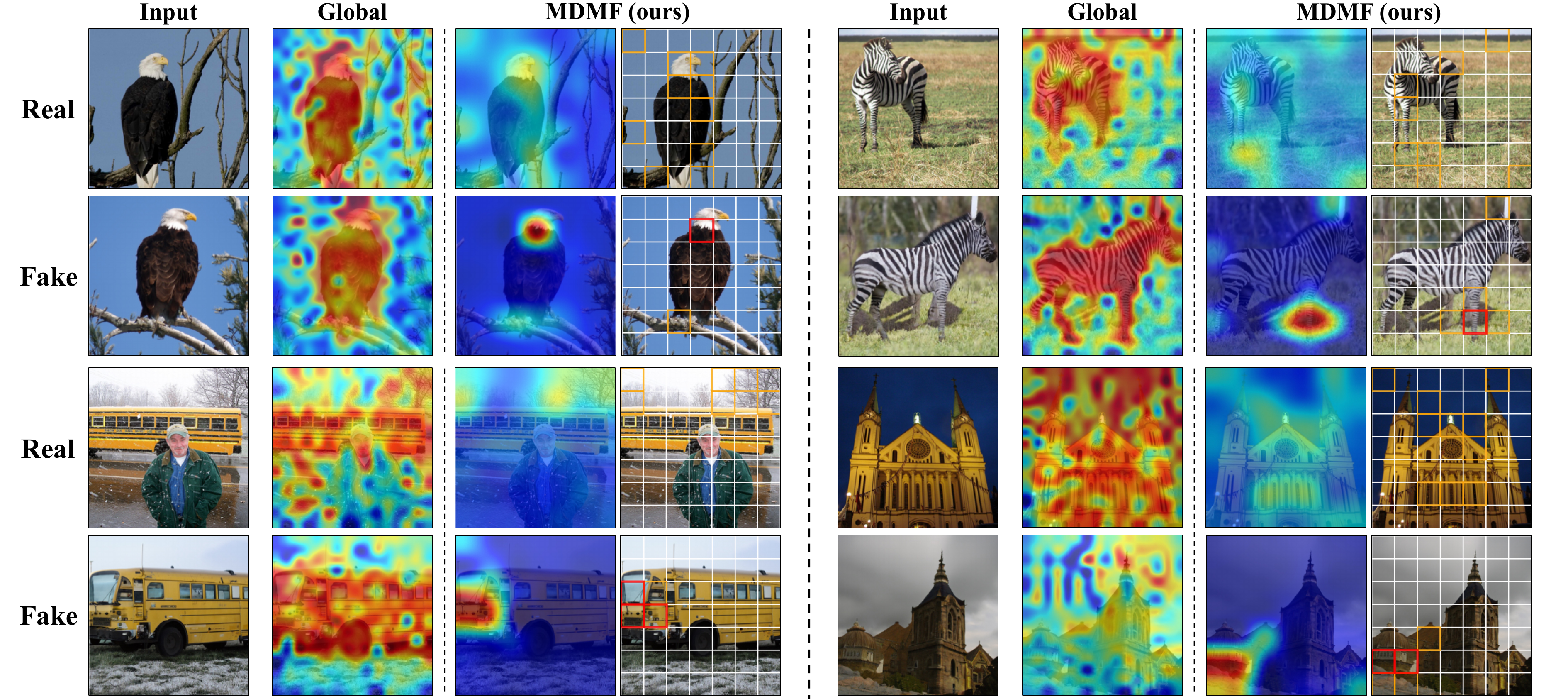}
  \caption{\small \textbf{Qualitative visualization of localized forensic evidence.} We compare representative real images and category-matched generated images with Grad-CAM, where warmer colors indicate higher predicted likelihood of being fake. Global-pooling baseline primarily highlights semantically salient regions with similar patterns for real and generated samples, whereas MDMF shows localized responses on generated images and diffuse activations on real images, consistent with capturing subtle generation-induced irregularities.} 
  \label{fig:exp:qualitative}
  \vspace{-4mm}
\end{figure*}

\noindent\textbf{Robustness to the encoder architecture.}
To evaluate sensitivity to the feature extractor, we instantiate MDMF with multiple DINOv2 backbone variants and compare it with F-ConV under the same setting. As shown in Figure~\ref{fig:abla}(b), MDMF consistently achieves higher detection performance across all evaluated encoders, demonstrating robustness to the underlying backbone choice. MDMF maintains an advantage with smaller backbones (e.g., ViT-S/14), and continues to improve as the backbone is scaled up, whereas F-ConV exhibits non-monotonic behavior and degrades on the largest encoder (e.g., ViT-G/14). This contrast suggests that the proposed PFS representation provides robust forensic signals that transfer effectively across encoder scales, leading to more stable behavior.

\noindent\textbf{Robustness to common post-processing perturbations.}
We further evaluate MDMF and F-ConV under JPEG compression, Gaussian blur, and Gaussian noise on ImageNet. As shown in Figure~\ref{fig:abla}(c), MDMF maintains higher AUROC across all severity levels, with markedly gentler degradation than F-ConV ($-4.3$ vs. $-5.3$ for JPEG, $-14.4$ vs. $-18.6$ for blur, $-13.9$ vs. $-19.2$ for noise), and the gap widens as severity grows ($+3.0$/$+8.5$/$+9.9$ AUROC at the most severe levels). This stability arises because PFS amplifies localized forensic cues rather than relying on global statistics, while MMD aggregates evidence across patches to provide redundancy that single-image classifiers lack.

\noindent\textbf{Comparison with patch-level hard voting.}
A natural alternative to MDMF's distributional aggregation is to classify each patch independently and aggregate via hard voting on the resulting fake-patch ratio. To isolate the effect of the aggregation strategy, we keep DINOv2 backbone, patch tokenization ($49$ patches at $W{=}32$), and training data identical to MDMF. Crucially, hard voting requires \emph{two coupled thresholds}, a per-patch sigmoid cutoff $\theta_{\mathrm{patch}}$ and an image-level decision $\tau$, whereas MDMF needs only $\tau$. As shown in Figure~\ref{fig:abla}(d), voting AUROC swings from $86.70$ to $94.43$ as $\theta_{\mathrm{patch}}$ varies, and even its best configuration trails MDMF by $+1.22$ AUROC, confirming that distributional testing over PFS captures the patch-population signal more reliably than independent per-patch decisions.

\noindent\textbf{Qualitative visualization of localized forensic cues.}
To better understand how MDMF detects highly realistic samples, Figure~\ref{fig:exp:qualitative} visualizes representative real images and category-matched generated images produced by ADM, together with Grad-CAM heatmaps from different models where warmer colors indicate a higher predicted probability of being fake. 
First, we can observe that the generated images exhibit strong semantic coherence and high visual fidelity.
Consistently, the global pooling visualization primarily highlights semantically salient regions, such as object boundaries and high-contrast textures, indicating similar patterns for real and generated images and limited sensitivity to sparse local artifacts.
In contrast, MDMF produces more localized responses on generated images and assigns higher activation to regions that contain subtle generative irregularities, while producing more diffuse patterns on real images. 
This reflects a pronounced distributional discrepancy between real and generated samples in the PFS space, which provides a strong basis for MMD to produce a stable detection signal.
This also aligns with our theoretical and quantitative analysis, suggesting that MDMF can surface localized evidence that is suppressed by semantic-dominated global features.

\vspace{-1mm}
\section{Conclusion}
In this paper, we present a distributional perspective for AI-generated image detection by modeling an image as a collection of localized visual evidence. Building on this view, we introduce \textbf{\emph{Patch Forensic Signature} (PFS)}, a learnable forensic representation that reparameterizes semantic embeddings into a latent space to suppress semantic invariances while amplifying generative artifacts. We further propose \textbf{\emph{Micro-Defects expose Macro-Fakes} (MDMF)}, which measures distributional discrepancy over PFS via MMD to aggregate localized evidence into stable image-level detection signals, and we provide theoretical analysis that establishes the advantage of PFS and the separation between real and generated images. Extensive experiments on multiple benchmarks with detailed ablations and analyses demonstrate the effectiveness and generalization of MDMF.

\bibliography{ref}

@article{
oquab2024dinov,
title={{DINO}v2: Learning Robust Visual Features without Supervision},
author={Maxime Oquab and Timoth{\'e}e Darcet and Th{\'e}o Moutakanni and Huy V. Vo and Marc Szafraniec and Vasil Khalidov and Pierre Fernandez and Daniel HAZIZA and Francisco Massa and Alaaeldin El-Nouby and Mido Assran and Nicolas Ballas and Wojciech Galuba and Russell Howes and Po-Yao Huang and Shang-Wen Li and Ishan Misra and Michael Rabbat and Vasu Sharma and Gabriel Synnaeve and Hu Xu and Herve Jegou and Julien Mairal and Patrick Labatut and Armand Joulin and Piotr Bojanowski},
journal={Transactions on Machine Learning Research},
issn={2835-8856},
year={2024},
url={https://openreview.net/forum?id=a68SUt6zFt},
note={Featured Certification}
}

@book{wainwright2019high,
  title={High-dimensional statistics: A non-asymptotic viewpoint},
  author={Wainwright, Martin J},
  volume={48},
  year={2019},
  publisher={Cambridge university press}
}

@article{gretton2012kernel,
  title={A kernel two-sample test},
  author={Gretton, Arthur and Borgwardt, Karsten M and Rasch, Malte J and Sch{\"o}lkopf, Bernhard and Smola, Alexander},
  journal={The journal of machine learning research},
  volume={13},
  number={1},
  pages={723--773},
  year={2012},
  publisher={JMLR. org}
}

@inproceedings{wang2020cnn,
  title={CNN-generated images are surprisingly easy to spot... for now},
  author={Wang, Sheng-Yu and Wang, Oliver and Zhang, Richard and Owens, Andrew and Efros, Alexei A},
  booktitle={Proceedings of the IEEE/CVF conference on computer vision and pattern recognition},
  pages={8695--8704},
  year={2020}
}

@inproceedings{ojha2023towards,
  title={Towards universal fake image detectors that generalize across generative models},
  author={Ojha, Utkarsh and Li, Yuheng and Lee, Yong Jae},
  booktitle={Proceedings of the IEEE/CVF Conference on Computer Vision and Pattern Recognition},
  pages={24480--24489},
  year={2023}
}

@inproceedings{wang2023dire,
  title={Dire for diffusion-generated image detection},
  author={Wang, Zhendong and Bao, Jianmin and Zhou, Wengang and Wang, Weilun and Hu, Hezhen and Chen, Hong and Li, Houqiang},
  booktitle={Proceedings of the IEEE/CVF International Conference on Computer Vision},
  pages={22445--22455},
  year={2023}
}

@inproceedings{tan2024rethinking,
  title={Rethinking the up-sampling operations in cnn-based generative network for generalizable deepfake detection},
  author={Tan, Chuangchuang and Zhao, Yao and Wei, Shikui and Gu, Guanghua and Liu, Ping and Wei, Yunchao},
  booktitle={Proceedings of the IEEE/CVF Conference on Computer Vision and Pattern Recognition},
  pages={28130--28139},
  year={2024}
}

@inproceedings{chen2024drct,
  title={Drct: Diffusion reconstruction contrastive training towards universal detection of diffusion generated images},
  author={Chen, Baoying and Zeng, Jishen and Yang, Jianquan and Yang, Rui},
  booktitle={Forty-first International Conference on Machine Learning},
  year={2024}
}

@inproceedings{liu2024forgery,
  title={Forgery-aware adaptive transformer for generalizable synthetic image detection},
  author={Liu, Huan and Tan, Zichang and Tan, Chuangchuang and Wei, Yunchao and Wang, Jingdong and Zhao, Yao},
  booktitle={Proceedings of the IEEE/CVF Conference on Computer Vision and Pattern Recognition},
  pages={10770--10780},
  year={2024}
}

@article{zhang2025detecting,
  title={Detecting Generated Images by Fitting Natural Image Distributions},
  author={Zhang, Yonggang and Nie, Jun and Tian, Xinmei and Gong, Mingming and Zhang, Kun and Han, Bo},
  journal={arXiv preprint arXiv:2511.01293},
  year={2025}
}

@inproceedings{deng2009imagenet,
  title={Imagenet: A large-scale hierarchical image database},
  author={Deng, Jia and Dong, Wei and Socher, Richard and Li, Li-Jia and Li, Kai and Fei-Fei, Li},
  booktitle={2009 IEEE conference on computer vision and pattern recognition},
  pages={248--255},
  year={2009},
  organization={Ieee}
}

@article{yu2015lsun,
  title={Lsun: Construction of a large-scale image dataset using deep learning with humans in the loop},
  author={Yu, Fisher and Seff, Ari and Zhang, Yinda and Song, Shuran and Funkhouser, Thomas and Xiao, Jianxiong},
  journal={arXiv preprint arXiv:1506.03365},
  year={2015}
}

@article{zhu2023genimage,
  title={Genimage: A million-scale benchmark for detecting ai-generated image},
  author={Zhu, Mingjian and Chen, Hanting and Yan, Qiangyu and Huang, Xudong and Lin, Guanyu and Li, Wei and Tu, Zhijun and Hu, Hailin and Hu, Jie and Wang, Yunhe},
  journal={Advances in Neural Information Processing Systems},
  volume={36},
  pages={77771--77782},
  year={2023}
}

@inproceedings{liu2020learning,
  title={Learning deep kernels for non-parametric two-sample tests},
  author={Liu, Feng and Xu, Wenkai and Lu, Jie and Zhang, Guangquan and Gretton, Arthur and Sutherland, Danica J},
  booktitle={International conference on machine learning},
  pages={6316--6326},
  year={2020},
  organization={PMLR}
}

@article{zhang2024detecting,
  title={Detecting machine-generated texts by multi-population aware optimization for maximum mean discrepancy},
  author={Zhang, Shuhai and Song, Yiliao and Yang, Jiahao and Li, Yuanqing and Han, Bo and Tan, Mingkui},
  journal={arXiv preprint arXiv:2402.16041},
  year={2024}
}

@article{zhang2025physics,
  title={Physics-Driven Spatiotemporal Modeling for AI-Generated Video Detection},
  author={Zhang, Shuhai and Lian, ZiHao and Yang, Jiahao and Li, Daiyuan and Pang, Guoxuan and Liu, Feng and Han, Bo and Li, Shutao and Tan, Mingkui},
  journal={arXiv preprint arXiv:2510.08073},
  year={2025}
}

@article{wang2025diffdoctor,
  title={DiffDoctor: Diagnosing Image Diffusion Models Before Treating},
  author={Wang, Yiyang and Chen, Xi and Xu, Xiaogang and Ji, Sihui and Liu, Yu and Shen, Yujun and Zhao, Hengshuang},
  journal={arXiv preprint arXiv:2501.12382},
  year={2025}
}

@article{wang2024detecting,
  title={Detecting Human Artifacts from Text-to-Image Models},
  author={Wang, Kaihong and Zhang, Lingzhi and Zhang, Jianming},
  journal={arXiv preprint arXiv:2411.13842},
  year={2024}
}

@article{wang2024embedding,
  title={Embedding trajectory for out-of-distribution detection in mathematical reasoning},
  author={Wang, Yiming and Zhang, Pei and Yang, Baosong and Wong, Derek and Zhang, Zhuosheng and Wang, Rui},
  journal={Advances in Neural Information Processing Systems},
  volume={37},
  pages={42965--42999},
  year={2024}
}

@inproceedings{rombach2022high,
  title={High-resolution image synthesis with latent diffusion models},
  author={Rombach, Robin and Blattmann, Andreas and Lorenz, Dominik and Esser, Patrick and Ommer, Bj{\"o}rn},
  booktitle={Proceedings of the IEEE/CVF conference on computer vision and pattern recognition},
  pages={10684--10695},
  year={2022}
}

@article{ramesh2022hierarchical,
  title={Hierarchical text-conditional image generation with clip latents},
  author={Ramesh, Aditya and Dhariwal, Prafulla and Nichol, Alex and Chu, Casey and Chen, Mark},
  journal={arXiv preprint arXiv:2204.06125},
  volume={1},
  number={2},
  pages={3},
  year={2022}
}

@article{saharia2022photorealistic,
  title={Photorealistic text-to-image diffusion models with deep language understanding},
  author={Saharia, Chitwan and Chan, William and Saxena, Saurabh and Li, Lala and Whang, Jay and Denton, Emily L and Ghasemipour, Kamyar and Gontijo Lopes, Raphael and Karagol Ayan, Burcu and Salimans, Tim and others},
  journal={Advances in neural information processing systems},
  volume={35},
  pages={36479--36494},
  year={2022}
}

@article{podell2023sdxl,
  title={Sdxl: Improving latent diffusion models for high-resolution image synthesis},
  author={Podell, Dustin and English, Zion and Lacey, Kyle and Blattmann, Andreas and Dockhorn, Tim and M{\"u}ller, Jonas and Penna, Joe and Rombach, Robin},
  journal={arXiv preprint arXiv:2307.01952},
  year={2023}
}

@article{lipman2022flow,
  title={Flow matching for generative modeling},
  author={Lipman, Yaron and Chen, Ricky TQ and Ben-Hamu, Heli and Nickel, Maximilian and Le, Matt},
  journal={arXiv preprint arXiv:2210.02747},
  year={2022}
}

@article{ho2020denoising,
  title={Denoising diffusion probabilistic models},
  author={Ho, Jonathan and Jain, Ajay and Abbeel, Pieter},
  journal={Advances in neural information processing systems},
  volume={33},
  pages={6840--6851},
  year={2020}
}

@inproceedings{zhou2023synthetic,
  title={Synthetic lies: Understanding ai-generated misinformation and evaluating algorithmic and human solutions},
  author={Zhou, Jiawei and Zhang, Yixuan and Luo, Qianni and Parker, Andrea G and De Choudhury, Munmun},
  booktitle={Proceedings of the 2023 CHI conference on human factors in computing systems},
  pages={1--20},
  year={2023}
}

@article{heidari2024deepfake,
  title={Deepfake detection using deep learning methods: A systematic and comprehensive review},
  author={Heidari, Arash and Jafari Navimipour, Nima and Dag, Hasan and Unal, Mehmet},
  journal={Wiley Interdisciplinary Reviews: Data Mining and Knowledge Discovery},
  volume={14},
  number={2},
  pages={e1520},
  year={2024},
  publisher={Wiley Online Library}
}

@inproceedings{somepalli2023diffusion,
  title={Diffusion art or digital forgery? investigating data replication in diffusion models},
  author={Somepalli, Gowthami and Singla, Vasu and Goldblum, Micah and Geiping, Jonas and Goldstein, Tom},
  booktitle={Proceedings of the IEEE/CVF conference on computer vision and pattern recognition},
  pages={6048--6058},
  year={2023}
}

@article{he2024rigid,
  title={Rigid: A training-free and model-agnostic framework for robust ai-generated image detection},
  author={He, Zhiyuan and Chen, Pin-Yu and Ho, Tsung-Yi},
  journal={arXiv preprint arXiv:2405.20112},
  year={2024}
}

@article{choi2025training,
  title={Training-free Detection of AI-generated images via Cropping Robustness},
  author={Choi, Sungik and Lee, Hankook and Lee, Moontae},
  journal={arXiv preprint arXiv:2511.14030},
  year={2025}
}

@article{zheng2024open,
  title={Open-sora: Democratizing efficient video production for all},
  author={Zheng, Zangwei and Peng, Xiangyu and Yang, Tianji and Shen, Chenhui and Li, Shenggui and Liu, Hongxin and Zhou, Yukun and Li, Tianyi and You, Yang},
  journal={arXiv preprint arXiv:2412.20404},
  year={2024}
}

@article{zhong2023patchcraft,
  title={Patchcraft: Exploring texture patch for efficient ai-generated image detection},
  author={Zhong, Nan and Xu, Yiran and Li, Sheng and Qian, Zhenxing and Zhang, Xinpeng},
  journal={arXiv preprint arXiv:2311.12397},
  year={2023}
}

@inproceedings{touvron2021training,
  title={Training data-efficient image transformers \& distillation through attention},
  author={Touvron, Hugo and Cord, Matthieu and Douze, Matthijs and Massa, Francisco and Sablayrolles, Alexandre and J{\'e}gou, Herv{\'e}},
  booktitle={International conference on machine learning},
  pages={10347--10357},
  year={2021},
  organization={PMLR}
}

@inproceedings{liu2021swin,
  title={Swin transformer: Hierarchical vision transformer using shifted windows},
  author={Liu, Ze and Lin, Yutong and Cao, Yue and Hu, Han and Wei, Yixuan and Zhang, Zheng and Lin, Stephen and Guo, Baining},
  booktitle={Proceedings of the IEEE/CVF international conference on computer vision},
  pages={10012--10022},
  year={2021}
}

@inproceedings{zhang2019detecting,
  title={Detecting and simulating artifacts in gan fake images},
  author={Zhang, Xu and Karaman, Svebor and Chang, Shih-Fu},
  booktitle={2019 IEEE international workshop on information forensics and security (WIFS)},
  pages={1--6},
  year={2019},
  organization={IEEE}
}

@article{zhu2023gendet,
  title={Gendet: Towards good generalizations for ai-generated image detection},
  author={Zhu, Mingjian and Chen, Hanting and Huang, Mouxiao and Li, Wei and Hu, Hailin and Hu, Jie and Wang, Yunhe},
  journal={arXiv preprint arXiv:2312.08880},
  year={2023}
}

@inproceedings{liu2020global,
  title={Global texture enhancement for fake face detection in the wild},
  author={Liu, Zhengzhe and Qi, Xiaojuan and Torr, Philip HS},
  booktitle={Proceedings of the IEEE/CVF conference on computer vision and pattern recognition},
  pages={8060--8069},
  year={2020}
}

@inproceedings{qian2020thinking,
  title={Thinking in frequency: Face forgery detection by mining frequency-aware clues},
  author={Qian, Yuyang and Yin, Guojun and Sheng, Lu and Chen, Zixuan and Shao, Jing},
  booktitle={European conference on computer vision},
  pages={86--103},
  year={2020},
  organization={Springer}
}

@inproceedings{xu2016msr,
  title={Msr-vtt: A large video description dataset for bridging video and language},
  author={Xu, Jun and Mei, Tao and Yao, Ting and Rui, Yong},
  booktitle={Proceedings of the IEEE conference on computer vision and pattern recognition},
  pages={5288--5296},
  year={2016}
}

@article{brock2018large,
  title={Large Scale GAN Training for High Fidelity Natural Image Synthesis},
  author={Brock, Andrew},
  journal={arXiv preprint arXiv:1809.11096},
  year={2018}
}

@inproceedings{karras2019style,
  title={A style-based generator architecture for generative adversarial networks},
  author={Karras, Tero and Laine, Samuli and Aila, Timo},
  booktitle={Proceedings of the IEEE/CVF conference on computer vision and pattern recognition},
  pages={4401--4410},
  year={2019}
}

@article{dhariwal2021diffusion,
  title={Diffusion models beat gans on image synthesis},
  author={Dhariwal, Prafulla and Nichol, Alexander},
  journal={Advances in neural information processing systems},
  volume={34},
  pages={8780--8794},
  year={2021}
}

@inproceedings{peebles2023scalable,
  title={Scalable diffusion models with transformers},
  author={Peebles, William and Xie, Saining},
  booktitle={Proceedings of the IEEE/CVF international conference on computer vision},
  pages={4195--4205},
  year={2023}
}

@article{gu2022wukong,
  title={Wukong: A 100 million large-scale chinese cross-modal pre-training benchmark},
  author={Gu, Jiaxi and Meng, Xiaojun and Lu, Guansong and Hou, Lu and Minzhe, Niu and Liang, Xiaodan and Yao, Lewei and Huang, Runhui and Zhang, Wei and Jiang, Xin and others},
  journal={Advances in Neural Information Processing Systems},
  volume={35},
  pages={26418--26431},
  year={2022}
}

@inproceedings{nichol2021improved,
  title={Improved denoising diffusion probabilistic models},
  author={Nichol, Alexander Quinn and Dhariwal, Prafulla},
  booktitle={International conference on machine learning},
  pages={8162--8171},
  year={2021},
  organization={PMLR}
}

@article{kingma2013auto,
  title={Auto-encoding variational bayes},
  author={Kingma, Diederik P and Welling, Max},
  journal={arXiv preprint arXiv:1312.6114},
  year={2013}
}

@article{sohn2015learning,
  title={Learning structured output representation using deep conditional generative models},
  author={Sohn, Kihyuk and Lee, Honglak and Yan, Xinchen},
  journal={Advances in neural information processing systems},
  volume={28},
  year={2015}
}

@article{nichol2021glide,
  title={Glide: Towards photorealistic image generation and editing with text-guided diffusion models},
  author={Nichol, Alex and Dhariwal, Prafulla and Ramesh, Aditya and Shyam, Pranav and Mishkin, Pamela and McGrew, Bob and Sutskever, Ilya and Chen, Mark},
  journal={arXiv preprint arXiv:2112.10741},
  year={2021}
}

@article{brooks2024video,
  title={Video generation models as world simulators},
  author={Brooks, Tim and Peebles, Bill and Holmes, Connor and DePue, Will and Guo, Yufei and Jing, Li and Schnurr, David and Taylor, Joe and Luhman, Troy and Luhman, Eric and others},
  journal={OpenAI Blog},
  volume={1},
  number={8},
  pages={1},
  year={2024}
}

@article{zhang2024if,
  title={What if the input is expanded in OOD detection?},
  author={Zhang, Boxuan and Zhu, Jianing and Wang, Zengmao and Liu, Tongliang and Du, Bo and Han, Bo},
  journal={Advances in Neural Information Processing Systems},
  volume={37},
  pages={21289--21329},
  year={2024}
}

@article{chen2024demamba,
  title={Demamba: Ai-generated video detection on million-scale genvideo benchmark},
  author={Chen, Haoxing and Hong, Yan and Huang, Zizheng and Xu, Zhuoer and Gu, Zhangxuan and Li, Yaohui and Lan, Jun and Zhu, Huijia and Zhang, Jianfu and Wang, Weiqiang and others},
  journal={arXiv preprint arXiv:2405.19707},
  year={2024}
}

@inproceedings{wang2025lota,
  title={LOTA: Bit-Planes Guided AI-Generated Image Detection},
  author={Wang, Hongsong and Cheng, Renxi and Zhang, Yang and Han, Chaolei and Gui, Jie},
  booktitle={Proceedings of the IEEE/CVF International Conference on Computer Vision},
  pages={17246--17255},
  year={2025}
}

@inproceedings{tan2025c2pclip,
  title={C2p-clip: Injecting category common prompt in clip to enhance generalization in deepfake detection},
  author={Tan, Chuangchuang and Tao, Renshuai and Liu, Huan and Gu, Guanghua and Wu, Baoyuan and Zhao, Yao and Wei, Yunchao},
  booktitle={Proceedings of the AAAI Conference on Artificial Intelligence},
  volume={39},
  number={7},
  pages={7184--7192},
  year={2025}
}

@inproceedings{li2025safe,
  title={Improving synthetic image detection towards generalization: An image transformation perspective},
  author={Li, Ouxiang and Cai, Jiayin and Hao, Yanbin and Jiang, Xiaolong and Hu, Yao and Feng, Fuli},
  booktitle={Proceedings of the 31st ACM SIGKDD Conference on Knowledge Discovery and Data Mining V. 1},
  pages={2405--2414},
  year={2025}
}

@article{yan2025aide,
  title={A sanity check for ai-generated image detection},
  author={Yan, Shilin and Li, Ouxiang and Cai, Jiayin and Hao, Yanbin and Jiang, Xiaolong and Hu, Yao and Xie, Weidi},
  journal={arXiv preprint arXiv:2406.19435},
  year={2024}
}

@article{yan2025effort,
  title={Orthogonal subspace decomposition for generalizable ai-generated image detection},
  author={Yan, Zhiyuan and Wang, Jiangming and Jin, Peng and Zhang, Ke-Yue and Liu, Chengchun and Chen, Shen and Yao, Taiping and Ding, Shouhong and Wu, Baoyuan and Yuan, Li},
  journal={arXiv preprint arXiv:2411.15633},
  year={2024}
}

@article{cavia2024wildrf,
  title={Real-time deepfake detection in the real-world},
  author={Cavia, Bar and Horwitz, Eliahu and Reiss, Tal and Hoshen, Yedid},
  journal={arXiv preprint arXiv:2406.09398},
  year={2024}
}

@article{rajan2025staypositive,
  title={Stay-Positive: A Case for Ignoring Real Image Features in Fake Image Detection},
  author={Rajan, Anirudh Sundara and Lee, Yong Jae},
  journal={arXiv preprint arXiv:2502.07778},
  year={2025}
}

@article{rajan2024aligned,
  title={Aligned datasets improve detection of latent diffusion-generated images},
  author={Rajan, Anirudh Sundara and Ojha, Utkarsh and Schloesser, Jedidiah and Lee, Yong Jae},
  journal={arXiv preprint arXiv:2410.11835},
  year={2024}
}

@inproceedings{chai2020makes,
  title={What makes fake images detectable? understanding properties that generalize},
  author={Chai, Lucy and Bau, David and Lim, Ser-Nam and Isola, Phillip},
  booktitle={European conference on computer vision},
  pages={103--120},
  year={2020},
  organization={Springer}
}

@inproceedings{corvi2023detection,
  title={On the detection of synthetic images generated by diffusion models},
  author={Corvi, Riccardo and Cozzolino, Davide and Zingarini, Giada and Poggi, Giovanni and Nagano, Koki and Verdoliva, Luisa},
  booktitle={ICASSP 2023-2023 IEEE International Conference on Acoustics, Speech and Signal Processing (ICASSP)},
  pages={1--5},
  year={2023},
  organization={IEEE}
}
\bibliographystyle{plain}

\clearpage
\appendix

\etocdepthtag.toc{mtappendix}
\etocsettagdepth{mtchapter}{none}
\etocsettagdepth{mtappendix}{subsection}
\renewcommand{\contentsname}{Appendices}
\tableofcontents

\clearpage


\section*{Reproducibility Statement}
To facilitate reproducibility, we summarize key experimental details and provide the necessary resources in the submitted supplementary materials.

\begin{itemize}
    \item \textbf{Datasets.} All benchmarks used in this paper are publicly available. We evaluate on ImageNet~\citep{deng2009imagenet}, LSUN-Bedroom~\citep{yu2015lsun}, GenImage~\citep{zhu2023genimage}, the in-the-wild WildRF~\citep{cavia2024wildrf}, and LDMFakeDetect~\citep{rajan2025staypositive} following standard protocols in prior AI-generated image detection works~\citep{zhang2025detecting}. For our stress test, we construct an OpenSora-generated dataset by sampling videos from the GenVideo's~\citep{chen2024demamba} OpenSora~\citep{zheng2024open} subset and extracting frames, and use MSR-VTT~\citep{xu2016msr} as the corresponding real-video source (details in Appendix~\ref{app:sec:datasets}).

    \item \textbf{Assumption.} Our method follows the common \emph{training-based} detection setting adopted by prior detectors~\citep{chen2024drct,liu2024forgery}, where a detector is trained on a designated training set and then evaluated on multiple generators and benchmarks for generalization. We keep the training pipeline consistent across all experiments.

    \item \textbf{Open source.} We include our source code in the submitted supplementary materials. The release contains training and evaluation scripts and pretrained checkpoints where applicable to reproduce our results.

    \item \textbf{Environment.} Experiments are conducted on a single NVIDIA H200 GPU using Python 3.10.19 and PyTorch 2.9.1. Key hyperparameters (optimizer, learning rate, batch size, epochs, patch granularity, etc.) are reported in Appendix~\ref{app:sec:imple}.
\end{itemize}

\section{Theoretical Analysis}
\label{app:sec:theory}

\subsection{Preliminaries and Modeling Assumptions}
\label{app:theo:prelim}
This section provides the probabilistic tools and regularity conditions used in our proofs.
We formalize (i) sub-Gaussian patch embeddings, (ii) weak spatial dependence across patches
(used only when relating patchwise and global pooling), and (iii) second-order regularity of the PFS mapping.

\paragraph{Sub-Gaussian random vectors.}
A random vector $X\in\mathbb{R}^m$ is called \emph{$\sigma$-sub-Gaussian} if for all unit vectors
$u\in\mathbb{S}^{m-1}$ and all $t\in\mathbb{R}$,
\begin{equation}
\label{eq:subg_def}
\mathbb{E}\Big[\exp\big(t\,u^\top (X-\mathbb{E}[X])\big)\Big]
\le
\exp\!\left(\frac{\sigma^2 t^2}{2}\right).
\end{equation}
We denote by $\mathcal{SG}(\boldsymbol{\mu}, \sigma^2 \mathbf{I})$ a $\sigma$-sub-Gaussian distribution
with mean $\boldsymbol{\mu}$ and isotropic proxy covariance $\sigma^2\mathbf I$.

\paragraph{Extracted patch embeddings.}
Real images are i.i.d. from $\mathbb P$ and generated images are i.i.d. from $\mathbb Q$.
Given an image, we extract $K$ non-overlapping patch embeddings
$\{\mathbf e_i\}_{i=1}^K\subset\mathbb R^D$ from a fixed pre-trained encoder (e.g., DINOv2).
Throughout the analysis, we assume each patch embedding is $\sigma_e$-sub-Gaussian:
\[
\mathbf e_i(x)\sim \mathcal{SG}(\mathbf 0,\sigma_e^2\mathbf I_D),\qquad
\mathbf e_i(y)\ \text{follows the sparse-defect model in Assumption~\ref{ass:sign_sparse_defect}.}
\]

\paragraph{Weak spatial dependence across patches.}
Within one image, patch embeddings may exhibit spatial correlation.
To quantify this, we model $\{\mathbf e_i\}_{i=1}^K$ as an $\alpha$-mixing sequence.
Let $\mathcal F_1^i$ be the $\sigma$-algebra generated by $\{\mathbf e_1,\ldots,\mathbf e_i\}$ and
$\mathcal F_{i+\ell}^K$ generated by $\{\mathbf e_{i+\ell},\ldots,\mathbf e_K\}$.
The $\alpha$-mixing coefficient is
\begin{equation}
\label{eq:mixing_def}
\alpha(\ell)
:=
\sup_{i}\ \sup_{A\in\mathcal F_1^i,\ B\in\mathcal F_{i+\ell}^K}
\big|\mathbb P(A\cap B)-\mathbb P(A)\mathbb P(B)\big|.
\end{equation}
We assume exponential mixing:
\begin{equation}
\label{eq:mixing_exp}
\alpha(\ell)\le C_\alpha e^{-c_\alpha \ell}
\quad\text{for some constants } C_\alpha,c_\alpha>0.
\end{equation}
This assumption is only used to control covariance shrinkage after global pooling.

\paragraph{Effective sample size.}
Define an \emph{effective} patch count
\begin{equation}
\label{eq:Keff_def}
\frac{1}{K_{\mathrm{eff}}}
:=
\frac{1}{K}
+
\frac{2}{K^2}\sum_{\ell=1}^{K-1}(K-\ell)\,\beta(\ell),
\end{equation}
where $\beta(\ell)$ upper-bounds cross-patch covariance contribution at lag $\ell$
(e.g., $\beta(\ell)\propto \alpha(\ell)^\eta$ for some $\eta\in(0,1]$ under standard mixing-to-covariance bounds).
Under exponential mixing \eqref{eq:mixing_exp}, $\sum_{\ell\ge 1}\beta(\ell)<\infty$ and thus
$K_{\mathrm{eff}}=\Theta(K)$ (i.e., it scales linearly with $K$ up to constants).

\paragraph{PFS mapping and second-order regularity.}
Let $\phi_\theta:\mathbb R^D\to\mathbb R^d$ be the learnable patch forensic signature (PFS) mapping,
and write $\phi_\theta=(\phi_{\theta,1},\ldots,\phi_{\theta,d})$.

\begin{assumption}[\textbf{Locally Smooth PFS Mapping (Second-order)}]
\label{ass:smooth_mapping_second}
There exist constants $L,M,R>0$ and a neighborhood $\mathcal E\subset\mathbb R^D$
containing the typical support mass of both real and generated patch embeddings such that for all $\mathbf e\in\mathcal E$,
\begin{equation}
\label{eq:jac_hess_bound}
\|J_\phi(\mathbf e)\|_{\mathrm{op}} \le L,
\qquad
\|\nabla^2 \phi_{\theta,\ell}(\mathbf e)\|_{\mathrm{op}} \le M\ \ \text{for all }\ell=1,\ldots,d,
\end{equation}
and the second-order Taylor remainder satisfies, for each $\ell$,
\begin{equation}
\label{eq:taylor_remainder_bound}
\big|R_{\ell}(\mathbf e)\big|
\le
\frac{R}{6}\|\mathbf e\|_2^3,
\end{equation}
where $R_\ell(\mathbf e)$ is the remainder term in the expansion of $\phi_{\theta,\ell}(\mathbf e)$ around $\mathbf 0$.
\end{assumption}

\begin{remark}
Assumption~\ref{ass:smooth_mapping_second} is mild when $\phi_\theta$ is implemented with smooth activations
(e.g., GELU or $\tanh$) and embeddings are $\ell_2$-normalized, which effectively restricts $\mathbf e$ to a compact region.
The Hessian bound ensures the second-order term is controlled, and \eqref{eq:taylor_remainder_bound}
formalizes that higher-order terms are negligible in the leading-order analysis.
\end{remark}

\subsection{Proof of Proposition~\ref{prop:pfs_mean_shift_second_order}
}
\label{sec:proof_prop_pfs_mean_shift_second_order}

\begin{proof}
We prove \eqref{eq:pfs_second_order_shift_final} by a second-order Taylor expansion.

Recall the definition of $\Delta_{\mathrm{PFS}}$:
\[
\Delta_{\mathrm{PFS}}
:=
\mathbb E_{\mathbb Q}\!\left[\phi_\theta(\mathbf e_i(y))\right]
-
\mathbb E_{\mathbb P}\!\left[\phi_\theta(\mathbf e_i(x))\right].
\]
Under Assumption~\ref{ass:sign_sparse_defect}, we have
\[
\mathbf e_i(x)=\mathbf u_i,\qquad
\mathbf e_i(y)=\mathbf u_i + a_i s_i \boldsymbol{\mu}_{\mathrm{defect}},
\]
where $\mathbf u_i\sim \mathcal{SG}(\mathbf 0,\sigma_e^2\mathbf I_D)$,
$a_i\sim\mathrm{Bernoulli}(\rho)$, and $s_i$ is Rademacher independent of $\mathbf u_i,a_i$.

First, we compute the mean of generated patch embeddings:
\begin{align*}
\mathbb E[\mathbf e_i(y)]
&=
\mathbb E[\mathbf u_i] + \mathbb E[a_i s_i]\,\boldsymbol{\mu}_{\mathrm{defect}} \\
&=
\mathbf 0 + \mathbb E[a_i]\mathbb E[s_i]\,\boldsymbol{\mu}_{\mathrm{defect}}
\qquad (\text{by independence of }a_i\text{ and }s_i)\\
&=
\rho\cdot 0\cdot \boldsymbol{\mu}_{\mathrm{defect}}=\mathbf 0,
\end{align*}
and similarly $\mathbb E[\mathbf e_i(x)]=\mathbb E[\mathbf u_i]=\mathbf 0$.
Hence any leading-order shift cannot arise from the linear (Jacobian) term.

\paragraph{Compute Second-order Taylor expansion of $\phi_\theta$.}
Let $\phi_\theta=(\phi_{\theta,1},\ldots,\phi_{\theta,d})$.
For each output coordinate $\ell\in\{1,\ldots,d\}$, since $\phi_\theta$ is twice differentiable at $\mathbf 0$,
a second-order Taylor expansion around $\mathbf 0$ yields
\begin{equation}
\label{eq:taylor_each_coord}
\phi_{\theta,\ell}(\mathbf e)
=
\phi_{\theta,\ell}(\mathbf 0)
+
\nabla \phi_{\theta,\ell}(\mathbf 0)^\top \mathbf e
+
\frac{1}{2}\mathbf e^\top \nabla^2 \phi_{\theta,\ell}(\mathbf 0)\mathbf e
+
R_{\ell}(\mathbf e),
\end{equation}
where the remainder $R_\ell(\mathbf e)=o(\|\mathbf e\|_2^2)$ as $\|\mathbf e\|_2\to 0$.

Apply \eqref{eq:taylor_each_coord} to $\mathbf e=\mathbf e_i(y)$ and $\mathbf e=\mathbf e_i(x)$ and take expectations:
\begin{align}
\mathbb E_{\mathbb Q}\big[\phi_{\theta,\ell}(\mathbf e_i(y))\big]
&=
\phi_{\theta,\ell}(\mathbf 0)
+
\nabla \phi_{\theta,\ell}(\mathbf 0)^\top \mathbb E[\mathbf e_i(y)]
+
\frac{1}{2}\mathbb E\!\left[\mathbf e_i(y)^\top \nabla^2 \phi_{\theta,\ell}(\mathbf 0)\mathbf e_i(y)\right]
+
\mathbb E[R_\ell(\mathbf e_i(y))], \label{eq:Ey}\\
\mathbb E_{\mathbb P}\big[\phi_{\theta,\ell}(\mathbf e_i(x))\big]
&=
\phi_{\theta,\ell}(\mathbf 0)
+
\nabla \phi_{\theta,\ell}(\mathbf 0)^\top \mathbb E[\mathbf e_i(x)]
+
\frac{1}{2}\mathbb E\!\left[\mathbf e_i(x)^\top \nabla^2 \phi_{\theta,\ell}(\mathbf 0)\mathbf e_i(x)\right]
+
\mathbb E[R_\ell(\mathbf e_i(x))]. \label{eq:Ex}
\end{align}
Subtracting \eqref{eq:Ex} from \eqref{eq:Ey}, the constant terms cancel. Since
$\mathbb E[\mathbf e_i(y)]=\mathbb E[\mathbf e_i(x)]=\mathbf 0$, so the linear terms also vanish. Therefore,
\begin{align}
\label{eq:delta_coord}
\Delta_{\mathrm{PFS},\ell}
:=
\mathbb E_{\mathbb Q}\big[\phi_{\theta,\ell}(\mathbf e_i(y))\big]
-
\mathbb E_{\mathbb P}\big[\phi_{\theta,\ell}(\mathbf e_i(x))\big]
&=
\frac{1}{2}\Big(
\mathbb E[\mathbf e_i(y)^\top H_\ell \mathbf e_i(y)]
-\mathbb E[\mathbf e_i(x)^\top H_\ell \mathbf e_i(x)]
\Big)
+
\epsilon_\ell,
\end{align}
where we write $H_\ell:=\nabla^2\phi_{\theta,\ell}(\mathbf 0)$ and group the remainder difference into
\[
\epsilon_\ell := \mathbb E[R_\ell(\mathbf e_i(y))]-\mathbb E[R_\ell(\mathbf e_i(x))],
\quad \text{which is higher order.}
\]

We now compute the difference
$\mathbb E[\mathbf e^\top H_\ell \mathbf e]$ under $x$ and $y$.

Under the real distribution, $\mathbf e_i(x)=\mathbf u_i$:
\[
\mathbb E[\mathbf e_i(x)^\top H_\ell \mathbf e_i(x)]
=
\mathbb E[\mathbf u_i^\top H_\ell \mathbf u_i].
\]
Under the generated distribution, $\mathbf e_i(y)=\mathbf u_i + a_i s_i \boldsymbol{\mu}_{\mathrm{defect}}$:
\begin{align}
\mathbb E[\mathbf e_i(y)^\top H_\ell \mathbf e_i(y)]
&=
\mathbb E\Big[(\mathbf u_i + a_i s_i \boldsymbol{\mu}_{\mathrm{defect}})^\top
H_\ell
(\mathbf u_i + a_i s_i \boldsymbol{\mu}_{\mathrm{defect}})\Big] \notag\\
&=
\mathbb E[\mathbf u_i^\top H_\ell \mathbf u_i]
+
2\,\mathbb E[a_i s_i]\,\mathbb E[\boldsymbol{\mu}_{\mathrm{defect}}^\top H_\ell \mathbf u_i]
+
\mathbb E[(a_i s_i)^2]\,
\boldsymbol{\mu}_{\mathrm{defect}}^\top H_\ell \boldsymbol{\mu}_{\mathrm{defect}}.
\label{eq:expand_quadratic}
\end{align}
Here we used bilinearity of the quadratic expansion and independence to separate expectations in the cross term, so we have
\[
\mathbb E[a_i s_i]=\mathbb E[a_i]\mathbb E[s_i]=\rho\cdot 0=0,
\]
so the cross term in \eqref{eq:expand_quadratic} vanishes.
Moreover, since $s_i^2=1$ and $a_i\in\{0,1\}$, we have $(a_i s_i)^2=a_i$ and therefore
\[
\mathbb E[(a_i s_i)^2]=\mathbb E[a_i]=\rho.
\]
Plugging these into \eqref{eq:expand_quadratic} yields
\begin{align}
\mathbb E[\mathbf e_i(y)^\top H_\ell \mathbf e_i(y)]
&=
\mathbb E[\mathbf u_i^\top H_\ell \mathbf u_i]
+
\rho\,\boldsymbol{\mu}_{\mathrm{defect}}^\top H_\ell \boldsymbol{\mu}_{\mathrm{defect}} \notag \\
&=
\mathbb E[\mathbf e_i(x)^\top H_\ell \mathbf e_i(x)]
+ 
\rho\,\boldsymbol{\mu}_{\mathrm{defect}}^\top H_\ell \boldsymbol{\mu}_{\mathrm{defect}} \notag
\end{align}
Hence,
\begin{equation}
\label{eq:quadratic_gap}
\mathbb E[\mathbf e_i(y)^\top H_\ell \mathbf e_i(y)]
-
\mathbb E[\mathbf e_i(x)^\top H_\ell \mathbf e_i(x)]
=
\rho\,\boldsymbol{\mu}_{\mathrm{defect}}^\top H_\ell \boldsymbol{\mu}_{\mathrm{defect}}.
\end{equation}

Substituting \eqref{eq:quadratic_gap} into \eqref{eq:delta_coord} and ignoring higher-order remainders $\epsilon_\ell$
gives the leading-order approximation
\[
\Delta_{\mathrm{PFS},\ell}
\approx
\frac{\rho}{2}\,\boldsymbol{\mu}_{\mathrm{defect}}^\top \nabla^2\phi_{\theta,\ell}(\mathbf 0)\boldsymbol{\mu}_{\mathrm{defect}}
=
\frac{\rho}{2}\,[\mathcal Q(\boldsymbol{\mu}_{\mathrm{defect}})]_\ell,
\qquad \ell=1,\ldots,d.
\]
Stacking $\ell=1,\ldots,d$ proves
\[
\Delta_{\mathrm{PFS}}
\approx
\frac{\rho}{2}\,\mathcal Q(\boldsymbol{\mu}_{\mathrm{defect}}),
\]
which is exactly \eqref{eq:pfs_second_order_shift_final}.

If $\mathcal Q(\boldsymbol{\mu}_{\mathrm{defect}})\neq 0$ and $\rho>0$, then
\[
\Delta_{\mathrm{PFS}} \approx \frac{\rho}{2}\,\mathcal Q(\boldsymbol{\mu}_{\mathrm{defect}})\neq 0,
\]
and thus $\|\Delta_{\mathrm{PFS}}\|_2>0$, which completes the proof.
\end{proof}


\subsection{Proof of Proposition~\ref{prop:patch_advantage}
}
\label{sec:proof_prop_patch_advantage}

\begin{proof}
We show that global pooling dilutes the same second-order defect signature by a factor $1/K$,
hence $\|\Delta_{\mathrm{PFS}}\|_2 \approx K\|\Delta_{\mathrm{global}}\|_2$ at leading order.

Recall the definition of the global pooled embeddings: 
\[
\bar{\mathbf e}(x)=\frac{1}{K}\sum_{i=1}^K \mathbf e_i(x),
\qquad
\bar{\mathbf e}(y)=\frac{1}{K}\sum_{i=1}^K \mathbf e_i(y).
\]
Under Assumption~\ref{ass:sign_sparse_defect},
\[
\mathbf e_i(x)=\mathbf u_i,
\qquad
\mathbf e_i(y)=\mathbf u_i + a_i s_i \boldsymbol{\mu}_{\mathrm{defect}}.
\]

Similar to the proof of Proposition~\ref{prop:pfs_mean_shift_second_order}, if we simply adopt the linearity of expectation: 
\[
\mathbb E[\bar{\mathbf e}(x)]
=
\frac{1}{K}\sum_{i=1}^K \mathbb E[\mathbf u_i]
=\mathbf 0,
\qquad
\mathbb E[\bar{\mathbf e}(y)]
=
\frac{1}{K}\sum_{i=1}^K \mathbb E[\mathbf u_i + a_i s_i \boldsymbol{\mu}_{\mathrm{defect}}]
=\mathbf 0.
\]
Thus, as in Proposition~\ref{prop:pfs_mean_shift_second_order}, the leading-order shift arises from second-order terms.
For each coordinate $\ell$, we apply the same second-order expansion at $\mathbf 0$:
\[
\phi_{\theta,\ell}(\bar{\mathbf e})
=
\phi_{\theta,\ell}(\mathbf 0)
+
\nabla \phi_{\theta,\ell}(\mathbf 0)^\top \bar{\mathbf e}
+
\frac{1}{2}\bar{\mathbf e}^\top H_\ell \bar{\mathbf e}
+
R_\ell(\bar{\mathbf e}),
\quad H_\ell=\nabla^2\phi_{\theta,\ell}(\mathbf 0).
\]
According to Appendix \ref{sec:proof_prop_pfs_mean_shift_second_order}, taking expectations and subtracting between $y$ and $x$,
the constant and linear terms cancel since $\mathbb E[\bar{\mathbf e}(y)]=\mathbb E[\bar{\mathbf e}(x)]=\mathbf 0$:
\begin{equation}
\label{eq:delta_global_coord}
\Delta_{\mathrm{global},\ell}
:=
\mathbb E_{\mathbb Q}\big[\phi_{\theta,\ell}(\bar{\mathbf e}(y))\big]
-
\mathbb E_{\mathbb P}\big[\phi_{\theta,\ell}(\bar{\mathbf e}(x))\big]
=
\frac12\Big(
\mathbb E[\bar{\mathbf e}(y)^\top H_\ell \bar{\mathbf e}(y)]
-
\mathbb E[\bar{\mathbf e}(x)^\top H_\ell \bar{\mathbf e}(x)]
\Big)
+\tilde\epsilon_\ell,
\end{equation}
where $\tilde\epsilon_\ell$ collects higher-order remainder differences.

Then we \textbf{reduce the pooled quadratic forms to covariances}.
Since $\mathbb E[\bar{\mathbf e}(x)]=\mathbb E[\bar{\mathbf e}(y)]=\mathbf 0$,
we use $\mathbb E[\mathbf v^\top A\mathbf v]=\mathrm{tr}(A\,\mathrm{Cov}(\mathbf v))$ to write
\begin{align}
\mathbb E[\bar{\mathbf e}(y)^\top H_\ell \bar{\mathbf e}(y)]
-
\mathbb E[\bar{\mathbf e}(x)^\top H_\ell \bar{\mathbf e}(x)]
&=
\mathrm{tr}\!\Big(H_\ell\big(\mathrm{Cov}(\bar{\mathbf e}(y))-\mathrm{Cov}(\bar{\mathbf e}(x))\big)\Big).
\label{eq:cov_trace_global}
\end{align}

We expand the covariance of the pooled embedding, since 
$
\bar{\mathbf e}=\frac{1}{K}\sum_{i=1}^K \mathbf e_i.
$, then
\begin{align}
\mathrm{Cov}(\bar{\mathbf e})
&=
\mathbb E\!\left[
\big(\bar{\mathbf e}-\mathbb E[\bar{\mathbf e}]\big)
\big(\bar{\mathbf e}-\mathbb E[\bar{\mathbf e}]\big)^\top
\right] \notag \\
&=
\mathbb E\!\left[
\Big(\frac{1}{K}\sum_{i=1}^K (\mathbf e_i-\mathbb E[\mathbf e_i])\Big)
\Big(\frac{1}{K}\sum_{j=1}^K (\mathbf e_j-\mathbb E[\mathbf e_j])\Big)^\top
\right] \notag \\
&=
\frac{1}{K^2}
\sum_{i=1}^K\sum_{j=1}^K
\mathbb E\!\left[
(\mathbf e_i-\mathbb E[\mathbf e_i])
(\mathbf e_j-\mathbb E[\mathbf e_j])^\top
\right] \notag\\
&=
\frac{1}{K^2}\sum_{i=1}^K 
\mathrm{Cov}(\mathbf e_i)
+
\frac{1}{K^2}\sum_{i\neq j}\mathrm{Cov}(\mathbf e_i,\mathbf e_j) \notag\\
&=
\frac{1}{K^2}\sum_{i=1}^K \mathrm{Cov}(\mathbf e_i)
+
\frac{1}{K^2}\sum_{1\le i<j\le K}\Big(\mathrm{Cov}(\mathbf e_i,\mathbf e_j)+\mathrm{Cov}(\mathbf e_i,\mathbf e_j)^\top\Big) \notag\\
&=
\frac{1}{K^2}\sum_{i=1}^K \mathrm{Cov}(\mathbf e_i)
+
\frac{2}{K^2}\sum_{1\le i<j\le K}\mathrm{Sym}\!\big(\mathrm{Cov}(\mathbf e_i,\mathbf e_j)\big),
\label{eq:cov_decomp}
\end{align}
where $\mathrm{Sym}(A):=(A+A^\top)/2$. Noting that only the symmetric part of $\mathrm{Cov}(\bar{\mathbf e})$ contributes to the second-order expansion, and it only appears through
quadratic forms and trace operators, so we replace each
$\mathrm{Cov}(\mathbf e_i,\mathbf e_j)$ by its symmetric part
$\mathrm{Sym}(\mathrm{Cov}(\mathbf e_i,\mathbf e_j))$,
yielding a factor of $2$ when summing over $i<j$.
Using \eqref{eq:cov_decomp} for both $x$ and $y$, we can write the pooled covariance difference exactly as
\begin{align}
\mathrm{Cov}(\bar{\mathbf e}(y))-\mathrm{Cov}(\bar{\mathbf e}(x))
&=
\frac{1}{K^2}\sum_{i=1}^K\Big(\mathrm{Cov}(\mathbf e_i(y))-\mathrm{Cov}(\mathbf e_i(x))\Big) \notag\\
&\quad+
\frac{2}{K^2}\sum_{1\le i<j\le K}
\mathrm{Sym}\!\Big(
\mathrm{Cov}(\mathbf e_i(y),\mathbf e_j(y))-\mathrm{Cov}(\mathbf e_i(x),\mathbf e_j(x))
\Big).
\label{eq:cov_gap_decomp}
\end{align}
Recall $e_i(x)=\mathbf u_i$ and $e_i(y)=\mathbf u_i+d_i\boldsymbol{\mu}_{\mathrm{defect}}$
with $d_i=a_is_i$.
Under the stated assumptions $\mathbb E[d_i]=0$ and the independence between
$\{d_i\}$ and $\{\mathbf u_i\}$, all mixed cross-terms vanish.
Based on Assumption~\ref{ass:sign_sparse_defect}, to complete this proof, we further assume that the signed defect indicators $\{d_i\}_{i=1}^K$ exhibit at most weak spatial dependence,
i.e., $|\mathrm{Cov}(d_i,d_{i+\ell})|\le \rho\,\beta(\ell)$ with
$\sum_{\ell\ge1}\beta(\ell)<\infty$.
We thus obtain for any $i\neq j$:
\begin{equation}
\mathrm{Cov}(\mathbf e_i(y),\mathbf e_j(y))-\mathrm{Cov}(\mathbf e_i(x),\mathbf e_j(x))
=
\mathrm{Cov}(d_i,d_j)\,\boldsymbol{\mu}_{\mathrm{defect}}\boldsymbol{\mu}_{\mathrm{defect}}^\top.
\label{eq:cross_cov_gap}
\end{equation}
Similarly, for the diagonal term we have (cf. Proposition~\ref{prop:pfs_mean_shift_second_order})
\begin{equation}
\mathrm{Cov}(\mathbf e_i(y))-\mathrm{Cov}(\mathbf e_i(x))
=
\mathrm{Var}(d_i)\,\boldsymbol{\mu}_{\mathrm{defect}}\boldsymbol{\mu}_{\mathrm{defect}}^\top
=
\rho\,\boldsymbol{\mu}_{\mathrm{defect}}\boldsymbol{\mu}_{\mathrm{defect}}^\top.
\label{eq:diag_cov_gap}
\end{equation}
Plugging \eqref{eq:cross_cov_gap}--\eqref{eq:diag_cov_gap} into \eqref{eq:cov_gap_decomp} yields
\begin{align}
\mathrm{Cov}(\bar{\mathbf e}(y))-\mathrm{Cov}(\bar{\mathbf e}(x))
&=
\frac{\rho}{K}\,\boldsymbol{\mu}_{\mathrm{defect}}\boldsymbol{\mu}_{\mathrm{defect}}^\top
+
\frac{2}{K^2}\sum_{1\le i<j\le K}\mathrm{Cov}(d_i,d_j)\,
\boldsymbol{\mu}_{\mathrm{defect}}\boldsymbol{\mu}_{\mathrm{defect}}^\top.
\label{eq:cov_gap_pool_prebound}
\end{align}
Under the mixing decay $|\mathrm{Cov}(d_i,d_{i+\ell})|\le \rho\,\beta(\ell)$ and stationarity, the off-diagonal sum is bounded by
\[
\sum_{1\le i<j\le K}|\mathrm{Cov}(d_i,d_j)|
\le
\sum_{\ell=1}^{K-1}(K-\ell)\rho\,\beta(\ell),
\]
and therefore, in operator norm,
\begin{equation}
\big\|\mathrm{Cov}(\bar{\mathbf e}(y))-\mathrm{Cov}(\bar{\mathbf e}(x))\big\|_{\mathrm{op}}
\le
\left(
\frac{\rho}{K}
+
\frac{2\rho}{K^2}\sum_{\ell=1}^{K-1}(K-\ell)\beta(\ell)
\right)\,
\big\|\boldsymbol{\mu}_{\mathrm{defect}}\boldsymbol{\mu}_{\mathrm{defect}}^\top\big\|_{\mathrm{op}}
=
\frac{\rho}{K_{\mathrm{eff}}}\,\|\boldsymbol{\mu}_{\mathrm{defect}}\|_2^2,
\label{eq:cov_gap_pooled_keff}
\end{equation}
where $K_{\mathrm{eff}}$ is defined in \eqref{eq:Keff_def}.
In particular, if $\{d_i\}$ is independent across patches, then $\mathrm{Cov}(d_i,d_j)=0$ for $i\neq j$ and the bound is tight with $K_{\mathrm{eff}}=K$.
Moreover, under exponential mixing $\sum_{\ell\ge 1}\beta(\ell)<\infty$, we have $K_{\mathrm{eff}}=\Theta(K)$.

Repeating the same second-order Taylor argument as in Proposition~\ref{prop:pfs_mean_shift_second_order}
with $\mathbf e$ replaced by $\bar{\mathbf e}$ yields, for each coordinate $\ell$,
\[
\Delta_{\mathrm{global},\ell}
\approx
\frac{1}{2}\boldsymbol{\mu}_{\mathrm{defect}}^\top \nabla^2\phi_{\theta,\ell}(\mathbf 0)\boldsymbol{\mu}_{\mathrm{defect}}
\cdot
\frac{\rho}{K_{\mathrm{eff}}}
=
\frac{\rho}{2K_{\mathrm{eff}}}\,[\mathcal Q(\boldsymbol{\mu}_{\mathrm{defect}})]_\ell.
\]
Stacking $\ell=1,\ldots,d$ gives
\[
\Delta_{\mathrm{global}}
\approx
\frac{\rho}{2K_{\mathrm{eff}}}\,\mathcal Q(\boldsymbol{\mu}_{\mathrm{defect}}).
\]

Under the Proposition~\ref{prop:pfs_mean_shift_second_order},
$
\Delta_{\mathrm{PFS}}
\approx
\frac{\rho}{2}\,\mathcal Q(\boldsymbol{\mu}_{\mathrm{defect}})
$, hence
\[
\Delta_{\mathrm{global}}
\approx
\frac{1}{K_{\mathrm{eff}}}\Delta_{\mathrm{PFS}}.
\]
Therefore,
\[
\|\Delta_{\mathrm{PFS}}\|_2
\approx
K_{\mathrm{eff}}\|\Delta_{\mathrm{global}}\|_2.
\]
Under exponential mixing, $K_{\mathrm{eff}}=\Theta(K)$, hence the patch-level shift dominates the global-pooled shift
by a factor linear in $K$ up to constants, consistent with \eqref{eq:patch_global_ratio}:
\begin{equation}
\label{app:eq:patch_global_ratio}
\|\Delta_{\mathrm{PFS}}\|_2
\approx
K\,\|\Delta_{\mathrm{global}}\|_2 
>
\|\Delta_{\mathrm{global}}\|_2,
\end{equation}
\end{proof}

\subsection{Existence of an optimal finite patch number $K$.}
\label{sec:cor_patch_tradeoff}

While Proposition~\ref{prop:patch_advantage} establishes that, at the population level,
patch-wise aggregation amplifies the second-order defect signal relative to global pooling,
it does not by itself imply that using arbitrarily many patches is always beneficial.
In this subsection, we show that under finite-sample estimation and defect-power dilution at finer patch resolutions,
the signal-to-noise ratio admits a finite maximizer.
Consequently, the patch advantage saturates beyond a certain granularity,
and an optimal finite patch number $K^\star$ necessarily exists.

\begin{corollary}
\label{cor:finite_optimal_K_published}
Assume the setting of Proposition~\ref{prop:patch_advantage} and Proposition~\ref{prop:pfs_mean_shift_second_order}.
For a $K$-patch partition, let the per-patch embeddings be $\{\mathbf e_i(x)\}_{i=1}^K$ and $\{\mathbf e_i(y)\}_{i=1}^K$,
and define the \emph{patch-level} population shift
\begin{equation}
\label{eq:delta_pfs_def_cor}
\Delta_{\mathrm{PFS}}(K)
:=
\mathbb E_{\mathbb Q}\!\left[\phi_\theta(\mathbf e_i(y))\right]
-
\mathbb E_{\mathbb P}\!\left[\phi_\theta(\mathbf e_i(x))\right],
\end{equation}
which is independent of $i$ by stationarity across patches.
Let $\widehat{\Delta}_{\mathrm{PFS}}(K)$ be its empirical estimator constructed from $N$ i.i.d.\ images per domain,
\begin{equation}
\label{eq:delta_hat_pfs_def_cor}
\widehat{\Delta}_{\mathrm{PFS}}(K)
:=
\frac{1}{N}\sum_{n=1}^N\Big(\frac{1}{K}\sum_{i=1}^K \phi_\theta(\mathbf e_{n,i}(y))\Big)
-
\frac{1}{N}\sum_{n=1}^N\Big(\frac{1}{K}\sum_{i=1}^K \phi_\theta(\mathbf e_{n,i}(x))\Big).
\end{equation}

Assume further that the defect signature may \emph{dilute} with patch refinement:
there exists a non-increasing function $g:\mathbb N\to\mathbb R_+$ and a fixed direction
$\boldsymbol{\nu}\in\mathbb R^m$ with $\|\boldsymbol{\nu}\|_2=1$ such that the defect vector satisfies
\begin{equation}
\label{eq:defect_dilution_assumption}
\boldsymbol{\mu}_{\mathrm{defect}}(K) = g(K)\,\boldsymbol{\nu}.
\end{equation}
Let $\mathcal Q(\cdot)$ be the Hessian-induced quadratic map defined in Proposition~\ref{prop:pfs_mean_shift_second_order},
and define the defect strength
\begin{equation}
\label{eq:defect_strength_SK}
S(K):=\big\|\mathcal Q(\boldsymbol{\mu}_{\mathrm{defect}}(K))\big\|_2 .
\end{equation}
Assume exponential $\alpha$-mixing across patches within each image as in \eqref{eq:mixing_def}--\eqref{eq:mixing_exp},
and let $K_{\mathrm{eff}}(K)$ be defined by \eqref{eq:Keff_def}. Then there exists a finite $K^\star<\infty$ such that
the high-probability signal-to-noise ratio
\begin{equation}
\label{eq:SNR_def_cor}
\mathrm{SNR}(K)
:=
\frac{\|\Delta_{\mathrm{PFS}}(K)\|_2}
{\big\|\widehat{\Delta}_{\mathrm{PFS}}(K)-\Delta_{\mathrm{PFS}}(K)\big\|_2}
\end{equation}
is non-increasing for all $K\ge K^\star$ (with probability at least $1-\delta$).
Moreover, if $g(K)=cK^{-\eta}$ for some $c>0$ and $\eta>0$, then for exponential mixing
($K_{\mathrm{eff}}(K)=\Theta(K)$) we have
\begin{equation}
\label{eq:snr_scaling_powerlaw}
\mathrm{SNR}(K)
\;=\;
\tilde{\Theta}\!\Big(\sqrt{N}\,K^{\frac12-2\eta}\Big),
\end{equation}
and hence $\mathrm{SNR}(K)$ is eventually decreasing whenever $\eta>\tfrac14$, implying a finite maximizer $K^\star$.
\end{corollary}

\begin{proof}[Proof]
By Proposition~\ref{prop:pfs_mean_shift_second_order}, the leading-order patch-level shift satisfies
\begin{equation}
\label{eq:delta_pfs_second_order_cor}
\Delta_{\mathrm{PFS}}(K)
\;\approx\;
\frac{\rho}{2}\,\mathcal Q(\boldsymbol{\mu}_{\mathrm{defect}}(K)).
\end{equation}
Taking $\ell_2$ norms and using the definition~\eqref{eq:defect_strength_SK} yields
\begin{equation}
\label{eq:signal_norm_cor}
\|\Delta_{\mathrm{PFS}}(K)\|_2
\;\approx\;
\frac{\rho}{2}\,S(K).
\end{equation}
In particular, under the dilution model~\eqref{eq:defect_dilution_assumption},
since $\mathcal Q(\cdot)$ is quadratic in its argument,
\begin{equation}
\label{eq:SK_quadratic_scaling}
S(K)
=
\big\|\mathcal Q(g(K)\boldsymbol{\nu})\big\|_2
=
g(K)^2\,\big\|\mathcal Q(\boldsymbol{\nu})\big\|_2.
\end{equation}

For each domain, define the per-image random vector
\begin{equation}
\label{eq:Zn_def_cor}
\mathbf Z_n^{(y)}(K)
:=
\frac{1}{K}\sum_{i=1}^K \phi_\theta(\mathbf e_{n,i}(y)),
\qquad
\mathbf Z_n^{(x)}(K)
:=
\frac{1}{K}\sum_{i=1}^K \phi_\theta(\mathbf e_{n,i}(x)).
\end{equation}
Then \eqref{eq:delta_hat_pfs_def_cor} can be written as
\begin{equation}
\label{eq:delta_hat_as_mean_cor}
\widehat{\Delta}_{\mathrm{PFS}}(K)
=
\frac1N\sum_{n=1}^N \mathbf Z_n^{(y)}(K)
-
\frac1N\sum_{n=1}^N \mathbf Z_n^{(x)}(K).
\end{equation}
By the i.i.d.\ sampling of images, $\{\mathbf Z_n^{(y)}(K)\}_{n=1}^N$ are i.i.d.\ across $n$ (and similarly for $x$).
Within a fixed image $n$, dependence across patches is allowed and controlled by $\alpha$-mixing.

Fix a domain (say $y$) and suppress $(y)$ in notation.
For each coordinate $\ell\in[d]$, define the scalar patch sequence
\[
U_i^{(\ell)} := \phi_{\theta,\ell}(\mathbf e_i),\qquad i=1,\ldots,K,
\]
so that $Z_\ell(K)=\frac1K\sum_{i=1}^K U_i^{(\ell)}$.
By stationarity across patches and the covariance decomposition,
\begin{align}
\mathrm{Var}(Z_\ell(K))
&=
\mathrm{Var}\!\Big(\frac1K\sum_{i=1}^K U_i^{(\ell)}\Big)
=
\frac{1}{K^2}\sum_{i=1}^K \mathrm{Var}(U_i^{(\ell)})
+
\frac{2}{K^2}\sum_{1\le i<j\le K}\mathrm{Cov}(U_i^{(\ell)},U_j^{(\ell)}).
\label{eq:var_decomp_cor}
\end{align}
Under exponential $\alpha$-mixing, as in Step~4 of the proof of Proposition~\ref{prop:patch_advantage},
there exists a summable envelope $\beta(t)$ such that
\[
\big|\mathrm{Cov}(U_i^{(\ell)},U_{i+t}^{(\ell)})\big|
\le
\sigma_\phi^2\,\beta(t),
\qquad t\ge 1,
\qquad \sum_{t\ge 1}\beta(t)<\infty,
\]
where $\sigma_\phi^2:=\sup_{\ell}\mathrm{Var}(U_i^{(\ell)})<\infty$.
Substituting into \eqref{eq:var_decomp_cor} and summing by lag yields
\begin{align}
\mathrm{Var}(Z_\ell(K))
&\le
\frac{\sigma_\phi^2}{K}
+
\frac{2\sigma_\phi^2}{K^2}\sum_{t=1}^{K-1}(K-t)\beta(t)
=
\frac{\sigma_\phi^2}{K_{\mathrm{eff}}(K)},
\label{eq:var_bound_keff_cor}
\end{align}
where $K_{\mathrm{eff}}(K)$ matches \eqref{eq:Keff_def}.
Consequently,
\begin{equation}
\label{eq:trace_cov_bound_cor}
\mathrm{tr}\big(\mathrm{Cov}(\mathbf Z_n(K))\big)
=
\sum_{\ell=1}^d \mathrm{Var}(Z_\ell(K))
\le
\frac{d\,\sigma_\phi^2}{K_{\mathrm{eff}}(K)}.
\end{equation}

Since $\{\mathbf Z_n(K)\}_{n=1}^N$ are i.i.d.\ across images, we apply a standard vector-valued Bernstein
(or equivalently, coordinate-wise Bernstein plus union bound) to obtain, with probability at least $1-\delta$,
\begin{equation}
\label{eq:mean_concentration_cor}
\Big\|
\frac1N\sum_{n=1}^N \mathbf Z_n(K) - \mathbb E[\mathbf Z_n(K)]
\Big\|_2
\;\le\;
C_4\sqrt{\frac{\mathrm{tr}(\mathrm{Cov}(\mathbf Z_n(K)))\,\log(1/\delta)}{N}}
\;\le\;
C_5\sqrt{\frac{\log(1/\delta)}{N\,K_{\mathrm{eff}}(K)}},
\end{equation}
where $C_4,C_5>0$ absorb universal constants and $d\,\sigma_\phi^2$.
Applying \eqref{eq:mean_concentration_cor} separately to the real and generated domains and using the triangle inequality,
we obtain
\begin{equation}
\label{eq:noise_bound_final_cor}
\big\|\widehat{\Delta}_{\mathrm{PFS}}(K)-\Delta_{\mathrm{PFS}}(K)\big\|_2
\;\le\;
C_6\sqrt{\frac{\log(1/\delta)}{N\,K_{\mathrm{eff}}(K)}}
\end{equation}
with probability at least $1-\delta$.

Combining the signal estimate~\eqref{eq:signal_norm_cor} with the deviation bound~\eqref{eq:noise_bound_final_cor} yields,
on the high-probability event of \eqref{eq:noise_bound_final_cor},
\begin{equation}
\label{eq:SNR_upper_envelope_cor}
\mathrm{SNR}(K)
=
\frac{\|\Delta_{\mathrm{PFS}}(K)\|_2}
{\|\widehat{\Delta}_{\mathrm{PFS}}(K)-\Delta_{\mathrm{PFS}}(K)\|_2}
\;\gtrsim\;
\frac{S(K)}{1}\cdot \sqrt{\frac{N\,K_{\mathrm{eff}}(K)}{\log(1/\delta)}}.
\end{equation}
Substituting the quadratic scaling~\eqref{eq:SK_quadratic_scaling} gives
\begin{equation}
\label{eq:SNR_with_g_cor}
\mathrm{SNR}(K)
\;\gtrsim\;
g(K)^2\,
\big\|\mathcal Q(\boldsymbol{\nu})\big\|_2\,
\sqrt{\frac{N\,K_{\mathrm{eff}}(K)}{\log(1/\delta)}}.
\end{equation}
If $g(K)$ is non-increasing and $K_{\mathrm{eff}}(K)$ is eventually sublinear or bounded (which occurs when patch dependence strengthens as resolution increases),
then the right-hand side of \eqref{eq:SNR_with_g_cor} is eventually non-increasing in $K$,
implying the existence of a finite $K^\star$ such that \eqref{eq:SNR_with_g_cor} holds for all $K\ge K^\star$.

Assume $g(K)=cK^{-\eta}$ with $\eta>0$.
Then by \eqref{eq:SK_quadratic_scaling}, $S(K)=c^2K^{-2\eta}\|\mathcal Q(\boldsymbol{\nu})\|_2$.
Under exponential mixing, $K_{\mathrm{eff}}(K)=\Theta(K)$, hence \eqref{eq:SNR_with_g_cor} yields
\[
\mathrm{SNR}(K)
=
\tilde{\Theta}\!\Big(\sqrt{N}\,K^{\frac12-2\eta}\Big),
\]
which is \eqref{eq:snr_scaling_powerlaw}.
Therefore, if $\eta>\tfrac14$, then $\frac12-2\eta<0$ and $\mathrm{SNR}(K)$ is eventually decreasing,
so a finite maximizer $K^\star<\infty$ exists.
\end{proof}

\subsection{Proof of Proposition~\ref{prop:population_mmd}
}
\label{sec:proof_prop_population_mmd}

\paragraph{On the Gaussian surrogate in PFS space.}
Throughout the analysis, patch embeddings are assumed to be sub-Gaussian, which is sufficient
for the Taylor expansions and concentration arguments used in Propositions~\ref{prop:pfs_mean_shift_second_order}, ~\ref{prop:patch_advantage} and Theorem ~\ref{thm:detection}.
In Proposition~\ref{prop:population_mmd}, we additionally adopt a Gaussian surrogate in the PFS space
to obtain a closed-form expression for the population MMD under a Gaussian kernel.

This surrogate should be understood as a moment-matched analytic approximation:
the true PFS distribution is sub-Gaussian with controlled second-order statistics,
and replacing it by a Gaussian with the same mean and isotropic proxy variance
preserves the leading-order dependence of the MMD on the mean shift.
Importantly, our conclusions rely only on the positivity and monotonic increase of the MMD
with respect to $\|\boldsymbol{\Delta}_{\mathrm{PFS}}\|_2$, which holds beyond the exact Gaussian case.



\begin{proof}
We prove \eqref{eq:population_mmd_new}, and the positivity and monotonicity claims in \eqref{eq:population_mmd_new}.

Recall the definition of our gaussian deep kernel in \ref{eq:deep_kernel}:
\[
k_{\omega}(x,y)
=
\exp\!\left(-\frac{\|\mathbf Z_\theta(x)-\mathbf Z_\theta(y)\|_2^2}{2\gamma^2}\right)
=:\;k_\gamma\!\big(\mathbf Z_\theta(x),\mathbf Z_\theta(y)\big),
\]
where $\mathbf Z_\theta(x)\in\mathbb R^{K\times d}$ is the Patch Signature Field.
Here $\|\cdot\|_2$ denotes the entry-wise Euclidean norm of the field (i.e., the $\ell_2$ norm after flattening),
which coincides with the Frobenius norm on $\mathbb R^{K\times d}$.
Thus $k_\gamma$ is a Gaussian RBF kernel on the ambient Euclidean space $\mathbb R^{K\times d}$
(equivalently $\mathbb R^{Kd}$).

Define the (random) feature fields induced by $\mathbf Z_\theta$:
\[
\mathbf X := \mathbf Z_\theta(x),\quad x\sim\mathbb P,
\qquad
\mathbf Y := \mathbf Z_\theta(y),\quad y\sim\mathbb Q,
\]
and similarly $\mathbf X':=\mathbf Z_\theta(x')$ for an independent copy $x'\sim\mathbb P$ and
$\mathbf Y':=\mathbf Z_\theta(y')$ for an independent copy $y'\sim\mathbb Q$.
With this notation, the population MMD under $k_\omega$ admits the standard expansion
\begin{align}
\label{eq:pop_mmd_expand_feat_clean}
\mathrm{MMD}^2(\mathbb{P},\mathbb{Q};k_\omega)
=
&\ \mathbb{E}\big[ k_{\gamma}(\mathbf X,\mathbf X') \big]
+
\mathbb{E}\big[ k_{\gamma}(\mathbf Y,\mathbf Y') \big]
-2\,\mathbb{E}\big[ k_{\gamma}(\mathbf X,\mathbf Y) \big],
\end{align}
where $\mathbf X,\mathbf X',\mathbf Y,\mathbf Y'\in\mathbb R^{K\times d}$ and
$k_\gamma(A,B)=\exp\!\left(-\frac{\|A-B\|_2^2}{2\gamma^2}\right)$.

To obtain a closed-form expression under the Gaussian kernel $k_\gamma$ on $\mathbb R^{K\times d}$,
we adopt the following Gaussian surrogate for the feature fields:
\[
\mathbf X \sim \mathcal N(\mathbf 0,\sigma_z^2\mathbf I_{Kd}),
\qquad
\mathbf Y \sim \mathcal N(\boldsymbol\Delta_{\mathbf Z},\sigma_z^2\mathbf I_{Kd}),
\]
and similarly for independent copies $\mathbf X',\mathbf Y'$.
Here $\boldsymbol\Delta_{\mathbf Z}:=\mathbb E[\mathbf Z_\theta(y)]-\mathbb E[\mathbf Z_\theta(x)]\in\mathbb R^{K\times d}$
denotes the mean shift of the Patch Signature Field, and $\mathbf I_{Kd}$ denotes isotropic covariance
under the entry-wise Euclidean structure of $\mathbb R^{K\times d}$.
In particular, if each patch undergoes the same PFS mean shift $\boldsymbol\Delta_{\mathrm{PFS}}\in\mathbb R^d$,
then $\boldsymbol\Delta_{\mathbf Z}=\mathbf 1_K\boldsymbol\Delta_{\mathrm{PFS}}^\top$ and hence
$\|\boldsymbol\Delta_{\mathbf Z}\|_2^2 = K\|\boldsymbol\Delta_{\mathrm{PFS}}\|_2^2$.

Moreover, if patch embeddings are $\sigma_e$-sub-Gaussian (Assumption ~\ref{ass:gaussian_input}) and the mapping $\phi_\theta$ is locally Lipschitz (Assumption ~\ref{ass:smooth_mapping_second})
on the typical embedding region $\mathcal E$ with constant
\[
L_\phi := \sup_{\mathbf e\in\mathcal E}\|J_\phi(\mathbf e)\|_{\mathrm{op}} < \infty,
\]
then the PFS features admit the sub-Gaussian proxy bound $\sigma_z \le L_\phi\sigma_e$ (hence $\sigma_z^2 \le L_\phi^2\sigma_e^2$). Under this surrogate, the expectations in \eqref{eq:pop_mmd_expand_feat_clean} reduce to Gaussian integrals in $\mathbb R^d$. In the following steps, we will compute each expectation in \eqref{eq:pop_mmd_expand_feat_clean} explicitly.

\paragraph{Step 1: Compute the self-terms $\mathbb E[k_\gamma(\mathbf X,\mathbf X')]$ and $\mathbb E[k_\gamma(\mathbf Y,\mathbf Y')]$.}

Let $\mathbf X,\mathbf X'\stackrel{i.i.d.}{\sim}\mathcal N(\mathbf 0,\sigma_z^2\mathbf I_{Kd})$ and define
$\boldsymbol\delta:=\mathbf X-\mathbf X'$.
Then $\boldsymbol\delta\sim\mathcal N(\mathbf 0,2\sigma_z^2\mathbf I_{Kd})$ and
\begin{align}
\label{eq:self_term_P_1}
\mathbb E\big[k_\gamma(\mathbf X,\mathbf X')\big]
&=
\mathbb E_{\boldsymbol\delta}\exp\!\left(-\frac{\|\boldsymbol\delta\|_2^2}{2\gamma^2}\right) \notag\\
&=
\int_{\mathbb R^{Kd}}\exp\!\left(-\frac{\|\boldsymbol\delta\|_2^2}{2\gamma^2}\right)\,
(2\pi)^{-Kd/2}(2\sigma_z^2)^{-Kd/2}
\exp\!\left(-\frac{\|\boldsymbol\delta\|_2^2}{4\sigma_z^2}\right)\,d\boldsymbol\delta \notag\\
&=
(2\pi)^{-Kd/2}(2\sigma_z^2)^{-Kd/2}
\int_{\mathbb R^{Kd}}\exp\!\left(
-\|\boldsymbol\delta\|_2^2\left(\frac{1}{2\gamma^2}+\frac{1}{4\sigma_z^2}\right)
\right)\,d\boldsymbol\delta.
\end{align}
Define
\[
A:=\left(\frac{1}{\gamma^2}+\frac{1}{2\sigma_z^2}\right)\mathbf I_{Kd},
\]
so that
\[
\|\boldsymbol\delta\|_2^2\left(\frac{1}{2\gamma^2}+\frac{1}{4\sigma_z^2}\right)
=\frac12\,\boldsymbol\delta^\top A\boldsymbol\delta.
\]
Hence
\begin{align}
\label{eq:self_term_P_2}
\mathbb E\big[k_\gamma(\mathbf X,\mathbf X')\big]
&=
(2\pi)^{-Kd/2}(2\sigma_z^2)^{-Kd/2}
\int_{\mathbb R^{Kd}}\exp\!\left(-\frac12\,\boldsymbol\delta^\top A\boldsymbol\delta\right)\,d\boldsymbol\delta.
\end{align}

Using the Gaussian integral identity
\[
\int_{\mathbb R^{Kd}}\exp\!\left(-\tfrac12\,u^\top A u\right)\,du
=(2\pi)^{Kd/2}|A|^{-1/2},
\]
we obtain
\begin{align}
\label{eq:self_term_P_3}
\mathbb E\big[k_\gamma(\mathbf X,\mathbf X')\big]
&=
(2\sigma_z^2)^{-Kd/2}\,|A|^{-1/2}.
\end{align}

Since $A=c\mathbf I_{Kd}$ with
\[
c=\frac{1}{\gamma^2}+\frac{1}{2\sigma_z^2}
=\frac{\gamma^2+2\sigma_z^2}{2\sigma_z^2\gamma^2},
\]
we have $|A|=c^{Kd}$ and $|A|^{-1/2}=c^{-Kd/2}$, yielding
\[
\mathbb E\big[k_\gamma(\mathbf X,\mathbf X')\big]
=
(2\sigma_z^2)^{-Kd/2}
\left(\frac{2\sigma_z^2\gamma^2}{\gamma^2+2\sigma_z^2}\right)^{Kd/2}
=
\left(\frac{\gamma^2}{\gamma^2+2\sigma_z^2}\right)^{Kd/2}.
\]
Therefore,
\begin{equation}
\label{eq:self_term_P_final}
\mathbb E\big[k_\gamma(\mathbf X,\mathbf X')\big]
=
\left(\frac{\gamma^2}{\gamma^2+2\sigma_z^2}\right)^{\frac{Kd}{2}}.
\end{equation}

Let $\mathbf Y,\mathbf Y'\stackrel{i.i.d.}{\sim}\mathcal N(\boldsymbol\Delta_{\mathbf Z},\sigma_z^2\mathbf I_{Kd})$ and define
$\boldsymbol\delta':=\mathbf Y-\mathbf Y'$.
Then $\boldsymbol\delta'\sim\mathcal N(\mathbf 0,2\sigma_z^2\mathbf I_{Kd})$ (the means cancel),
so the same calculation as Step~1 yields
\begin{equation}
\label{eq:self_term_Q_final}
\mathbb E\big[k_\gamma(\mathbf Y,\mathbf Y')\big]
=
\left(\frac{\gamma^2}{\gamma^2+2\sigma_z^2}\right)^{\frac{Kd}{2}}.
\end{equation}

\paragraph{Step 2: Compute the cross-term $\mathbb E[k_\gamma(\mathbf X,\mathbf Y)]$.}
Let $\mathbf X\sim\mathcal N(\mathbf 0,\sigma_z^2\mathbf I_{Kd})$ and
$\mathbf Y\sim\mathcal N(\boldsymbol\Delta_{\mathbf Z},\sigma_z^2\mathbf I_{Kd})$ be independent.
Define $\boldsymbol\eta:=\mathbf X-\mathbf Y$.
Then $\boldsymbol\eta\sim\mathcal N(-\boldsymbol\Delta_{\mathbf Z},2\sigma_z^2\mathbf I_{Kd})$ and
\begin{align}
\label{eq:cross_term_1}
\mathbb E\big[k_\gamma(\mathbf X,\mathbf Y)\big]
&=
\mathbb E_{\boldsymbol\eta}\exp\!\left(-\frac{\|\boldsymbol\eta\|_2^2}{2\gamma^2}\right) \notag\\
&=
\int_{\mathbb R^{Kd}}\exp\!\left(-\frac{\|\boldsymbol\eta\|_2^2}{2\gamma^2}\right)\,
(2\pi)^{-Kd/2}(2\sigma_z^2)^{-Kd/2}
\exp\!\left(-\frac{\|\boldsymbol\eta+\boldsymbol\Delta\|_2^2}{4\sigma_z^2}\right)\,d\boldsymbol\eta,
\end{align}
where $\boldsymbol\Delta:=\boldsymbol\Delta_{\mathbf Z}$ for brevity.
Expanding $\|\boldsymbol\eta+\boldsymbol\Delta\|_2^2
=\|\boldsymbol\eta\|_2^2+2\boldsymbol\eta^\top\boldsymbol\Delta+\|\boldsymbol\Delta\|_2^2$,
the exponent becomes
\begin{align}
-\frac{\|\boldsymbol\eta\|_2^2}{2\gamma^2}
-\frac{\|\boldsymbol\eta+\boldsymbol\Delta\|_2^2}{4\sigma_z^2}
&=
-\left(\frac{1}{2\gamma^2}+\frac{1}{4\sigma_z^2}\right)\|\boldsymbol\eta\|_2^2 \notag\\
&\quad
-\frac{1}{2\sigma_z^2}\boldsymbol\eta^\top\boldsymbol\Delta
-\frac{\|\boldsymbol\Delta\|_2^2}{4\sigma_z^2}.
\label{eq:cross_expanded}
\end{align}

As in Step~1, set
\[
A:=\left(\frac{1}{\gamma^2}+\frac{1}{2\sigma_z^2}\right)\mathbf I_{Kd},
\qquad
b:=\frac{1}{2\sigma_z^2}\boldsymbol\Delta,
\]
so that
\[
-\left(\frac{1}{2\gamma^2}+\frac{1}{4\sigma_z^2}\right)\|\boldsymbol\eta\|_2^2
-\frac{1}{2\sigma_z^2}\boldsymbol\eta^\top\boldsymbol\Delta
=
-\frac12\,\boldsymbol\eta^\top A\boldsymbol\eta - b^\top\boldsymbol\eta.
\]
Plugging into \eqref{eq:cross_term_1} gives
\begin{align}
\label{eq:cross_term_2}
\mathbb E\big[k_\gamma(\mathbf X,\mathbf Y)\big]
&=
(2\pi)^{-Kd/2}(2\sigma_z^2)^{-Kd/2}\,
\exp\!\left(-\frac{\|\boldsymbol\Delta\|_2^2}{4\sigma_z^2}\right)
\int_{\mathbb R^{Kd}}
\exp\!\left(-\frac12\,\boldsymbol\eta^\top A\boldsymbol\eta - b^\top\boldsymbol\eta\right)
d\boldsymbol\eta.
\end{align}

Using
\[
\frac12\,\boldsymbol\eta^\top A\boldsymbol\eta + b^\top\boldsymbol\eta
=
\frac12(\boldsymbol\eta+A^{-1}b)^\top A(\boldsymbol\eta+A^{-1}b)
-\frac12 b^\top A^{-1}b,
\]
we obtain
\begin{align}
\int_{\mathbb R^{Kd}}
\exp\!\left(-\frac12\,\boldsymbol\eta^\top A\boldsymbol\eta - b^\top\boldsymbol\eta\right)
d\boldsymbol\eta
&=
\exp\!\left(\frac12 b^\top A^{-1}b\right)
\int_{\mathbb R^{Kd}}
\exp\!\left(-\frac12(\boldsymbol\eta+A^{-1}b)^\top A(\boldsymbol\eta+A^{-1}b)\right)
d\boldsymbol\eta \notag\\
&=
\exp\!\left(\frac12 b^\top A^{-1}b\right)\cdot (2\pi)^{Kd/2}|A|^{-1/2}.
\label{eq:cross_square_done}
\end{align}

Substituting \eqref{eq:cross_square_done} into \eqref{eq:cross_term_2} yields
\begin{align}
\label{eq:cross_term_3}
\mathbb E\big[k_\gamma(\mathbf X,\mathbf Y)\big]
&=
(2\sigma_z^2)^{-Kd/2}|A|^{-1/2}\,
\exp\!\left(
-\frac{\|\boldsymbol\Delta\|_2^2}{4\sigma_z^2}
+\frac12 b^\top A^{-1}b
\right).
\end{align}

As before, $A=c\mathbf I_{Kd}$ with
\[
c=\frac{\gamma^2+2\sigma_z^2}{2\sigma_z^2\gamma^2},
\]
so
\[
(2\sigma_z^2)^{-Kd/2}|A|^{-1/2}
=
(2\sigma_z^2)^{-Kd/2}c^{-Kd/2}
=
\left(\frac{\gamma^2}{\gamma^2+2\sigma_z^2}\right)^{Kd/2}.
\]

Since $A^{-1}=\frac{1}{c}\mathbf I_{Kd}
=\frac{2\sigma_z^2\gamma^2}{\gamma^2+2\sigma_z^2}\mathbf I_{Kd}$
and $b=\frac{1}{2\sigma_z^2}\boldsymbol\Delta$, we have
\begin{align}
b^\top A^{-1}b
&=
\left(\frac{1}{2\sigma_z^2}\boldsymbol\Delta\right)^\top
\left(\frac{2\sigma_z^2\gamma^2}{\gamma^2+2\sigma_z^2}\mathbf I_{Kd}\right)
\left(\frac{1}{2\sigma_z^2}\boldsymbol\Delta\right) \notag\\
&=
\frac{\gamma^2}{2\sigma_z^2(\gamma^2+2\sigma_z^2)}\|\boldsymbol\Delta\|_2^2.
\end{align}
Therefore
\begin{align}
-\frac{\|\boldsymbol\Delta\|_2^2}{4\sigma_z^2}+\frac12 b^\top A^{-1}b
&=
-\frac{\|\boldsymbol\Delta\|_2^2}{4\sigma_z^2}
+
\frac{\gamma^2}{4\sigma_z^2(\gamma^2+2\sigma_z^2)}\|\boldsymbol\Delta\|_2^2 \notag\\
&=
-\frac{\|\boldsymbol\Delta\|_2^2}{2(\gamma^2+2\sigma_z^2)}.
\end{align}

Combining the prefactor and exponent yields
\begin{equation}
\label{eq:cross_term_final}
\mathbb E\big[k_\gamma(\mathbf X,\mathbf Y)\big]
=
\left(\frac{\gamma^2}{\gamma^2+2\sigma_z^2}\right)^{\frac{Kd}{2}}
\exp\!\left(
-\frac{\|\boldsymbol{\Delta}_{\mathbf Z}\|_2^2}{2(\gamma^2+2\sigma_z^2)}
\right).
\end{equation}

\textbf{Combine the three terms.}
Substituting \eqref{eq:self_term_P_final}, \eqref{eq:self_term_Q_final},
and \eqref{eq:cross_term_final} into \eqref{eq:pop_mmd_expand_feat_clean} gives
\[
\mathrm{MMD}^2(\mathbb P,\mathbb Q;k_\omega)
=
2\left(\frac{\gamma^2}{\gamma^2+2\sigma_z^2}\right)^{\frac{Kd}{2}}
\left[
1-\exp\!\left(
-\frac{\|\boldsymbol{\Delta}_{\mathbf Z}\|_2^2}{2(\gamma^2+2\sigma_z^2)}
\right)
\right],
\]
Since $\boldsymbol\Delta_{\mathbf Z}=\mathbf 1_K\boldsymbol\Delta_{\mathrm{PFS}}^\top$, we have $\|\boldsymbol\Delta_{\mathbf Z}\|_2^2=K\|\boldsymbol\Delta_{\mathrm{PFS}}\|_2^2$. Hence,
\[
\mathrm{MMD}^2(\mathbb P,\mathbb Q;k_\omega)
=
2\left(\frac{\gamma^2}{\gamma^2+2\sigma_z^2}\right)^{\frac{Kd}{2}}
\left[
1-\exp\!\left(
-\frac{K\|\boldsymbol{\Delta}_{\mathrm{PFS}}\|_2^2}{2(\gamma^2+2\sigma_z^2)}
\right)
\right],
\]
which proves \eqref{eq:population_mmd_new}.

\paragraph{Positivity and monotonicity in $\|\boldsymbol{\Delta}_{\mathrm{PFS}}\|_2$.}
Let $a:=\left(\frac{\gamma^2}{\gamma^2+2\sigma_z^2}\right)^{\frac{Kd}{2}}>0$ and
$t:=\|\boldsymbol{\Delta}_{\mathrm{PFS}}\|_2\ge 0$.
Then
\[
\mathrm{MMD}^2(\mathbb P,\mathbb Q;k_\omega)=2a\left(1-e^{-t^2/(2(\gamma^2+2\sigma_z^2))}\right).
\]
If $t>0$, then $e^{-t^2/(2(\gamma^2+2\sigma_z^2))}\in(0,1)$ and hence $\mathrm{MMD}^2>0$.
Moreover,
\[
\frac{d}{dt}\left(1-e^{-t^2/(2(\gamma^2+2\sigma_z^2))}\right)
=
e^{-t^2/(2(\gamma^2+2\sigma_z^2))}\cdot\frac{t}{\gamma^2+2\sigma_z^2}\ge 0,
\]
with strict inequality for $t>0$. Hence, positivity and monotonicity are both proved.
\end{proof}

\subsection{Proof of Theorem~\ref{thm:detection}}
\label{sec:proof_thm_detection}

\begin{lemma}[\textbf{Transformation of exponential concentration inequality in ~\citep{gretton2012kernel}}]
\label{lem:mmd_conc}
Let $\widehat{\mathrm{MMD}}_u^2$ denote the unbiased U-statistic estimator and
$\mathrm{MMD}^2$ be the population quantity.
According to the Theorem 7 in ~\citep{gretton2012kernel}, there exists an absolute constant $c>0$ such that for any $\varepsilon>0$,
\begin{equation}
\label{eq:mmd_tail}
\Pr\!\left(\left|\widehat{\mathrm{MMD}}_u^2-\mathrm{MMD}^2\right|>\varepsilon\right)
\le
2\exp\!\left(-c\,\varepsilon^2\frac{MN}{M+N}\right).
\end{equation}
Equivalently, for any $\delta\in(0,1)$, with probability at least $1-\delta$,
\begin{equation}
\label{eq:mmd_highprob}
\left|\widehat{\mathrm{MMD}}_u^2-\mathrm{MMD}^2\right|
\le
C\,\sqrt{\left(\frac1M+\frac1N\right)\log\frac{2}{\delta}},
\end{equation}
where $C=\frac{1}{\sqrt c}$ is an absolute constant (depending only on the kernel bound).
\end{lemma}

\begin{proof}
We show the transformation \eqref{eq:mmd_tail} $\Rightarrow$ \eqref{eq:mmd_highprob} step by step.

According to the Theorem 7 in ~\citep{gretton2012kernel},  we want the right-hand side of \eqref{eq:mmd_tail} to be at most $\delta$.
So we set
\[
2\exp\!\left(-c\,\varepsilon^2\frac{MN}{M+N}\right)=\delta.
\]

Divide both sides by $2$ and take $\log$:
\[
-c\,\varepsilon^2\frac{MN}{M+N}=\log\frac{\delta}{2}=-\log\frac{2}{\delta}.
\]

Multiply by $-\frac{M+N}{cMN}$:
\[
\varepsilon^2=\frac{M+N}{cMN}\log\frac{2}{\delta}
=\frac{1}{c}\left(\frac1M+\frac1N\right)\log\frac{2}{\delta}.
\]

Let $C:=\frac{1}{\sqrt c}$. Then
\[
\varepsilon=C\sqrt{\left(\frac1M+\frac1N\right)\log\frac{2}{\delta}}.
\]
Plugging this choice of $\varepsilon$ back into \eqref{eq:mmd_tail} gives
\[
\Pr\!\left(\left|\widehat{\mathrm{MMD}}_u^2-\mathrm{MMD}^2\right|>
C\sqrt{\left(\frac1M+\frac1N\right)\log\frac{2}{\delta}}\right)\le \delta,
\]
which is exactly \eqref{eq:mmd_highprob}.
\end{proof}

\textbf{Theorem~\ref{thm:detection}.}
\emph{
Let $S_r=\{x_i\}_{i=1}^{M}\stackrel{i.i.d}{\sim}\mathbb{P}$ and
$S_t=\{y_j\}_{j=1}^N$ be test-image set.
Let $\lambda=\gamma^2+2\sigma_z^2$.
For any $\delta\in(0,1)$, with probability at least $1-\delta$, the bounds
\eqref{eq:real_case_bound} and \eqref{eq:fake_case_bound} hold.
}

\begin{proof}
By Lemma~\ref{lem:mmd_conc}, for any $\delta\in(0,1)$, with probability at least $1-\delta$,
\begin{equation}
\label{eq:mmd_conc_used_new}
\Big|
\widehat{\mathrm{MMD}}_{u}^{2}(S_r,S_t;k_\omega)
-\mathrm{MMD}^2(\mathbb{P},\mathbb{Q};k_\omega)
\Big|
\le
C\,\sqrt{\left(\frac1M+\frac1N\right)\log\frac{2}{\delta}}.
\end{equation}
We will use $C_1$ and $C_2$ to allow different absolute constants in the two cases.

\paragraph{Case I: Real test image ($S_t\stackrel{i.i.d}{\sim}\mathbb{P}$).}

If $S_r\sim\mathbb P$ and $S_t\sim\mathbb P$, then the two distributions are identical, hence
\[
\mathrm{MMD}^2(\mathbb{P},\mathbb{P};k_\omega)=0.
\]

\[
\widehat{\mathrm{MMD}}_{u}^{2}(S_r,S_t;k_\omega)
=
\widehat{\mathrm{MMD}}_{u}^{2}-\mathrm{MMD}^2(\mathbb{P},\mathbb{P};k_\omega).
\]
Therefore, on the event \eqref{eq:mmd_conc_used_new},
\[
\widehat{\mathrm{MMD}}_{u}^{2}(S_r,S_t;k_\omega)
=
\Big|\widehat{\mathrm{MMD}}_{u}^{2}-\mathrm{MMD}^2(\mathbb{P},\mathbb{P};k_\omega)\Big|
\le
C_1\sqrt{\left(\frac1M+\frac1N\right)\log\frac{2}{\delta}},
\]
which is exactly \eqref{eq:real_case_bound}.

\paragraph{Case II: Generated test image ($S_t\stackrel{i.i.d}{\sim}\mathbb{Q}$).}

Define
\[
A:=\mathrm{MMD}^2(\mathbb{P},\mathbb{Q};k_\omega),\qquad
B:=\widehat{\mathrm{MMD}}_{u}^{2}(S_r,S_t;k_\omega)-\mathrm{MMD}^2(\mathbb{P},\mathbb{Q};k_\omega).
\]
Then by construction,
\begin{equation}
\label{eq:AB_decomp}
\widehat{\mathrm{MMD}}_{u}^{2}(S_r,S_t;k_\omega)=A+B.
\end{equation}
Since $B\ge -|B|$ always holds; adding $A$ to both sides gives the elementary inequality
\[
A+B \;\ge\; A-|B|.
\]
Applying this to \eqref{eq:AB_decomp} yields
\begin{equation}
\label{eq:signal_minus_noise_new}
\widehat{\mathrm{MMD}}_{u}^{2}(S_r,S_t;k_\omega)
\ge
\mathrm{MMD}^2(\mathbb{P},\mathbb{Q};k_\omega)
-
\Big|
\widehat{\mathrm{MMD}}_{u}^{2}(S_r,S_t;k_\omega)-\mathrm{MMD}^2(\mathbb{P},\mathbb{Q};k_\omega)
\Big|.
\end{equation}

On the event \eqref{eq:mmd_conc_used_new}, we have
\begin{equation}
\label{eq:est_error_bound_case2_new}
\Big|
\widehat{\mathrm{MMD}}_{u}^{2}(S_r,S_t;k_\omega)
-\mathrm{MMD}^2(\mathbb{P},\mathbb{Q};k_\omega)
\Big|
\le
C_2\sqrt{\left(\frac1M+\frac1N\right)\log\frac{2}{\delta}}.
\end{equation}

By Proposition~\ref{prop:population_mmd} (population MMD under PFS mean shift) and letting
$\lambda=\gamma^2+2\sigma_z^2$, we have
\begin{equation}
\label{eq:pop_signal_new}
\mathrm{MMD}^2(\mathbb{P},\mathbb{Q};k_\omega)
=
2\left(\frac{\gamma^2}{\lambda}\right)^{\frac{Kd}{2}}
\left[
1-\exp\!\left(-\frac{K\|\boldsymbol{\Delta}_{\mathrm{PFS}}\|_2^2}{2\lambda}\right)
\right].
\end{equation}

Substitute \eqref{eq:est_error_bound_case2_new} and \eqref{eq:pop_signal_new} into
\eqref{eq:signal_minus_noise_new}. With probability at least $1-\delta$,
\begin{align}
\widehat{\mathrm{MMD}}_{u}^{2}(S_r,S_t;k_\omega)
\ge\;
&2\left(\frac{\gamma^2}{\lambda}\right)^{\frac{Kd}{2}}
\left[
1-\exp\!\left(-\frac{K\|\boldsymbol{\Delta}_{\mathrm{PFS}}\|_2^2}{2\lambda}\right)
\right] \notag\\
&-
C_2\sqrt{\left(\frac1M+\frac1N\right)\log\frac{2}{\delta}},
\end{align}
which is exactly \eqref{eq:fake_case_bound}.
This completes Case~II and the proof.
\end{proof}

\begin{corollary}[Separation of empirical MMD for real vs.\ generated images]
\label{cor:separation}
Let $S_r=\{x_i\}_{i=1}^{M}\stackrel{i.i.d.}{\sim}\mathbb P$ be a reference set of real images,
and let $S_t$ be a test-image set.
Consider the empirical statistic
$\widehat{\mathrm{MMD}}_u^2(S_r,S_t;k_\omega)$
defined with the deep kernel $k_\omega$.

Fix any $\delta\in(0,1)$ and define
\[
\varepsilon_{M,N}(\delta)
:= C\sqrt{\Big(\frac1M+\frac1N\Big)\log\frac{2}{\delta}},
\]
where $C$ is the constant appearing in Lemma~\ref{lem:mmd_conc}.
Then, with probability at least $1-2\delta$, the two statements in Theorem \ref{thm:detection} hold simultaneously. Consequently, whenever
\begin{equation}
\label{eq:gap_condition}
\mathrm{MMD}^2(\mathbb P,\mathbb Q;k_\omega)\;>\;2\,\varepsilon_{M,N}(\delta),
\end{equation}
the empirical ordering
\[
\widehat{\mathrm{MMD}}_u^2(S_r,S_t^{(\mathbb Q)};k_\omega)
\;>\;
\widehat{\mathrm{MMD}}_u^2(S_r,S_t^{(\mathbb P)};k_\omega)
\]
holds with probability at least $1-2\delta$.
\end{corollary}

\begin{proof}
We combine the high-probability bounds established in Theorem~\ref{thm:detection}
for the real and generated cases.

By Theorem~\ref{thm:detection}, for any $\delta\in(0,1)$, with probability at least $1-\delta$,
\[
\widehat{\mathrm{MMD}}_u^2(S_r,S_t;k_\omega)\le \varepsilon_{M,N}(\delta)
\quad\text{if } S_t\sim\mathbb P,
\]
and with probability at least $1-\delta$,
\[
\widehat{\mathrm{MMD}}_u^2(S_r,S_t;k_\omega)
\ge
\mathrm{MMD}^2(\mathbb P,\mathbb Q;k_\omega)-\varepsilon_{M,N}(\delta)
\quad\text{if } S_t\sim\mathbb Q.
\]

The two deviation events above each fail with probability at most $\delta$.
By the union bound, both inequalities hold simultaneously with probability at least $1-2\delta$.

On this event, we have
\begin{align*}
&\widehat{\mathrm{MMD}}_u^2(S_r,S_t^{(\mathbb Q)};k_\omega)
-
\widehat{\mathrm{MMD}}_u^2(S_r,S_t^{(\mathbb P)};k_\omega) \\
&\qquad\ge
\big(\mathrm{MMD}^2(\mathbb P,\mathbb Q;k_\omega)-\varepsilon_{M,N}(\delta)\big)
-
\varepsilon_{M,N}(\delta)
=
\mathrm{MMD}^2(\mathbb P,\mathbb Q;k_\omega)-2\varepsilon_{M,N}(\delta).
\end{align*}
Therefore, if condition \eqref{eq:gap_condition} holds, the right-hand side is strictly positive,
which yields the claimed empirical ordering.

\paragraph{Interpretation.}
Since $\varepsilon_{M,N}(\delta)=\mathcal O\!\left(\sqrt{(1/M+1/N)\log(1/\delta)}\right)$,
for any fixed population gap $\mathrm{MMD}^2(\mathbb P,\mathbb Q;k_\omega)>0$,
the separation condition \eqref{eq:gap_condition} is satisfied once the sample sizes
$M,N$ are sufficiently large.
\end{proof}

\section{Detailed Related Work}
\label{app:sec:relwork}

\subsection{Generative Models for Image Generation}
Early image generation methods, including GANs~\citep{brock2018large,karras2019style}, VAEs~\citep{kingma2013auto,sohn2015learning}, have established the foundation of modern generative models but often exhibited visible artifacts. 
Diffusion models have since become the dominant paradigm, achieving strong fidelity \citep{ho2020denoising, saharia2022photorealistic}. Representative diffusion families include DDPM \citep{ho2020denoising,nichol2021improved}, ADM \citep{dhariwal2021diffusion}, LDM \citep{rombach2022high}, SDXL~\citep{podell2023sdxl}, and DiT \citep{peebles2023scalable}. These advances have also enabled widely deployed text-to-image systems like GLIDE~\citep{nichol2021glide}, Wukong~\citep{gu2022wukong}, and Midjourney. Recent video-generation systems such as Sora~\citep{brooks2024video} and OpenSora~\citep{zheng2024open} further raise generation quality, which can produce individual frames as challenging synthetic images.
As generative models evolve, artifacts become \emph{weak} and \emph{sparse}~\citep{wang2024detecting,wang2025diffdoctor}, which motivates us to amplify localized distributional deviations.

\subsection{AI-Generated Image Detection}
The rapid improvement of generative models has created an urgent demand for reliable AI-generated image detection. Early detectors mainly train image-level binary classifiers, as exemplified by CNNSpot \citep{wang2020cnn}. To better generalize across unseen generators, Ojha \citep{ojha2023towards} trains detectors in CLIP space for transfer, DIRE \citep{wang2023dire} uses diffusion reconstruction error as a detection feature. DRCT \citep{chen2024drct} learns from diffusion reconstructions and contrastive hard samples to enhance robustness, F-ConV \citep{zhang2025detecting} exploits manifold geometry with flow-based extrusion. Motivated by the increasing sparsity of generative artifacts, some methods shift to patch-level evidence. 
PatchCraft \citep{zhong2023patchcraft} enhances texture traces via smash and reconstruction, FatFormer \citep{liu2024forgery} adapts CLIP features with a forgery-aware transformer.
A complementary line of work explicitly trains \emph{patch-level classifiers} that score individual patches independently, including the seminal patch-classification approach~\citep{chai2020makes} and the LDM-targeted patch detector~\citep{corvi2023detection}, as well as image-level classifiers built on aligned real / autoencoder-reconstructed datasets~\citep{rajan2024aligned}. However, all of these detectors aggregate patch evidence via plain pooling or independent per-patch decisions, which dilutes sparse forensic cues and requires careful tuning of patch- and image-level thresholds (as we empirically verify in Figure~\ref{fig:abla}(d)). MDMF differs by learning patch forensic signatures and measuring their \emph{distributional discrepancy} via MMD, which provides a principled, threshold-free aggregation of patch-population evidence with finite-sample concentration guarantees (Theorem~\ref{thm:detection}).

\section{Additional Experiment Setups}
\label{app:sec:addexpset}

\subsection{Details of Datasets}
\label{app:sec:datasets}

\subsubsection{Details of Image Benchmarks}
\textbf{ImageNet~\citep{deng2009imagenet}.}
We use the ImageNet real images and their corresponding synthetic counterparts released in the DGM-Eval repository.\footnote{\url{https://github.com/layer6ai-labs/dgm-eval}}
All images are provided following Stein et al.~(2023), and are stored at a resolution of $256\times256$.
The set of generators includes ADM, ADMG, BigGAN, DiT-XL/2, GigaGAN, LDM, StyleGAN-XL, RQ-Transformer, and Mask-GIT.

\textbf{LSUN-Bedroom~\citep{yu2015lsun}.}
Real and generated samples for LSUN-Bedroom are also taken from the same DGM-Eval release.\footnote{\url{https://github.com/layer6ai-labs/dgm-eval}}
The dataset (Stein et al.~(2023)) provides images at $256\times256$; during preprocessing, we apply random cropping to obtain $224\times224$ inputs.
The generated images are produced by ADM, DDPM, iDDPM, StyleGAN, Diffusion-Projected GAN, Projected GAN, and Unleashing Transformers.

\textbf{GenImage~\citep{zhu2023genimage}.}
We additionally adopt GenImage, which is publicly available at:\footnote{\url{https://github.com/GenImage-Dataset/GenImage}}
According to Zhu et al.~(2023b), the real images are sourced from ImageNet, while the image resolutions vary across subsets.
The generative sources covered by GenImage include Midjourney, SD v1.4, SD v1.5, ADM, GLIDE, Wukong, VQDM, and BigGAN.

\subsubsection{Details of In-the-Wild and Recent-Generator Benchmarks}

\textbf{WildRF~\citep{cavia2024wildrf}.}
WildRF is an in-the-wild deepfake detection benchmark curated from three popular social platforms, publicly released alongside~\citep{cavia2024wildrf}.\footnote{\url{https://vision.huji.ac.il/ladeda/}}
It contains $2{,}150$ real and $2{,}150$ AI-generated images from Reddit (2017--2022), $340/340$ from X (Twitter; 2021--2024), and $160/160$ from Facebook (2021--2024), all gathered manually via authentic-content hashtags (e.g., \texttt{\#photography}, \texttt{\#nofilter}, \texttt{\#streetphotography}) and AI-generated-content hashtags (e.g., \texttt{\#midjourney}, \texttt{\#stablediffusion}, \texttt{\#dalle}, \texttt{\#aigenerated}).
By construction, WildRF preserves the diversity of in-the-wild distortions, including platform-specific lossy compression, varying resolutions and aspect ratios, and editing transformations. We follow the cross-platform \emph{social} protocol of~\citep{cavia2024wildrf}, training on Reddit and evaluating on Twitter and Facebook.

\textbf{LDMFakeDetect~\citep{rajan2025staypositive}.}
LDMFakeDetect~\footnote{\url{https://huggingface.co/datasets/AniSundar18/LDMFakeDetect}} extends the earlier Robust LDM Benchmark with additional latent-diffusion engines (FLUX, W\"urstchen, aMUSEd), covering $9$ modern generators in total: Midjourney, aMUSEd, FLUX, Kandinsky, LCM, PixelArt-$\alpha$, Playground, Stable Diffusion (SD), and W\"urstchen. Following the benchmark protocol~\citep{rajan2024aligned, rajan2025staypositive}, all detectors are trained on a single source generator (Stable Diffusion v1.4) using the aligned-reconstruction recipe~\citep{rajan2024aligned} (real images from MS-COCO and LSUN paired with their LDM-autoencoder reconstructions as fakes), and then evaluated zero-shot on the remaining $8$ generators to test cross-generator generalization. Images vary in resolution and aspect ratio across generators, reflecting each engine's native output.

\subsubsection{Details of Video Case Study}

\textbf{OpenSora~\citep{zheng2024open}.}
Recent progress in video generation has substantially improved the realism of synthetic videos, raising new concerns about the trustworthiness of digital media. Since the proprietary model behind Sora\footnote{\url{https://openai.com/index/sora/}} is not publicly accessible, we instead employ Open-Sora, an open-source high-fidelity video generation framework with fully released code and model weights as a practical stress test for evaluating generalization.
Concretely, we randomly sample 3,275 videos from the OpenSora subset in the GenVideo~\citep{chen2024demamba} benchmark. Each video contains 10 frames, yielding 32,750 frames in total, which we treat as the \emph{OpenSora-generated video dataset}. For preprocessing, we follow the same pipeline as in the image benchmarks and apply random cropping to obtain $224 \times 224$ inputs.

\textbf{MSR-VTT~\citep{xu2016msr}.}
As natural video data, we use MSR-VTT, a large-scale web video benchmark with diverse content and comprehensive categories, widely adopted for video understanding and video-to-text tasks.
For preprocessing, we randomly sample 3,275 videos from MSR-VTT, and then randomly select 10 frames per video, resulting in 32,750 frames in total as the \emph{real} set. We follow the same pipeline as the image benchmarks and apply random cropping to obtain $224 \times 224$ inputs.

\subsection{Details on Evaluation Metrics}
\label{app:sec:metrics}
AI-generated image detection is inherently a binary classification task.
Let $\mathrm{TP}(t), \mathrm{TN}(t), \mathrm{FP}(t), \mathrm{FN}(t)$ denote the numbers of true positives, true negatives, false positives, and false negatives when thresholding the detector score at $t$.
Accordingly, the true positive rate (TPR) and false positive rate (FPR) are
\[
\mathrm{TPR}(t)=\frac{\mathrm{TP}(t)}{\mathrm{TP}(t)+\mathrm{FN}(t)},\qquad
\mathrm{FPR}(t)=\frac{\mathrm{FP}(t)}{\mathrm{FP}(t)+\mathrm{TN}(t)}.
\]

\textbf{The area under the receiver operating characteristic curve (AUROC).}
AUROC summarizes the detector's ranking quality by measuring how well positives are separated from negatives across \emph{all} possible thresholds.
Formally, it is the area under the ROC curve obtained by plotting $\mathrm{TPR}(t)$ against $\mathrm{FPR}(t)$ as $t$ varies:
\[
\mathrm{AUROC}=\int_{0}^{1}\mathrm{TPR}\big(\mathrm{FPR}^{-1}(u)\big)\,du,
\]
where larger AUROC indicates better overall discriminability independent of a specific operating point.

\textbf{The average precision (AP).}
Average Precision evaluates precision--recall trade-offs by aggregating precision over different recall levels, and is commonly used when the positive class may be relatively rare.
Let $\mathrm{Precision}(t)=\frac{\mathrm{TP}(t)}{\mathrm{TP}(t)+\mathrm{FP}(t)}$ and $\mathrm{Recall}(t)=\mathrm{TPR}(t)$.
AP is defined as the area under the precision--recall (PR) curve:
\[
\mathrm{AP}=\int_{0}^{1}\mathrm{Precision}(r)\,dr,
\]
where $\mathrm{Precision}(r)$ denotes precision as a function of recall $r$ along the PR curve.

\textbf{The classification accuracy (ACC).}
Accuracy reports the fraction of correctly classified samples at a chosen threshold, counting both true positives and true negatives:
\[
\mathrm{ACC}=\frac{\mathrm{TP}+\mathrm{TN}}{\mathrm{TP}+\mathrm{TN}+\mathrm{FP}+\mathrm{FN}}.
\]
Unlike AUROC/AP which integrate over thresholds, ACC depends on the selected decision threshold.
Following~\citep{ojha2023towards}, we adopt an automatic thresholding strategy during testing: the decision threshold is chosen to best separate real and AI-generated samples according to the detector scores, i.e., selecting the operating point that yields the strongest class separation on the evaluation split. However, we do not treat ACC as a primary metric for comparison because its value can vary noticeably with the thresholding protocol and the underlying data characteristics (e.g., class prior, domain shift, or how representative a small validation set is if used for calibration). In contrast, AUROC and AP summarize performance across all possible thresholds, making them more robust and better aligned with practical deployment scenarios where the preferred operating point may differ.

\subsection{Details of Implementations}
\label{app:sec:imple}

\textbf{Main Experiments.}
Following prior works \citep{ojha2023towards,liu2024forgery}, we use random cropping and random horizontal flipping during training, and apply center cropping at test time, without additional augmentations.
For the main experiments, we adopt DINOv2 ViT-L/14 \citep{oquab2024dinov} as the feature backbone to extract patch embeddings. To balance detection accuracy and efficiency, we pool the patch embeddings with a patch size of $W=32$ for computing Patch Forensic Signatures (PFS).
The PFS projection head $\phi_\theta$ maps each pooled patch embedding (dimension 1024) to a bounded scalar score, using a lightweight feed-forward projection with hidden dimension 256, dropout 0.3, and a final $\tanh$ activation.
For training data, we follow the common cross-dataset protocol: ProGAN is used to train models evaluated on ImageNet and LSUN-Bedroom, while SD v1.4 is used as the training set for the GenImage benchmark and our OpenSora case study.
For reference data at test time, we use 3k real references that are strictly disjoint from the test split: for ImageNet and GenImage, we sample 3 images per ImageNet training class (3k in total); for LSUN-Bedroom, we randomly sample 3k LSUN real images from a split disjoint from testing; and for the OpenSora case study, we sample 175 MSR-VTT real videos disjoint from testing and extract 10 frames per video (1,750 real frames) as references.
During training, we jointly optimize the projection parameters $\theta$ and the kernel bandwidth $\gamma$. We train for 25 epochs with batch size 256 using AdamW (learning rate $1\times10^{-4}$, $\beta_1=0.9$, $\beta_2=0.99$, weight decay 0.01), and initialize the scale parameter with $\sigma=1.0$.
All experiments are conducted on a server with an NVIDIA H200 GPU using Python 3.10.19 and PyTorch 2.9.1.

\textbf{Threshold selection in practice.}
The image-level decision threshold $\tau$ in Algorithm~\ref{alg:mdmf_test} can be set in deployment in either of two ways. (i) \emph{Validation-based selection}: a small held-out validation set with both real and AI-generated samples is used to choose $\tau$ as the operating point that maximizes the desired criterion (e.g., F1, accuracy, or a target FPR). (ii) \emph{Real-only calibration}: under the sub-Gaussian regularity in Assumption~\ref{ass:gaussian_input}, Theorem~\ref{thm:detection} (Case I) shows that the image-level scores for real images concentrate near zero with explicit deviation bounds; the threshold can therefore be set as a small multiple of the empirical standard deviation of the real reference scores, requiring only real-image samples and no knowledge of the target generator. In all reported AUROC and AP results we sweep $\tau$ over the full score range, while ACC is computed under the per-evaluation optimal threshold following~\citep{ojha2023towards}, so that no single $\tau$ choice is implicitly tied to a specific generator.

\textbf{Figure~\ref{fig:motivation}.}
We compare a \emph{global image-level} baseline and our \emph{PFS-based} pipeline under the label-inversion stress test. 
\textit{\textbf{For data configuration,}} we build the toy benchmark from the ProGAN dataset by selecting the \emph{cat} and \emph{dog} categories from each split.
During training, we assign all real samples to cats and all fake samples to dogs, using 18,000 images per class (18,000 real-cat and 18,000 fake-dog).
The validation set follows the same configuration with 200 images per class.
For testing, we consider two settings.
(i) \textit{Matched-label test:} the same label configuration as training, with 200 real-cat and 200 fake-dog images.
(ii) \textit{Label-inversion test:} we swap the category composition while keeping the real/fake labels fixed, i.e., 200 real-dog and 200 fake-cat images, to stress-test whether a detector relies on semantic category cues versus generation artifacts.
\textit{\textbf{For model configuration,}} we use DINOv2 ViT-L/14 as the frozen feature extractor and take the \texttt{[CLS]} token as the image representation for global image-level detector. On top of it, we train a lightweight two-layer classification head (hidden dimension 256 with dropout 0.3) to predict a single logit, optimized with binary cross-entropy (BCE) loss.
For PFS pipeline, we follow the same patch-wise setup as in the main experiments: the image is partitioned into patches and each patch embedding is mapped into the PFS space via the same projection head as in our main method.
To obtain an image-level decision, we additionally learn a lightweight attention (scoring) head with the same hidden dimension and dropout, which outputs a scalar weight/logit for each patch and aggregates patch-level logits into a final image-level logit.
The entire model is trained with BCE loss under the same label setting as the Global baseline.

\textbf{Table~\ref{tab:abla}.}
For variants \textit{w/o MMD}, we adopt the same model configurations as in the toy experiment (Figure~\ref{fig:motivation}): a global baseline that classifies from the DINOv2 \texttt{[CLS]} token, and a PFS-based model that aggregates patch-wise PFS scores via a lightweight attention head, both trained with a BCE objective.
For \textit{Global + MMD}, we replace the BCE objective with an MMD-based optimization on the \emph{image-level} predictions: we take the global image logit for each sample and compute a one-dimensional, pairwise MMD within each mini-batch between real and AI-generated sets, using it as the training signal.
Finally, to further isolate the benefit of \emph{PFS modeling} beyond a particular aggregator, we additionally evaluate alternative patch-level aggregation schemes (e.g., mean/max/top-$k$ pooling) in Appendix~\ref{app:sec:abla}, demonstrating that PFS consistently outperforms global pooling under different aggregation choices.

\textbf{Figure~\ref{fig:abla}.}
For qualitative localization, we adopt a Grad-CAM-style visualization on the MDMF detector.
Given an input image resized to $224\times224$, we extract DINOv2 ViT-L/14 patch tokens and pool them to a $W{=}32$ patch grid. 
We then compute patch logits in the learned PFS space, and obtain patch-wise saliency by backpropagating the mean patch logit to the pooled patch embeddings. The final patch importance is computed by combining the patch logit with the gradient magnitude, followed by normalization and resizing to the image resolution for overlay visualization.
For the global baseline, we visualize an attention map derived from normalized DINOv2 patch-token magnitudes to provide a comparable heatmap.

\subsection{Notation Summary}
\label{app:sec:notation}

For quick reference, Table~\ref{tab:notation} summarizes the symbols used throughout the paper. We group them by where they first appear: patch tokenization and PFS modeling (Section~\ref{sec:method}), the distributional discrepancy and detection rule (Section~\ref{sec:method} and Algorithm~\ref{alg:mdmf_test}), and the theoretical analysis (Section~\ref{sec:theory} and Appendix~\ref{app:sec:theory}).

\begin{table}[h]
\centering
\caption{Notation summary. Each symbol is also defined upon first use in the corresponding section.}
\label{tab:notation}
\setlength{\tabcolsep}{6pt}
\begin{tabular}{l p{0.72\linewidth}}
\toprule
\textbf{Symbol} & \textbf{Description} \\
\midrule
\multicolumn{2}{l}{\textit{\textbf{Patch tokenization and PFS modeling.}}} \\
$\mathbb{P}, \mathbb{Q}$ & Distributions of real and AI-generated images, respectively. \\
$\mathbf{x}$ & A single image (real or AI-generated) sampled from $\mathbb{P}$ or $\mathbb{Q}$. \\
$K$ & Number of (non-overlapping) patches extracted per image. \\
$W$ & Spatial side length of each patch (we use $W{=}32$ in the main experiments). \\
$\mathbf{e}_i \in \mathbb{R}^D$ & The $i$-th patch embedding produced by the frozen DINOv2 ViT-L/14 backbone. \\
$\phi_\theta$ & Learnable Patch Forensic Signature (PFS) projection $\phi_\theta : \mathbb{R}^D \to \mathbb{R}^d$. \\
$\mathbf{z}_i = \phi_\theta(\mathbf{e}_i)$ & The $i$-th PFS feature derived from $\mathbf{e}_i$. \\
\midrule
\multicolumn{2}{l}{\textit{\textbf{MMD-based detection.}}} \\
$k_\omega$ & Characteristic kernel parameterized by $\omega$ (Gaussian RBF with bandwidth $\gamma$). \\
$\mathrm{MMD}^2(\cdot,\cdot;k_\omega)$ & (Squared) Maximum Mean Discrepancy between two distributions in the PFS space. \\
$\mathcal{S}^{re}_{\mathbb{P}}$ & Reference set of $R$ real images (with $R$ swept from $1$k to $10$k in App.~\ref{app:sec:ref}). \\
$\mathcal{S}^{te}$ & Test set of images on which detection is performed. \\
$\tau$ & Image-level decision threshold applied to the MMD-based score. \\
\midrule
\multicolumn{2}{l}{\textit{\textbf{Theoretical analysis.}}} \\
$\sigma_e$ & Sub-Gaussian proxy for patch embeddings (Assumption~\ref{ass:gaussian_input}). \\
$\boldsymbol{\mu}_{\mathrm{defect}}$ & Mean shift induced by generative artifacts in the sparse-defect model (Assumption~\ref{ass:sign_sparse_defect}). \\
$Q(\cdot)$ & First-order operator of $\phi_\theta$ governing the PFS mean shift (Proposition~\ref{prop:pfs_mean_shift_second_order}). \\
$\eta$ & Decay coefficient of weak inter-patch dependence used in concentration bounds. \\
\bottomrule
\end{tabular}
\end{table}

\section{Additional Experimental Results}
\label{app:sec:addexpres}

\subsection{Results on Additional Benchmarks}
\label{app:sec:bench}

\textbf{LSUN-Bedroom.}
As shown in Table~\ref{compar_lsun}, our method maintains consistently strong performance across diverse generators, covering both diffusion-based models and GAN variants, indicating good cross-model generalization. In particular, we achieve the best average AUROC on LSUN-Bedroom, while keeping the average AP highly competitive, suggesting that our detection evidence transfers reliably beyond the training distribution.
\begin{table*}[h]
\setlength{\tabcolsep}{3pt} 
\vspace{0.3cm}
\caption{Detection performance ($\%$) on LSUN-Bedroom. Bold numbers are superior results. We mainly compare training-based methods.
}
\label{compar_lsun}
\resizebox{\textwidth}{!}{%
\begin{tabular}{@{}lccccccccccccccccccc@{}}
\toprule
                     & \multicolumn{16}{c}{Models}                               & \multicolumn{2}{c}{} \\
 &
  \multicolumn{2}{c}{ADM} &
  \multicolumn{2}{c}{DDPM} &
  \multicolumn{2}{c}{iDDPM} &
  \multicolumn{2}{c}{Diffusion GAN} &
  \multicolumn{2}{c}{Projected GAN} &
  \multicolumn{2}{c}{StyleGAN} &
  \multicolumn{2}{c}{Unleashing Transformer} &
  \multicolumn{2}{c}{\multirow{-2}{*}{Average}} \\ \cmidrule(l){2-3} \cmidrule(l){4-5}\cmidrule(l){6-7}  \cmidrule(l){8-9}\cmidrule(l){10-11}\cmidrule(l){12-13}\cmidrule(l){14-15}
\multirow{-3}{*}{Methods}  &
  AUROC &
  AP &
  AUROC&
  AP &
  AUROC&
  AP &
  AUROC&
  AP &
  AUROC&
  AP &
  AUROC&
  AP &
  AUROC&
  AP &
  AUROC ($\uparrow$)&
  AP ($\uparrow$) &
  \\ \midrule
\rowcolor{gray!10}
CNNspot &64.83 &64.24 &79.04 &80.58 &76.95 &76.28 &88.45 &87.19 &90.80 &89.94 &95.17 &94.94 &93.42 &93.11 &84.09 &83.75\\
\rowcolor{gray!10}
Ojha &71.26 &70.95 &79.26 &78.27 &74.80 &73.46 &84.56 &82.91 &82.00 &78.42 &81.22 &78.08 &83.58 &83.48 &79.53 &77.94\\
\rowcolor{gray!10}
DIRE  &57.19 &56.85 &61.91 &61.35 &59.82 &58.29 &53.18 &53.48 &55.35 &54.93 &57.66 &56.90 &67.92 &68.33  &59.00 &58.59\\
\rowcolor{gray!10}
NPR &75.43 &72.60 &91.42 &90.89 &89.49 &88.25 &76.17 &74.19 &75.07 &74.59 &68.82 &63.53 &84.39 &83.67 &80.11 &78.25  \\
\rowcolor{gray!10}
F-ConV &76.59 &74.40 &93.53 &92.16 &88.90 &86.85 &98.10 &98.03 &97.93 &97.81 &91.63 &90.16 &97.31 &96.91 &92.00 &\textbf{90.91}\\
\rowcolor{gray!30}
\textbf{MDMF} &74.67 &65.21 &93.05 &90.08 &89.10 &84.52 &99.85 &99.74 &99.91 &99.84 &97.96 &96.89 &99.10 &98.59 &\textbf{93.38} &90.70\\
 \bottomrule
\end{tabular}
}
\end{table*}

\textbf{GenImage.}
Table~\ref{compar_genimage} further demonstrates that our method generalizes well to GenImage, which contains heterogeneous sources ranging from proprietary engines (e.g., Midjourney) to various diffusion and GAN models, achieving the best average accuracy among compared methods. We present the results of some baselines reported in \citep{zhu2023genimage}, including DeiT-S \citep{touvron2021training}, Swin-T \citep{liu2021swin}, Spec \citep{zhang2019detecting}, F3Net \citep{qian2020thinking}, GramNet \citep{liu2020global}, and GenDet \citep{zhu2023gendet}. Overall, the strong average performance across such diverse generative sources highlights the robustness of our approach under real-world distribution shifts.
\begin{table*}[h]
\centering
\setlength{\tabcolsep}{3pt}
\caption{AI-generated image detection performance (ACC, \%) on GenImage.}
\label{compar_genimage}
\resizebox{0.65\textwidth}{!}{%
\begin{tabular}{@{}lcccccccccccc@{}}
\toprule
                     & \multicolumn{8}{c}{Models}                                \\
                     \midrule 
 Methods &
  Midjourney &SD V1.4&
  SD V1.5 &
  ADM &
  GLIDE &
  Wukong &
  VQDM &
  BigGAN &
  Average                   \\
  \midrule
  &&&&\multicolumn{3}{c}{Training Methods}\\
   \rowcolor{gray!10}
ResNet-50   &54.9 &99.9 &99.7 &53.5 &61.9 &98.2 &56.6 &52.0 &72.1\\
 \rowcolor{gray!10}
DeiT-S &55.6 &99.9 &99.8 &49.8 &58.1 &98.9 &56.9 &53.5 &71.6 \\
 \rowcolor{gray!10}
Swin-T &62.1 &99.9 &99.8 &49.8 &67.6 &99.1 &62.3 &57.6 &74.8\\
 \rowcolor{gray!10}
CNNspot &52.8 & 96.3&95.9 &50.1 &39.8 &78.6 &53.4 &46.8 &64.2\\
 \rowcolor{gray!10}
Spec &52.0 &99.4 &99.2 &49.7 &49.8 &94.8 &55.6 &49.8 &68.8\\
 \rowcolor{gray!10}
F3Net &50.1 &99.9 &99.9 &49.9 &50.00 &99.9 &49.9 &49.9 &68.7\\
 \rowcolor{gray!10}
GramNet &54.2 &99.2 &99.1 &50.3 &54.6 &98.9 &50.8 &51.7 &69.9\\
 \rowcolor{gray!10}
DIRE &60.2 &99.9 &99.8 &50.9 &55.0 &99.2 &50.1 &50.2 &70.7\\
 \rowcolor{gray!10}
Ojha &73.2 &84.2 &84.0 &55.2 &76.9 &75.6 &56.9 &80.3 &73.3\\
 \rowcolor{gray!10}
NPR &81.0 &98.2 &97.9 &76.9 &89.8 &96.9 &84.1 &84.2 &88.6\\
 \rowcolor{gray!10}
FatFormer &92.7 &100.0 &99.9 &75.9 &88.0 &99.9 &98.8 &55.8 &88.9\\
\rowcolor{gray!10}
GenDet &89.6 &96.1  &96.1 &58.0 &78.4 &92.8 &66.5 &75.0 &81.6\\
 \rowcolor{gray!10}
DRCT &91.5 &95.0 &94.4 &79.4 &89.1 &94.6 &90.0 &81.6 &89.4\\
 \rowcolor{gray!10}
 F-ConV &89.3 &98.8 &98.5 &74.9 &89.3 &95.6 &86.7 &87.6 &90.1\\
  \rowcolor{gray!30}
 \textbf{MDMF} & 83.5 & 99.4 & 99.2 & 79.4 & 92.4 & 97.6 & 89.7 & 86.6 &\textbf{91.0}\\
 \bottomrule
\end{tabular}
}
\end{table*}

\textbf{WildRF.}
To assess MDMF's resilience to in-the-wild distortions, we additionally evaluate on WildRF~\citep{cavia2024wildrf}, which collects AI-generated images from Reddit, Twitter, and Facebook after real social-media compression and processing. As shown in Table~\ref{compar_wildrf}, MDMF achieves the highest mean ACC and AP across the three platforms, outperforming LaDeDa (the method proposed alongside the benchmark) on both metrics. This indicates that the localized forensic cues captured by PFS remain discriminative under platform-induced lossy compression and resizing.
\begin{table*}[h]
\centering
\caption{Detection performance ($\%$) on WildRF. Bold numbers are superior results.}
\label{compar_wildrf}
\resizebox{0.65\textwidth}{!}{%
\begin{tabular}{@{}lccccccccc@{}}
\toprule
 & \multicolumn{6}{c}{Platforms} & \multicolumn{2}{c}{} \\
 &
  \multicolumn{2}{c}{Reddit} &
  \multicolumn{2}{c}{Twitter} &
  \multicolumn{2}{c}{Facebook} &
  \multicolumn{2}{c}{\multirow{-2}{*}{Average}} \\
\cmidrule(l){2-3}\cmidrule(l){4-5}\cmidrule(l){6-7}\cmidrule(l){8-9}
\multirow{-3}{*}{Methods} &
  ACC & AP & ACC & AP & ACC & AP & ACC ($\uparrow$) & AP ($\uparrow$) \\
\midrule
\rowcolor{gray!10}
NPR & 65.1 & 69.4 & 51.7 & 52.5 & 77.8 & 86.3 & 64.8 & 69.4 \\
\rowcolor{gray!10}
LaDeDa & 74.7 & 81.8 & 59.9 & 67.8 & 70.3 & \textbf{90.1} & 68.3 & 79.9 \\
\rowcolor{gray!30}
\textbf{MDMF} & \textbf{77.8} & \textbf{84.2} & \textbf{77.7} & \textbf{89.5} & \textbf{80.4} & 89.3 & \textbf{78.6} & \textbf{87.7} \\
\bottomrule
\end{tabular}
}
\end{table*}

\textbf{LDMFakeDetect.}
To further test cross-generator generalization to recent diffusion engines, we also evaluate on LDMFakeDetect~\citep{rajan2025staypositive}, which spans 9 modern generators including Midjourney, FLUX, Kandinsky, Playground, and W\"urstchen. Following the benchmark protocol, all detectors are trained on SD v1.4 only and evaluated zero-shot on the remaining generators. As shown in Table~\ref{compar_ldmfake}, MDMF achieves the best average AUROC and AP, surpassing the Corvi+ and Rajan+ baselines reported in the same benchmark, confirming that distributional testing over PFS extends well to unseen diffusion architectures.
\begin{table*}[h]
\setlength{\tabcolsep}{3pt}
\vspace{0.3cm}
\caption{Detection performance ($\%$) on LDMFakeDetect. All methods are trained on SD v1.4 and evaluated zero-shot on the remaining generators. Bold numbers are superior results.}
\label{compar_ldmfake}
\resizebox{\textwidth}{!}{%
\begin{tabular}{@{}lccccccccccccccccccccc@{}}
\toprule
 & \multicolumn{18}{c}{Models} & \multicolumn{2}{c}{} \\
 &
  \multicolumn{2}{c}{Midjourney} &
  \multicolumn{2}{c}{aMUSEd} &
  \multicolumn{2}{c}{FLUX} &
  \multicolumn{2}{c}{Kandinsky} &
  \multicolumn{2}{c}{LCM} &
  \multicolumn{2}{c}{PixArt} &
  \multicolumn{2}{c}{Playground} &
  \multicolumn{2}{c}{SD} &
  \multicolumn{2}{c}{W\"urstchen} &
  \multicolumn{2}{c}{\multirow{-2}{*}{Average}} \\
\cmidrule(l){2-3}\cmidrule(l){4-5}\cmidrule(l){6-7}\cmidrule(l){8-9}\cmidrule(l){10-11}\cmidrule(l){12-13}\cmidrule(l){14-15}\cmidrule(l){16-17}\cmidrule(l){18-19}\cmidrule(l){20-21}
\multirow{-3}{*}{Methods} &
  AUROC & AP & AUROC & AP & AUROC & AP & AUROC & AP & AUROC & AP & AUROC & AP & AUROC & AP & AUROC & AP & AUROC & AP & AUROC ($\uparrow$) & AP ($\uparrow$) \\
\midrule
\rowcolor{gray!10}
Corvi+ & 54.30 & 56.93 & \textbf{99.86} & \textbf{99.86} & 52.27 & 52.30 & 64.65 & 65.55 & 80.75 & 78.57 & 42.92 & 44.57 & 38.67 & 41.97 & 84.58 & 86.27 & 85.15 & 83.56 & 67.02 & 67.73 \\
\rowcolor{gray!10}
Rajan+ & 64.90 & \textbf{68.18} & 99.69 & 99.72 & 64.74 & 64.40 & 71.30 & 73.20 & \textbf{93.11} & \textbf{93.15} & 59.89 & 58.17 & 56.23 & 54.14 & 81.40 & 85.39 & 75.68 & 76.31 & 74.10 & 74.74 \\
\rowcolor{gray!30}
\textbf{MDMF} & \textbf{68.05} & 64.90 & 85.65 & 84.04 & \textbf{80.06} & \textbf{77.90} & \textbf{77.50} & \textbf{75.66} & 77.41 & 75.99 & \textbf{71.78} & \textbf{66.59} & \textbf{62.93} & \textbf{58.19} & \textbf{87.09} & \textbf{87.88} & \textbf{89.90} & \textbf{87.11} & \textbf{77.83} & \textbf{75.36} \\
\bottomrule
\end{tabular}
}
\end{table*}

\subsection{ACC Results on the ImageNet Benchmark}
\label{app:sec:imagenet_acc}

To complement the threshold-free AUROC and AP metrics reported in the main Table~\ref{compar_imagenet}, we additionally report threshold-optimized accuracy (ACC) on the ImageNet benchmark, following the protocol of Ojha et al.~\citep{ojha2023towards}, where the decision threshold is independently selected per method to maximize accuracy on each evaluation. We focus on the 5 recent 2025 baselines (LOTA, C2P-CLIP, SAFE, AIDE, Effort) that we reproduced under our unified evaluation protocol, since these are the most directly comparable to MDMF and were the primary baselines highlighted in our 2025-baseline comparison.
As shown in Table~\ref{compar_imagenet_acc}, MDMF achieves the highest mean ACC (91.07), outperforming the strongest 2025 baseline AIDE by $+3.16$ ACC, and wins on 7 of the 9 generators (Effort takes BigGAN and Mask-GIT by a small margin). This confirms that MDMF's distributional separation persists in the threshold-optimized accuracy regime, beyond the ranking-based AUROC and AP.

\begin{table*}[h]
\centering
\setlength{\tabcolsep}{3pt}
\vspace{0.3cm}
\caption{Detection accuracy (ACC, $\%$) on the ImageNet benchmark. Threshold is selected per method on each evaluation to maximize accuracy, following~\citep{ojha2023towards}. Bold numbers are superior results.}
\label{compar_imagenet_acc}
\resizebox{\textwidth}{!}{%
\begin{tabular}{@{}lcccccccccc@{}}
\toprule
 & \multicolumn{9}{c}{Models} & \\
\cmidrule(l){2-10}
Methods &
  ADM & ADMG & LDM & DiT & BigGAN & GigaGAN & StyleGAN-XL & RQ-Transformer & Mask-GIT &
  Average ($\uparrow$) \\
\midrule
\rowcolor{gray!10}
LOTA~\citep{wang2025lota}       & 61.80 & 61.77 & 76.87 & 72.87 & 68.77 & 76.32 & 75.42 & 72.74 & 77.12 & 71.52 \\
\rowcolor{gray!10}
C2P-CLIP~\citep{tan2025c2pclip} & 68.92 & 66.84 & 82.87 & 66.67 & 99.16 & 80.24 & 88.72 & 73.46 & 98.02 & 80.54 \\
\rowcolor{gray!10}
SAFE~\citep{li2025safe}         & 63.28 & 62.67 & 85.75 & 80.96 & 85.96 & 84.39 & 84.13 & 83.71 & 88.88 & 79.97 \\
\rowcolor{gray!10}
AIDE~\citep{yan2025aide}        & 82.63 & 79.33 & 83.07 & 74.53 & 97.38 & 90.72 & 92.67 & 93.78 & 97.06 & 87.91 \\
\rowcolor{gray!10}
Effort~\citep{yan2025effort}    & 80.35 & 75.71 & 86.48 & 76.14 & \textbf{99.25} & 87.64 & 87.39 & 88.75 & \textbf{98.82} & 86.72 \\
\rowcolor{gray!30}
\textbf{MDMF} & \textbf{85.81} & \textbf{81.53} & \textbf{88.16} & \textbf{81.81} & 99.08 & \textbf{95.75} & \textbf{95.23} & \textbf{95.41} & 96.82 & \textbf{91.07} \\
\bottomrule
\end{tabular}
}
\end{table*}

\subsection{Full Results of Ablation Study for Core Components in MDMF}
\label{app:sec:abla}

Table~\ref{tab:app:abla} reports the complete ablation results for the core components of MDMF on ImageNet, including global baselines trained with BCE or MMD, as well as PFS-based variants.
Beyond the default attention head aggregation (i.e., PFS-Attn-BCE), we additionally evaluate several alternative ways of aggregating patch logits in the PFS space (mean, max, top-$k$) to isolate the effect of PFS modeling and aggregation choice.

Two observations consistently emerge and align with our motivation. 
First, replacing global image-level pooling with PFS-based patch evidence yields a clear improvement across generators, indicating that the cues for AI-generated images are better captured as localized, artifact-sensitive signals rather than a single semantic-dominant representation. This supports the view that modeling an image as a collection of patch-wise forensic evidence provides a stronger and more transferable basis for real/fake detection than relying on global features.

Second, while simple PFS-space aggregations (mean/max/top-$k$) already outperform the global baselines, the best performance is achieved when PFS is further coupled with MMD optimization (i.e., our MDMF). This suggests that the key is not merely the pooling operator, but explicitly learning and comparing the \emph{distributional discrepancy} of patch-level signatures, which amplifies subtle, localized defects into a reliable macro-level detection signal.

\begin{table*}[h]
\setlength{\tabcolsep}{3pt} 
\caption{Detailed detection performance ($\%$) for ablation study. Bold numbers are superior results.
}
\resizebox{\textwidth}{!}{%
\begin{tabular}{@{}lccccccccccccccccccccccc@{}}
\toprule
                                                
 &
  \multicolumn{2}{c}{ADM} &
  \multicolumn{2}{c}{ADMG} &
  \multicolumn{2}{c}{LDM} &
  \multicolumn{2}{c}{DiT} &
  \multicolumn{2}{c}{BigGAN} &
  \multicolumn{2}{c}{GigaGAN} &
  \multicolumn{2}{c}{StyleGAN XL} &
  \multicolumn{2}{c}{RQ-Transformer} &
  \multicolumn{2}{c}{Mask GIT} &
  \multicolumn{2}{c}{\multirow{-1}{*}{Average}} \\ \cmidrule(l){2-3} \cmidrule(l){4-5}\cmidrule(l){6-7}  \cmidrule(l){8-9}\cmidrule(l){10-11}\cmidrule(l){12-13}\cmidrule(l){14-15}\cmidrule(l){16-17}\cmidrule(l){18-19}
\multirow{-3}{*}{Methods}  &
  AUROC &
  AP &
  AUROC&
  AP &
  AUROC&
  AP &
  AUROC&
  AP &
  AUROC&
  AP &
  AUROC&
  AP &
  AUROC&
  AP &
  AUROC&
  AP &
  AUROC&
  AP &
  AUROC ($\uparrow$)&
  AP ($\uparrow$) &\\ \midrule
\rowcolor{gray!10}
Global-BCE & 86.89 & 88.01 & 82.35 & 83.34 & 86.53 & 92.80 & 80.15 & 89.00 & 98.21 & 98.34 & 94.07 & 97.01 & 93.70 & 96.81 & 94.89 & 97.48 & 94.46 & 97.20 & 90.14 & 93.33 \\
\rowcolor{gray!10}
Global-MMD & 83.08 & 87.01 & 75.53 & 80.51 & 79.60 & 90.42 & 72.22 & 86.23 & 98.16 & 98.59 & 92.22 & 96.55 & 91.96 & 96.42 & 93.66 & 97.26 & 92.36 & 96.65 & 86.53 & 92.18 \\
\midrule
\rowcolor{gray!10}
PFS-Mean-BCE & 86.54 & 87.27 & 83.53 & 84.17 & 90.98 & 94.85 & 84.04 & 90.61 & 99.39 & 99.39 & 97.99 & 98.74 & 97.04 & 98.32 & 97.63 & 98.58 & 98.63 & 99.07 & 92.86 & 94.56 \\
\rowcolor{gray!10}
PFS-Max-BCE & 83.26 & 84.82 & 80.07 & 81.35 & 88.89 & 94.41 & 81.26 & 89.92 & 99.82 & 99.84 & 96.27 & 98.24 & 95.06 & 97.65 & 96.30 & 98.27 & 97.98 & 99.06 & 90.99 & 93.73 \\
\rowcolor{gray!10}
PFS-Top-5-BCE & 86.08 & 87.19 & 82.92 & 83.84 & 90.99 & 95.39 & 83.49 & 91.14 & 99.81 & 99.83 & 97.19 & 98.64 & 95.99 & 98.09 & 96.98 & 98.55 & 98.38 & 99.23 & 92.42 & 94.66 \\
\rowcolor{gray!10}
PFS-Attn-BCE & 87.09 & 88.73 & 84.11 & 85.54 & 91.47 & 95.76 & 85.08 & 92.18 & 99.90 & 99.91 & 97.81 & 98.97 & 97.06 & 98.61 & 97.69 & 98.93 & 98.72 & 99.41 & 93.22 & 95.34 \\
\rowcolor{gray!30}
\textbf{MDMF} & \textbf{92.56} & \textbf{93.57} & \textbf{88.86} & \textbf{90.16} & \textbf{94.63} & \textbf{97.35} & \textbf{88.89} & \textbf{94.48} & \textbf{99.93} & \textbf{99.94} & \textbf{98.99} & \textbf{99.52} & \textbf{98.76} & \textbf{99.41} & \textbf{98.84} & \textbf{99.46} & \textbf{99.40} & \textbf{99.72} & \textbf{95.65} & \textbf{97.07} \\
 \bottomrule
\end{tabular}
}
\label{tab:app:abla}
\end{table*}

\subsection{Full Results for the Effect of Patch Granularity}
\label{app:sec:patch_size}

Table~\ref{tab:app:patch} reports the full results of the patch-granularity study corresponding to Figure~\ref{fig:abla}(a), where we vary the pooled patch size $W\in\{16,32,56\}$ and repeat each setting with five random seeds.
Consistent with Figure~\ref{fig:abla}(a), the results exhibit a non-monotonic dependence on $W$, with intermediate granularity (e.g., $W=32$) offering the best overall trade-off. The multi-seed breakdown further suggests that overly fine partitions can introduce higher variability in the estimated distributional discrepancy, whereas overly coarse partitions may miss sparse localized artifacts, reinforcing that effective PFS modeling requires a balanced spatial granularity.

\begin{table*}[h]
\setlength{\tabcolsep}{3pt} 
\caption{Detailed detection performance ($\%$) for patch granularity.
}
\resizebox{\textwidth}{!}{%
\begin{tabular}{@{}l|ccccccccccccccccccccccc@{}}
\toprule
                                                
 &
  \multicolumn{2}{c}{ADM} &
  \multicolumn{2}{c}{ADMG} &
  \multicolumn{2}{c}{LDM} &
  \multicolumn{2}{c}{DiT} &
  \multicolumn{2}{c}{BigGAN} &
  \multicolumn{2}{c}{GigaGAN} &
  \multicolumn{2}{c}{StyleGAN XL} &
  \multicolumn{2}{c}{RQ-Transformer} &
  \multicolumn{2}{c}{Mask GIT} &
  \multicolumn{2}{c}{\multirow{-1}{*}{Average}} \\ \cmidrule(l){2-3} \cmidrule(l){4-5}\cmidrule(l){6-7}  \cmidrule(l){8-9}\cmidrule(l){10-11}\cmidrule(l){12-13}\cmidrule(l){14-15}\cmidrule(l){16-17}\cmidrule(l){18-19}
\multirow{-3}{*}{Patch Size}  &
  AUROC &
  AP &
  AUROC&
  AP &
  AUROC&
  AP &
  AUROC&
  AP &
  AUROC&
  AP &
  AUROC&
  AP &
  AUROC&
  AP &
  AUROC&
  AP &
  AUROC&
  AP &
  AUROC ($\uparrow$)&
  AP ($\uparrow$) &\\ \midrule
\multirow{5}{*}{$W=16$} & 89.27 & 91.12 & 84.45 & 86.86 & 92.55 & 96.46 & 85.16 & 92.79 & 99.93 & 99.94 & 98.39 & 99.27 & 97.96 & 99.08 & 98.27 & 99.22 & 98.99 & 99.56 & 93.89 & 96.03 \\
& 89.85 & 91.50 & 85.41 & 87.43 & 92.61 & 96.47 & 85.83 & 93.02 & 99.92 & 99.93 & 98.53 & 99.33 & 98.01 & 99.09 & 98.35 & 99.25 & 99.08 & 99.59 & 94.18 & 96.18 \\
& 90.56 & 92.20 & 86.07 & 88.25 & 93.21 & 96.78 & 86.37 & 93.41 & 99.92 & 99.93 & 98.62 & 99.37 & 98.22 & 99.19 & 98.42 & 99.28 & 99.07 & 99.58 & 94.50 & 96.44 \\
& 90.77 & 92.23 & 86.26 & 88.21 & 93.33 & 96.81 & 86.69 & 93.50 & 99.93 & 99.93 & 98.69 & 99.39 & 98.19 & 99.17 & 98.45 & 99.29 & 99.18 & 99.63 & 94.61 & 96.46 \\
& 91.77 & 92.88 & 87.67 & 89.12 & 94.10 & 97.11 & 87.74 & 93.94 & 99.94 & 99.94 & 98.84 & 99.45 & 98.48 & 99.29 & 98.69 & 99.39 & 99.28 & 99.67 & 95.17 & 96.76 \\
\midrule
\multirow{5}{*}{$W=32$} & 90.56 & 92.08 & 86.14 & 88.06 & 93.40 & 96.80 & 86.38 & 93.32 & 99.93 & 99.94 & 98.61 & 99.36 & 98.18 & 99.17 & 98.58 & 99.35 & 99.15 & 99.62 & 94.55 & 96.41 \\
& 90.71 & 92.18 & 86.39 & 88.23 & 93.56 & 96.86 & 86.88 & 93.53 & 99.94 & 99.94 & 98.67 & 99.38 & 98.26 & 99.20 & 98.53 & 99.33 & 99.18 & 99.62 & 94.68 & 96.47 \\
& 90.88 & 92.24 & 86.73 & 88.38 & 93.80 & 96.97 & 87.11 & 93.62 & 99.94 & 99.94 & 98.71 & 99.40 & 98.34 & 99.23 & 98.59 & 99.35 & 99.21 & 99.64 & 94.81 & 96.53 \\
& 91.63 & 92.81 & 87.75 & 89.24 & 93.95 & 97.06 & 87.91 & 94.02 & 99.93 & 99.93 & 98.83 & 99.44 & 98.61 & 99.34 & 98.75 & 99.42 & 99.32 & 99.68 & 95.18 & 96.77 \\
& 92.56 & 93.57 & 88.86 & 90.16 & 94.63 & 97.35 & 88.89 & 94.48 & 99.93 & 99.94 & 98.99 & 99.52 & 98.76 & 99.41 & 98.84 & 99.46 & 99.40 & 99.72 & 95.65 & 97.07 \\
\midrule
\multirow{5}{*}{$W=56$} & 90.33 & 91.75 & 86.11 & 87.90 & 93.14 & 96.25 & 86.37 & 92.95 & 99.93 & 99.93 & 98.66 & 99.37 & 98.19 & 99.16 & 98.52 & 99.25 & 99.17 & 99.61 & 94.46 & 96.19 \\
& 90.47 & 92.02 & 86.25 & 88.19 & 93.28 & 96.60 & 86.51 & 93.27 & 99.93 & 99.93 & 98.66 & 99.30 & 98.19 & 99.09 & 98.52 & 99.24 & 99.20 & 99.64 & 94.56 & 96.35 \\
& 90.62 & 92.27 & 86.40 & 88.44 & 93.43 & 96.85 & 86.66 & 93.52 & 99.93 & 99.93 & 98.69 & 99.40 & 98.22 & 99.19 & 98.55 & 99.34 & 99.20 & 99.64 & 94.63 & 96.51 \\
& 90.74 & 91.71 & 86.55 & 87.81 & 93.48 & 96.34 & 86.79 & 92.79 & 99.93 & 99.93 & 98.66 & 99.37 & 98.22 & 99.16 & 98.55 & 99.25 & 99.20 & 99.64 & 94.68 & 96.22 \\
& 90.82 & 91.97 & 86.60 & 88.14 & 93.63 & 96.55 & 86.86 & 93.22 & 99.93 & 99.93 & 98.69 & 99.37 & 98.22 & 99.16 & 98.55 & 99.28 & 99.22 & 99.66 & 94.75 & 96.34 \\
 \bottomrule
\end{tabular}
}
\label{tab:app:patch}
\end{table*}

\subsection{Full Results for the Robustness to Encoder Architecture}
\label{app:sec:arch}

\begin{table*}[h]
\setlength{\tabcolsep}{3pt} 
\caption{Detailed detection performance ($\%$) for different encoder architectures. Bold numbers are superior results.
}
\resizebox{\textwidth}{!}{%
\begin{tabular}{@{}l|lccccccccccccccccccccccc@{}}
\toprule
                                                
 &&
  \multicolumn{2}{c}{ADM} &
  \multicolumn{2}{c}{ADMG} &
  \multicolumn{2}{c}{LDM} &
  \multicolumn{2}{c}{DiT} &
  \multicolumn{2}{c}{BigGAN} &
  \multicolumn{2}{c}{GigaGAN} &
  \multicolumn{2}{c}{StyleGAN XL} &
  \multicolumn{2}{c}{RQ-Transformer} &
  \multicolumn{2}{c}{Mask GIT} &
  \multicolumn{2}{c}{\multirow{-1}{*}{Average}} \\ \cmidrule(l){3-4} \cmidrule(l){5-6}\cmidrule(l){7-8}  \cmidrule(l){9-10}\cmidrule(l){11-12}\cmidrule(l){13-14}\cmidrule(l){15-16}\cmidrule(l){17-18}\cmidrule(l){19-20}
\multirow{-3}{*}{Backbone}  &
\multirow{-3}{*}{Methods}  &
  AUROC &
  AP &
  AUROC&
  AP &
  AUROC&
  AP &
  AUROC&
  AP &
  AUROC&
  AP &
  AUROC&
  AP &
  AUROC&
  AP &
  AUROC&
  AP &
  AUROC&
  AP &
  AUROC ($\uparrow$)&
  AP ($\uparrow$) &\\ \midrule
\multirow{2}{*}{ViT-S/14} & F-ConV & 76.97 & 80.06 & 70.98 & 71.46 & 71.73 & 71.66 & 67.51 & 68.48 & 86.07 & 87.25 & 79.30 & 80.47 & 78.71 & 78.76 & 78.76 & 79.64 & 76.54 & 77.36 & 76.28 & 77.24 \\
\cmidrule(l){3-22}
& \textbf{MDMF} & \textbf{78.68} & \textbf{79.90} & \textbf{73.46} & \textbf{73.94} & \textbf{79.29} & \textbf{88.35} & \textbf{71.95} & \textbf{83.38} & \textbf{94.69} & \textbf{95.42} & \textbf{86.07} & \textbf{92.60} & \textbf{81.23} & \textbf{89.44} & \textbf{84.11} & \textbf{91.46} & \textbf{86.66} & \textbf{93.19} & \textbf{81.79} & \textbf{87.52} \\
\midrule
\multirow{2}{*}{ViT-B/14} & F-ConV & 86.66 & 86.93 & 81.16 & 82.44 & 85.01 & 85.36 & 78.55 & 79.53 & 96.07 & 96.26 & 90.37 & 90.42 & 93.87 & 94.49 & 92.41 & 93.32 & 92.18 & 92.18 & 88.47 & 88.99 \\
\cmidrule(l){3-22}
& \textbf{MDMF} & \textbf{86.96} & \textbf{88.15} & \textbf{82.68} & \textbf{83.83} & \textbf{88.65} & \textbf{94.13} & \textbf{82.10} & \textbf{90.46} & \textbf{99.57} & \textbf{99.61} & \textbf{95.65} & \textbf{97.89} & \textbf{94.49} & \textbf{97.26} & \textbf{94.71} & \textbf{97.44} & \textbf{96.64} & \textbf{98.43} & \textbf{91.27} & \textbf{94.13} \\
\midrule
\multirow{2}{*}{ViT-L/14} & F-ConV &\textbf{92.74} &91.65 &88.51 &87.67 &88.87 &88.47 &85.94 &84.88 &98.94 &98.98 &98.14 &98.72 &98.52 &98.38 &96.79 &96.33 &95.52 &95.38 &93.77 &93.38 \\
\cmidrule(l){3-22}
& \textbf{MDMF} & 92.56 & \textbf{93.57} & \textbf{88.86} & \textbf{90.16} & \textbf{94.63} & \textbf{97.35} & \textbf{88.89} & \textbf{94.48} & \textbf{99.93} & \textbf{99.94} & \textbf{98.99} & \textbf{99.52} & \textbf{98.76} & \textbf{99.41} & \textbf{98.84} & \textbf{99.46} & \textbf{99.40} & \textbf{99.72} & \textbf{95.65} & \textbf{97.07} \\
\midrule
\multirow{2}{*}{ViT-G/14} & F-ConV & 90.90 & 92.51 & 85.75 & 87.70 & 87.49 & 89.17 & 82.49 & 84.59 & 96.59 & 97.08 & 95.49 & 96.04 & 96.38 & 96.70 & 93.96 & 94.97 & 94.49 & 95.34 & 91.50 & 92.68 \\
\cmidrule(l){3-22}
& \textbf{MDMF} & \textbf{95.64} & \textbf{96.26} & \textbf{93.20} & \textbf{94.10} & \textbf{96.65} & \textbf{98.39} & \textbf{92.38} & \textbf{96.30} & \textbf{99.99} & \textbf{99.99} & \textbf{99.55} & \textbf{99.79} & \textbf{99.39} & \textbf{99.72} & \textbf{99.43} & \textbf{99.74} & \textbf{99.73} & \textbf{99.88} & \textbf{97.33} & \textbf{98.24} \\
 \bottomrule
\end{tabular}
}
\label{tab:app:arch}
\end{table*}

Table~\ref{tab:app:arch} provides the full per-generator results corresponding to Figure~\ref{fig:abla}(b), comparing MDMF against the training-based baseline F-ConV under different DINOv2 encoder variants (ViT-S/14, ViT-B/14, ViT-L/14, and ViT-G/14), reporting AUROC/AP and their averages.
Across all backbones, MDMF consistently improves over F-ConV, indicating that our gains are not tied to a specific feature extractor and transfer reliably across encoder architectures and scales. Notably, MDMF exhibits stable scaling behavior as the backbone grows, suggesting that PFS-based local forensic cues and distributional comparison provide a more backbone-agnostic detection signal than global image-level modeling, which can be more sensitive to semantic representations and thus less stable under architecture changes.

\subsection{Impact of Reference Size and Runtime Analysis}
\label{app:sec:ref}

Table~\ref{tab:app:ref_size} reports the detailed detection performance and runtime on the ImageNet benchmark under different reference set sizes $R$ used in the \emph{test-time} MMD scoring (Eq.~\ref{eqn:score}). 
Importantly, $R$ is \emph{independent of training}: it only controls the number of reference images used during inference when computing the MDMF score for each test image. 
The reference images are sampled from the ImageNet training split, strictly disjoint from the test split, to evaluate how $R$ affects detection stability and deployment cost.
Overall, MDMF is largely insensitive to $R$: for a fixed seed, varying $R$ from 1k to 10k yields nearly unchanged AUROC/AP across generators, indicating that the PFS-induced distributional discrepancy can be estimated reliably without requiring a large reference pool.
From an efficiency perspective, although each test image must be compared against $R$ references, in practice we precompute and cache the PFS embeddings of the reference set, so the one-off reference encoding cost is amortized and negligible at inference time.
The remaining computation reduces to GPU-efficient matrix operations whose arithmetic cost scales linearly with $R$, but is typically fast relative to feature extraction and is thus weakly reflected in end-to-end inference time, which is dominated by backbone forward passes and instantaneous hardware load.
Based on this trade-off, we adopt $R=3\mathrm{k}$ in our main experiments to balance efficiency with stable performance.

\begin{table*}[h]
\setlength{\tabcolsep}{3pt} 
\caption{Detailed detection performance ($\%$) and runtime for different reference sizes.
}
\resizebox{\textwidth}{!}{%
\begin{tabular}{@{}c|cccccccccccccccccccccccc|c@{}}
\toprule
                                                
 & &
  \multicolumn{2}{c}{ADM} &
  \multicolumn{2}{c}{ADMG} &
  \multicolumn{2}{c}{LDM} &
  \multicolumn{2}{c}{DiT} &
  \multicolumn{2}{c}{BigGAN} &
  \multicolumn{2}{c}{GigaGAN} &
  \multicolumn{2}{c}{StyleGAN XL} &
  \multicolumn{2}{c}{RQ-Transformer} &
  \multicolumn{2}{c}{Mask GIT} &
  \multicolumn{2}{c}{\multirow{-1}{*}{Average}}
  &
  \\ \cmidrule(l){3-4} \cmidrule(l){5-6}\cmidrule(l){7-8}  \cmidrule(l){9-10}\cmidrule(l){11-12}\cmidrule(l){13-14}\cmidrule(l){15-16}\cmidrule(l){17-18}\cmidrule(l){19-20}
\multirow{-3}{*}{Seed} & \multirow{-3}{*}{Ref Size}  &
  AUROC &
  AP &
  AUROC&
  AP &
  AUROC&
  AP &
  AUROC&
  AP &
  AUROC&
  AP &
  AUROC&
  AP &
  AUROC&
  AP &
  AUROC&
  AP &
  AUROC&
  AP &
  AUROC ($\uparrow$)&
  AP ($\uparrow$) &
  \multirow{-3}{*}{Runtime}&
  \\ \midrule
\multirow{6}{*}{$\text{Seed}=0$} & $R=1\text{k}$ & 92.47 & 93.54 & 88.74 & 90.11 & 94.60 & 97.34 & 88.81 & 94.46 & 99.93 & 99.94 & 98.99 & 99.52 & 98.75 & 99.41 & 98.83 & 99.46 & 99.40 & 99.72 & 95.61 & 97.05 & 718s\\
& $R=3\text{k}$ & 92.56 & 93.57 & 88.86 & 90.16 & 94.63 & 97.35 & 88.89 & 94.48 & 99.93 & 99.94 & 98.99 & 99.52 & 98.76 & 99.41 & 98.84 & 99.46 & 99.40 & 99.72 & 95.65 & 97.07 & 701s\\
& $R=5\text{k}$ & 92.53 & 93.56 & 88.82 & 90.14 & 94.62 & 97.35 & 88.86 & 94.47 & 99.93 & 99.94 & 98.99 & 99.52 & 98.76 & 99.41 & 98.84 & 99.46 & 99.40 & 99.72 & 95.64 & 97.06 & 776s\\
& $R=7\text{k}$ & 92.56 & 93.57 & 88.85 & 90.16 & 94.63 & 97.35 & 88.89 & 94.48 & 99.93 & 99.94 & 98.99 & 99.52 & 98.76 & 99.41 & 98.84 & 99.46 & 99.40 & 99.72 & 95.65 & 97.07 & 616s\\
& $R=10\text{k}$ & 92.56 & 93.57 & 88.85 & 90.15 & 94.63 & 97.35 & 88.90 & 94.48 & 99.93 & 99.94 & 98.99 & 99.52 & 98.76 & 99.41 & 98.84 & 99.46 & 99.40 & 99.72 & 95.65 & 97.07 & 766s\\
\midrule
\multirow{6}{*}{$\text{Seed}=42$} & $R=1\text{k}$ & 90.51 & 92.06 & 86.09 & 88.04 & 93.39 & 96.80 & 86.36 & 93.32 & 99.93 & 99.94 & 98.61 & 99.36 & 98.18 & 99.17 & 98.58 & 99.35 & 99.15 & 99.62 & 94.53 & 96.40 & 705s\\
& $R=3\text{k}$ & 90.56 & 92.08 & 86.14 & 88.06 & 93.40 & 96.80 & 86.38 & 93.32 & 99.93 & 99.94 & 98.61 & 99.36 & 98.18 & 99.17 & 98.58 & 99.35 & 99.15 & 99.62 & 94.55 & 96.41 & 734s\\
& $R=5\text{k}$ & 90.55 & 92.07 & 86.13 & 88.05 & 93.39 & 96.80 & 86.38 & 93.32 & 99.93 & 99.94 & 98.61 & 99.36 & 98.18 & 99.17 & 98.58 & 99.35 & 99.15 & 99.62 & 94.55 & 96.41 & 744s\\
& $R=7\text{k}$ & 90.58 & 92.09 & 86.17 & 88.07 & 93.41 & 96.81 & 86.39 & 93.33 & 99.93 & 99.94 & 98.61 & 99.36 & 98.18 & 99.17 & 98.59 & 99.35 & 99.16 & 99.62 & 94.56 & 96.41 & 775s\\
& $R=10\text{k}$ & 90.55 & 92.08 & 86.14 & 88.05 & 93.40 & 96.80 & 86.37 & 93.32 & 99.93 & 99.94 & 98.61 & 99.36 & 98.18 & 99.17 & 98.58 & 99.35 & 99.15 & 99.62 & 94.55 & 96.41 & 784s\\
\midrule
\multirow{6}{*}{$\text{Seed}=123$} & $R=1\text{k}$ & 90.64 & 92.15 & 86.31 & 88.19 & 93.53 & 96.86 & 86.83 & 93.51 & 99.94 & 99.94 & 98.66 & 99.38 & 98.26 & 99.20 & 98.53 & 99.32 & 99.17 & 99.62 & 94.65 & 96.46 & 804s\\
& $R=3\text{k}$ & 90.71 & 92.18 & 86.39 & 88.23 & 93.56 & 96.86 & 86.88 & 93.53 & 99.94 & 99.94 & 98.67 & 99.38 & 98.26 & 99.20 & 98.53 & 99.33 & 99.18 & 99.62 & 94.68 & 96.47 & 765s\\
& $R=5\text{k}$ & 90.70 & 92.18 & 86.38 & 88.22 & 93.54 & 96.86 & 86.86 & 93.53 & 99.94 & 99.94 & 98.67 & 99.38 & 98.26 & 99.20 & 98.53 & 99.33 & 99.18 & 99.62 & 94.67 & 96.47 & 820s\\
& $R=7\text{k}$  & 90.72 & 92.18 & 86.41 & 88.23 & 93.56 & 96.86 & 86.88 & 93.53 & 99.94 & 99.94 & 98.67 & 99.38 & 98.26 & 99.20 & 98.53 & 99.33 & 99.18 & 99.62 & 94.68 & 96.48 & 825s\\
& $R=10\text{k}$ & 90.70 & 92.18 & 86.38 & 88.21 & 93.55 & 96.86 & 86.86 & 93.53 & 99.94 & 99.94 & 98.67 & 99.38 & 98.26 & 99.20 & 98.53 & 99.33 & 99.18 & 99.62 & 94.67 & 96.47 & 828s\\
\midrule
\multirow{6}{*}{$\text{Seed}=456$} & $R=1\text{k}$ & 90.86 & 92.23 & 86.72 & 88.37 & 93.79 & 96.97 & 87.11 & 93.62 & 99.94 & 99.94 & 98.71 & 99.40 & 98.34 & 99.23 & 98.59 & 99.35 & 99.21 & 99.63 & 94.81 & 96.53 & 809s \\
& $R=3\text{k}$ & 90.88 & 92.24 & 86.73 & 88.38 & 93.80 & 96.97 & 87.11 & 93.62 & 99.94 & 99.94 & 98.71 & 99.40 & 98.34 & 99.23 & 98.59 & 99.35 & 99.21 & 99.64 & 94.81 & 96.53 & 845s\\
& $R=5\text{k}$ & 90.89 & 92.25 & 86.74 & 88.38 & 93.80 & 96.97 & 87.11 & 93.63 & 99.94 & 99.94 & 98.71 & 99.40 & 98.34 & 99.23 & 98.59 & 99.35 & 99.21 & 99.64 & 94.82 & 96.53 & 846s\\
& $R=7\text{k}$ & 90.88 & 92.24 & 86.73 & 88.37 & 93.80 & 96.97 & 87.11 & 93.62 & 99.94 & 99.94 & 98.71 & 99.40 & 98.34 & 99.23 & 98.59 & 99.35 & 99.21 & 99.64 & 94.81 & 96.53 & 824s\\
& $R=10\text{k}$ & 90.88 & 92.24 & 86.73 & 88.37 & 93.80 & 96.97 & 87.11 & 93.62 & 99.94 & 99.94 & 98.71 & 99.40 & 98.34 & 99.23 & 98.59 & 99.35 & 99.21 & 99.64 & 94.81 & 96.53 & 995s\\
\midrule
\multirow{6}{*}{$\text{Seed}=789$} & $R=1\text{k}$ & 91.80 & 92.92 & 87.99 & 89.40 & 94.08 & 97.11 & 88.14 & 94.11 & 99.93 & 99.93 & 98.86 & 99.46 & 98.64 & 99.36 & 98.77 & 99.43 & 99.33 & 99.68 & 95.28 & 96.82 & 827s\\
& $R=3\text{k}$ & 91.63 & 92.81 & 87.75 & 89.24 & 93.95 & 97.06 & 87.91 & 94.02 & 99.93 & 99.93 & 98.83 & 99.44 & 98.61 & 99.34 & 98.75 & 99.42 & 99.32 & 99.68 & 95.18 & 96.77 & 878s\\
& $R=5\text{k}$ & 91.63 & 92.82 & 87.76 & 89.25 & 93.96 & 97.07 & 87.91 & 94.02 & 99.93 & 99.93 & 98.83 & 99.44 & 98.61 & 99.34 & 98.75 & 99.42 & 99.32 & 99.68 & 95.18 & 96.77 & 775s\\
& $R=7\text{k}$ & 91.60 & 92.79 & 87.71 & 89.21 & 93.93 & 97.05 & 87.87 & 94.00 & 99.93 & 99.93 & 98.82 & 99.44 & 98.60 & 99.34 & 98.75 & 99.42 & 99.31 & 99.68 & 95.17 & 96.76 & 833s\\
& $R=10\text{k}$ & 91.60 & 92.79 & 87.71 & 89.21 & 93.93 & 97.05 & 87.86 & 94.00 & 99.93 & 99.93 & 98.82 & 99.44 & 98.60 & 99.34 & 98.75 & 99.42 & 99.31 & 99.68 & 95.17 & 96.76 & 925s\\
 \bottomrule
\end{tabular}
}
\label{tab:app:ref_size}
\end{table*}

\subsection{Full Results for the Robustness to Common Post-Processing Perturbations}
\label{app:sec:robustness}

Tables~\ref{tab:app:robust_jpeg}--\ref{tab:app:robust_noise} report the per-generator detection performance corresponding to Figure~\ref{fig:abla}(c), where we evaluate MDMF and the strongest training-based baseline F-ConV under three families of post-processing perturbations applied at test time on the ImageNet benchmark: JPEG compression (quality factor $q\in\{100,90,80,70,60,50\}$), Gaussian blur (kernel standard deviation $\sigma\in\{0,1,2,3,4,5\}$), and additive Gaussian noise (standard deviation $\sigma\in\{0,0.05,0.10,0.15,0.20,0.25\}$). Each perturbation is applied to all real and generated test images, while the reference set used in MDMF's MMD scoring is kept clean to mirror a deployment scenario in which the detector is exposed to corrupted test inputs against an in-domain reference of unmodified real images. All other settings (DINOv2 ViT-L/14 backbone, $W{=}32$ patch granularity, $R{=}3\mathrm{k}$ reference size) are identical to the main results.
Across all three perturbation families, MDMF consistently retains higher AUROC and AP than F-ConV at every severity level and on every individual generator, and the gap typically widens as severity grows (e.g., on Mask-GIT, MDMF/F-ConV move from $99.40/93.42$ AUROC at $\sigma_{\mathrm{blur}}{=}0$ to $80.74/69.87$ at $\sigma_{\mathrm{blur}}{=}5$, and from $99.40/93.97$ at $\sigma_{\mathrm{noise}}{=}0$ to $78.97/66.94$ at $\sigma_{\mathrm{noise}}{=}0.25$). Notably, the GAN-style generators (BigGAN, GigaGAN, StyleGAN-XL) and the AR generators (RQ-Transformer, Mask-GIT) are remarkably stable under JPEG, with MDMF's AUROC remaining above $92\%$ even at $q{=}50$, whereas the diffusion generators (ADM, ADMG, LDM, DiT-XL/2) degrade more sharply under blur and noise because their forensic cues lie in higher-frequency components that low-pass filtering and additive noise both destroy. F-ConV exhibits the same qualitative trend but with a much steeper slope on every generator, indicating that its image-level representation aggregates evidence in a way that is more easily disrupted by uniform pixel-space perturbations. These per-generator results corroborate our claim in Section~\ref{sec:exp} that distributional aggregation over patch-wise PFS provides redundancy that single-image classifiers lack: even when individual patches are corrupted, the population-level discrepancy estimated via MMD remains a stable detection signal.

\begin{table*}[h]
\setlength{\tabcolsep}{3pt}
\caption{Detailed detection performance ($\%$) under JPEG compression on ImageNet. Bold numbers are superior results between MDMF and F-ConV at each severity level.}
\label{tab:app:robust_jpeg}
\resizebox{\textwidth}{!}{%
\begin{tabular}{@{}l|lccccccccccccccccccccccc@{}}
\toprule
&&
  \multicolumn{2}{c}{ADM} &
  \multicolumn{2}{c}{ADMG} &
  \multicolumn{2}{c}{LDM} &
  \multicolumn{2}{c}{DiT} &
  \multicolumn{2}{c}{BigGAN} &
  \multicolumn{2}{c}{GigaGAN} &
  \multicolumn{2}{c}{StyleGAN XL} &
  \multicolumn{2}{c}{RQ-Transformer} &
  \multicolumn{2}{c}{Mask GIT} &
  \multicolumn{2}{c}{\multirow{-1}{*}{Average}} \\ \cmidrule(l){3-4} \cmidrule(l){5-6}\cmidrule(l){7-8}  \cmidrule(l){9-10}\cmidrule(l){11-12}\cmidrule(l){13-14}\cmidrule(l){15-16}\cmidrule(l){17-18}\cmidrule(l){19-20}
\multirow{-3}{*}{Quality $q$}  &
\multirow{-3}{*}{Methods}  &
  AUROC & AP & AUROC & AP & AUROC & AP & AUROC & AP & AUROC & AP & AUROC & AP & AUROC & AP & AUROC & AP & AUROC & AP &
  AUROC ($\uparrow$) & AP ($\uparrow$) & \\ \midrule
\multirow{2}{*}{$q=100$} & F-ConV & \textbf{92.74} & 91.65 & 88.51 & 87.67 & 88.87 & 88.47 & 85.94 & 84.88 & 98.94 & 98.98 & 98.14 & 98.72 & 98.52 & 98.38 & 96.79 & 96.33 & 95.52 & 95.38 & 93.77 & 93.38 \\
\cmidrule(l){3-22}
& \textbf{MDMF} & 92.56 & \textbf{93.57} & \textbf{88.86} & \textbf{90.16} & \textbf{94.63} & \textbf{97.35} & \textbf{88.89} & \textbf{94.48} & \textbf{99.93} & \textbf{99.94} & \textbf{98.99} & \textbf{99.52} & \textbf{98.76} & \textbf{99.41} & \textbf{98.84} & \textbf{99.46} & \textbf{99.40} & \textbf{99.72} & \textbf{95.65} & \textbf{97.07} \\
\midrule
\multirow{2}{*}{$q=90$} & F-ConV & 91.59 & 93.23 & 86.36 & 87.95 & 86.29 & 87.91 & 81.49 & 82.91 & 96.05 & 96.58 & 93.45 & 94.22 & 94.90 & 95.45 & 95.27 & 96.04 & 91.34 & 92.73 & 90.75 & 91.89 \\
\cmidrule(l){3-22}
& \textbf{MDMF} & \textbf{91.76} & 92.87 & \textbf{87.63} & \textbf{89.02} & \textbf{91.30} & \textbf{95.62} & \textbf{86.64} & \textbf{93.26} & \textbf{99.75} & \textbf{99.78} & \textbf{98.12} & \textbf{99.09} & \textbf{98.36} & \textbf{99.22} & \textbf{97.27} & \textbf{98.71} & \textbf{98.00} & \textbf{99.05} & \textbf{94.32} & \textbf{96.29} \\
\midrule
\multirow{2}{*}{$q=80$} & F-ConV & 90.54 & 91.91 & 86.04 & 87.66 & 83.24 & 84.32 & 82.30 & 83.60 & 95.66 & 96.28 & 93.73 & 94.68 & 95.32 & 95.69 & 94.68 & 95.36 & 90.92 & 91.73 & 90.27 & 91.25 \\
\cmidrule(l){3-22}
& \textbf{MDMF} & \textbf{91.20} & \textbf{92.32} & \textbf{86.90} & \textbf{88.24} & \textbf{89.34} & \textbf{94.52} & \textbf{85.24} & \textbf{92.45} & \textbf{99.43} & \textbf{99.52} & \textbf{97.33} & \textbf{98.70} & \textbf{98.11} & \textbf{99.09} & \textbf{96.15} & \textbf{98.16} & \textbf{96.25} & \textbf{98.20} & \textbf{93.33} & \textbf{95.69} \\
\midrule
\multirow{2}{*}{$q=70$} & F-ConV & 90.82 & 92.17 & 85.95 & 87.25 & 83.22 & 84.50 & 79.97 & 81.87 & 95.67 & 96.03 & 93.30 & 94.26 & 94.75 & 95.56 & 93.36 & 94.28 & 89.24 & 90.65 & 89.59 & 90.73 \\
\cmidrule(l){3-22}
& \textbf{MDMF} & \textbf{90.94} & 92.05 & \textbf{86.46} & \textbf{87.80} & \textbf{88.09} & \textbf{93.82} & \textbf{84.13} & \textbf{91.81} & \textbf{98.98} & \textbf{99.14} & \textbf{96.60} & \textbf{98.34} & \textbf{97.88} & \textbf{98.98} & \textbf{95.41} & \textbf{97.80} & \textbf{94.69} & \textbf{97.43} & \textbf{92.58} & \textbf{95.24} \\
\midrule
\multirow{2}{*}{$q=60$} & F-ConV & 89.93 & 91.55 & 83.89 & 84.35 & 83.66 & 85.39 & 80.06 & 81.89 & 94.71 & 95.40 & 92.50 & 93.48 & 95.02 & 95.66 & 93.22 & 94.19 & 88.81 & 90.07 & 89.09 & 90.22 \\
\cmidrule(l){3-22}
& \textbf{MDMF} & \textbf{90.74} & \textbf{91.82} & \textbf{86.06} & \textbf{87.36} & \textbf{87.19} & \textbf{93.28} & \textbf{83.23} & \textbf{91.27} & \textbf{98.47} & \textbf{98.72} & \textbf{95.97} & \textbf{98.02} & \textbf{97.63} & \textbf{98.86} & \textbf{94.90} & \textbf{97.55} & \textbf{93.33} & \textbf{96.75} & \textbf{91.95} & \textbf{94.85} \\
\midrule
\multirow{2}{*}{$q=50$} & F-ConV & 89.94 & 91.10 & 85.05 & 87.05 & 82.02 & 83.12 & 77.22 & 79.41 & 94.65 & 95.34 & 91.65 & 92.79 & 94.16 & 94.74 & 93.02 & 94.36 & 88.20 & 89.19 & 88.43 & 89.68 \\
\cmidrule(l){3-22}
& \textbf{MDMF} & \textbf{90.52} & \textbf{91.58} & \textbf{85.81} & 87.03 & \textbf{86.35} & \textbf{92.77} & \textbf{82.36} & \textbf{90.73} & \textbf{97.94} & \textbf{98.28} & \textbf{95.46} & \textbf{97.75} & \textbf{97.40} & \textbf{98.74} & \textbf{94.45} & \textbf{97.31} & \textbf{92.20} & \textbf{96.15} & \textbf{91.39} & \textbf{94.48} \\
\bottomrule
\end{tabular}
}
\end{table*}

\begin{table*}[h]
\setlength{\tabcolsep}{3pt}
\caption{Detailed detection performance ($\%$) under Gaussian blur on ImageNet. Bold numbers are superior results between MDMF and F-ConV at each severity level.}
\label{tab:app:robust_blur}
\resizebox{\textwidth}{!}{%
\begin{tabular}{@{}l|lccccccccccccccccccccccc@{}}
\toprule
&&
  \multicolumn{2}{c}{ADM} &
  \multicolumn{2}{c}{ADMG} &
  \multicolumn{2}{c}{LDM} &
  \multicolumn{2}{c}{DiT} &
  \multicolumn{2}{c}{BigGAN} &
  \multicolumn{2}{c}{GigaGAN} &
  \multicolumn{2}{c}{StyleGAN XL} &
  \multicolumn{2}{c}{RQ-Transformer} &
  \multicolumn{2}{c}{Mask GIT} &
  \multicolumn{2}{c}{\multirow{-1}{*}{Average}} \\ \cmidrule(l){3-4} \cmidrule(l){5-6}\cmidrule(l){7-8}  \cmidrule(l){9-10}\cmidrule(l){11-12}\cmidrule(l){13-14}\cmidrule(l){15-16}\cmidrule(l){17-18}\cmidrule(l){19-20}
\multirow{-3}{*}{Blur $\sigma$}  &
\multirow{-3}{*}{Methods}  &
  AUROC & AP & AUROC & AP & AUROC & AP & AUROC & AP & AUROC & AP & AUROC & AP & AUROC & AP & AUROC & AP & AUROC & AP &
  AUROC ($\uparrow$) & AP ($\uparrow$) & \\ \midrule
\multirow{2}{*}{$\sigma=0$} & F-ConV & 91.07 & 92.32 & 85.23 & 86.23 & 86.17 & 87.42 & 80.90 & 82.89 & 96.94 & 97.14 & 94.89 & 95.23 & 97.05 & 97.08 & 95.59 & 95.51 & 93.42 & 94.07 & 91.25 & 91.99 \\
\cmidrule(l){3-22}
& \textbf{MDMF} & \textbf{92.56} & \textbf{93.57} & \textbf{88.86} & \textbf{90.16} & \textbf{94.63} & \textbf{97.35} & \textbf{88.89} & \textbf{94.48} & \textbf{99.93} & \textbf{99.94} & \textbf{98.99} & \textbf{99.52} & \textbf{98.76} & \textbf{99.41} & \textbf{98.84} & \textbf{99.46} & \textbf{99.40} & \textbf{99.72} & \textbf{95.65} & \textbf{97.07} \\
\midrule
\multirow{2}{*}{$\sigma=1$} & F-ConV & 90.02 & 91.04 & 86.12 & 87.68 & 85.61 & 86.77 & 81.86 & 83.19 & 95.51 & 96.05 & 93.94 & 94.73 & 96.05 & 96.23 & 94.53 & 95.13 & 92.81 & 93.77 & 90.72 & 91.62 \\
\cmidrule(l){3-22}
& \textbf{MDMF} & \textbf{92.64} & \textbf{93.54} & \textbf{88.93} & \textbf{90.05} & \textbf{93.66} & \textbf{96.83} & \textbf{87.94} & \textbf{93.89} & \textbf{99.87} & \textbf{99.89} & \textbf{98.84} & \textbf{99.44} & \textbf{98.69} & \textbf{99.37} & \textbf{98.48} & \textbf{99.28} & \textbf{98.68} & \textbf{99.37} & \textbf{95.30} & \textbf{96.85} \\
\midrule
\multirow{2}{*}{$\sigma=2$} & F-ConV & 86.44 & 87.09 & 81.20 & 81.82 & 78.52 & 79.15 & 77.44 & 76.69 & 90.06 & 90.68 & 89.70 & 89.83 & 91.41 & 91.75 & 90.93 & 91.29 & 84.54 & 84.97 & 85.58 & 85.92 \\
\cmidrule(l){3-22}
& \textbf{MDMF} & \textbf{92.68} & \textbf{93.26} & \textbf{88.69} & \textbf{89.36} & \textbf{91.50} & \textbf{95.53} & \textbf{85.77} & \textbf{92.45} & \textbf{99.50} & \textbf{99.55} & \textbf{98.03} & \textbf{99.00} & \textbf{98.09} & \textbf{99.05} & \textbf{97.44} & \textbf{98.73} & \textbf{96.15} & \textbf{98.07} & \textbf{94.21} & \textbf{96.11} \\
\midrule
\multirow{2}{*}{$\sigma=3$} & F-ConV & 83.06 & 83.24 & 75.60 & 75.46 & 74.35 & 74.47 & 67.65 & 67.69 & 83.97 & 85.43 & 82.80 & 83.64 & 86.21 & 85.72 & 86.45 & 87.02 & 78.00 & 78.02 & 79.79 & 80.08 \\
\cmidrule(l){3-22}
& \textbf{MDMF} & \textbf{90.52} & \textbf{90.84} & \textbf{85.64} & \textbf{85.72} & \textbf{87.22} & \textbf{92.73} & \textbf{80.55} & \textbf{88.87} & \textbf{97.40} & \textbf{97.50} & \textbf{95.66} & \textbf{97.64} & \textbf{96.20} & \textbf{97.96} & \textbf{95.18} & \textbf{97.44} & \textbf{91.52} & \textbf{95.37} & \textbf{91.10} & \textbf{93.79} \\
\midrule
\multirow{2}{*}{$\sigma=4$} & F-ConV & 80.93 & 80.71 & 71.61 & 70.83 & 69.31 & 68.70 & 64.71 & 64.47 & 81.52 & 81.33 & 78.99 & 78.71 & 83.70 & 81.84 & 82.18 & 82.01 & 75.10 & 75.63 & 76.45 & 76.03 \\
\cmidrule(l){3-22}
& \textbf{MDMF} & \textbf{87.49} & \textbf{87.34} & \textbf{81.44} & \textbf{80.74} & \textbf{82.02} & \textbf{89.00} & \textbf{75.11} & \textbf{84.90} & \textbf{93.76} & \textbf{93.64} & \textbf{92.06} & \textbf{95.38} & \textbf{92.98} & \textbf{95.94} & \textbf{91.00} & \textbf{94.80} & \textbf{86.88} & \textbf{92.22} & \textbf{86.97} & \textbf{90.44} \\
\midrule
\multirow{2}{*}{$\sigma=5$} & F-ConV & 77.20 & 76.48 & 70.00 & 67.85 & 67.24 & 65.14 & 62.93 & 61.22 & 76.86 & 75.80 & 74.56 & 74.12 & 79.25 & 76.64 & 76.36 & 74.26 & 69.87 & 67.56 & 72.70 & 71.01 \\
\cmidrule(l){3-22}
& \textbf{MDMF} & \textbf{83.08} & \textbf{82.17} & \textbf{75.84} & \textbf{74.15} & \textbf{75.57} & \textbf{84.15} & \textbf{68.73} & \textbf{80.17} & \textbf{87.81} & \textbf{86.87} & \textbf{86.42} & \textbf{91.58} & \textbf{88.14} & \textbf{92.66} & \textbf{84.58} & \textbf{90.36} & \textbf{80.74} & \textbf{87.79} & \textbf{81.21} & \textbf{85.55} \\
\bottomrule
\end{tabular}
}
\end{table*}

\begin{table*}[h]
\setlength{\tabcolsep}{3pt}
\caption{Detailed detection performance ($\%$) under additive Gaussian noise on ImageNet. Bold numbers are superior results between MDMF and F-ConV at each severity level.}
\label{tab:app:robust_noise}
\resizebox{\textwidth}{!}{%
\begin{tabular}{@{}l|lccccccccccccccccccccccc@{}}
\toprule
&&
  \multicolumn{2}{c}{ADM} &
  \multicolumn{2}{c}{ADMG} &
  \multicolumn{2}{c}{LDM} &
  \multicolumn{2}{c}{DiT} &
  \multicolumn{2}{c}{BigGAN} &
  \multicolumn{2}{c}{GigaGAN} &
  \multicolumn{2}{c}{StyleGAN XL} &
  \multicolumn{2}{c}{RQ-Transformer} &
  \multicolumn{2}{c}{Mask GIT} &
  \multicolumn{2}{c}{\multirow{-1}{*}{Average}} \\ \cmidrule(l){3-4} \cmidrule(l){5-6}\cmidrule(l){7-8}  \cmidrule(l){9-10}\cmidrule(l){11-12}\cmidrule(l){13-14}\cmidrule(l){15-16}\cmidrule(l){17-18}\cmidrule(l){19-20}
\multirow{-3}{*}{Noise $\sigma$}  &
\multirow{-3}{*}{Methods}  &
  AUROC & AP & AUROC & AP & AUROC & AP & AUROC & AP & AUROC & AP & AUROC & AP & AUROC & AP & AUROC & AP & AUROC & AP &
  AUROC ($\uparrow$) & AP ($\uparrow$) & \\ \midrule
\multirow{2}{*}{$\sigma=0$} & F-ConV & 91.76 & 93.00 & 84.37 & 85.94 & 85.67 & 86.82 & 81.43 & 83.13 & 96.74 & 97.01 & 94.10 & 94.92 & 96.08 & 96.53 & 94.88 & 95.47 & 93.97 & 94.55 & 91.00 & 91.93 \\
\cmidrule(l){3-22}
& \textbf{MDMF} & \textbf{92.56} & \textbf{93.57} & \textbf{88.86} & \textbf{90.16} & \textbf{94.63} & \textbf{97.35} & \textbf{88.89} & \textbf{94.48} & \textbf{99.93} & \textbf{99.94} & \textbf{98.99} & \textbf{99.52} & \textbf{98.76} & \textbf{99.41} & \textbf{98.84} & \textbf{99.46} & \textbf{99.40} & \textbf{99.72} & \textbf{95.65} & \textbf{97.07} \\
\midrule
\multirow{2}{*}{$\sigma=0.05$} & F-ConV & 90.55 & 91.89 & 84.97 & 86.72 & 84.19 & 85.60 & 78.30 & 79.90 & 94.12 & 94.40 & 91.99 & 93.00 & 95.04 & 95.38 & 94.53 & 95.05 & 88.43 & 89.47 & 89.12 & 90.16 \\
\cmidrule(l){3-22}
& \textbf{MDMF} & \textbf{90.79} & \textbf{91.94} & \textbf{86.12} & \textbf{87.49} & \textbf{87.42} & \textbf{93.38} & \textbf{82.97} & \textbf{91.09} & \textbf{98.10} & \textbf{98.38} & \textbf{95.29} & \textbf{97.65} & \textbf{97.22} & \textbf{98.65} & \textbf{95.10} & \textbf{97.60} & \textbf{92.71} & \textbf{96.35} & \textbf{91.75} & \textbf{94.73} \\
\midrule
\multirow{2}{*}{$\sigma=0.10$} & F-ConV & 89.85 & 91.20 & 83.72 & 84.05 & 79.79 & 81.11 & 76.87 & 77.49 & 92.57 & 93.39 & 90.76 & 91.39 & 93.85 & 94.01 & 92.10 & 92.48 & 88.08 & 89.23 & 87.51 & 88.26 \\
\cmidrule(l){3-22}
& \textbf{MDMF} & \textbf{90.20} & 91.20 & \textbf{85.02} & \textbf{86.16} & \textbf{84.80} & \textbf{91.69} & \textbf{79.48} & \textbf{88.82} & \textbf{95.53} & \textbf{96.12} & \textbf{92.85} & \textbf{96.31} & \textbf{95.77} & \textbf{97.85} & \textbf{93.34} & \textbf{96.64} & \textbf{89.20} & \textbf{94.36} & \textbf{89.58} & \textbf{93.24} \\
\midrule
\multirow{2}{*}{$\sigma=0.15$} & F-ConV & 88.30 & 88.56 & 80.01 & 79.89 & 77.07 & 77.95 & 68.20 & 69.34 & 89.03 & 88.40 & 85.67 & 85.62 & 90.44 & 89.85 & 89.52 & 88.97 & 81.68 & 81.65 & 83.32 & 83.36 \\
\cmidrule(l){3-22}
& \textbf{MDMF} & \textbf{89.29} & \textbf{90.08} & \textbf{83.57} & \textbf{84.29} & \textbf{82.40} & \textbf{90.01} & \textbf{76.35} & \textbf{86.62} & \textbf{93.06} & \textbf{93.75} & \textbf{90.30} & \textbf{94.77} & \textbf{93.97} & \textbf{96.79} & \textbf{91.42} & \textbf{95.48} & \textbf{86.15} & \textbf{92.44} & \textbf{87.39} & \textbf{91.58} \\
\midrule
\multirow{2}{*}{$\sigma=0.20$} & F-ConV & 86.16 & 85.11 & 77.80 & 77.34 & 71.74 & 71.14 & 63.43 & 62.87 & 84.32 & 83.99 & 78.39 & 77.08 & 85.05 & 83.32 & 85.35 & 85.05 & 74.77 & 74.21 & 78.56 & 77.79 \\
\cmidrule(l){3-22}
& \textbf{MDMF} & \textbf{88.06} & \textbf{88.64} & \textbf{81.63} & \textbf{81.98} & \textbf{79.71} & \textbf{88.08} & \textbf{73.33} & \textbf{84.39} & \textbf{90.25} & \textbf{90.83} & \textbf{87.23} & \textbf{92.83} & \textbf{91.66} & \textbf{95.37} & \textbf{88.91} & \textbf{93.89} & \textbf{82.77} & \textbf{90.22} & \textbf{84.84} & \textbf{89.58} \\
\midrule
\multirow{2}{*}{$\sigma=0.25$} & F-ConV & 80.00 & 78.91 & 71.90 & 69.59 & 65.82 & 65.02 & 60.61 & 60.19 & 75.65 & 74.29 & 71.32 & 69.67 & 77.21 & 75.38 & 76.94 & 77.04 & 66.94 & 65.49 & 71.82 & 70.62 \\
\cmidrule(l){3-22}
& \textbf{MDMF} & \textbf{86.37} & \textbf{86.48} & \textbf{79.12} & \textbf{78.72} & \textbf{76.52} & \textbf{85.61} & \textbf{70.09} & \textbf{81.90} & \textbf{86.76} & \textbf{86.84} & \textbf{83.61} & \textbf{90.32} & \textbf{88.75} & \textbf{93.38} & \textbf{85.65} & \textbf{91.61} & \textbf{78.97} & \textbf{87.47} & \textbf{81.76} & \textbf{86.93} \\
\bottomrule
\end{tabular}
}
\end{table*}

\subsection{Full Results for the Comparison with Patch-Level Hard Voting}
\label{app:sec:voting}

\noindent\textbf{Setting and motivation.} Hard voting is the most natural alternative to MDMF's distributional aggregation: instead of estimating a population-level discrepancy, classify each patch independently and combine the resulting per-patch decisions into an image-level prediction. To make the comparison as informative as possible, we share \emph{everything} between voting and MDMF except the aggregation step, so the gap reflects the aggregation strategy alone. Concretely, both methods use the same DINOv2 ViT-L/14 backbone, the same patch tokenization that yields $K{=}49$ patches at granularity $W{=}32$, and the same 4-class ProGAN training data. The voting variant is implemented by training a lightweight per-patch binary classifier (a single linear head on top of the frozen DINOv2 patch features) under the standard BCE loss using the same ProGAN training pairs that MDMF consumes.

A central feature of hard voting that distinguishes it from MDMF is its dependence on \emph{two} coupled thresholds at test time. (i) A per-patch sigmoid cutoff $\theta_{\mathrm{patch}}\in[0,1]$, which converts each per-patch fake probability into a hard 0/1 decision. (ii) An image-level decision threshold $\tau$ applied to the resulting fake-patch ratio $r(I)=\frac{1}{K}\sum_{k}\mathbb{1}[\sigma(z_k)>\theta_{\mathrm{patch}}]$. The detector flags the image as fake when $r(I)>\tau$. In contrast, MDMF requires only the single image-level threshold $\tau$: per-patch decisions are never materialized, and forensic evidence is aggregated continuously through the MMD score over PFS embeddings. The presence of the additional $\theta_{\mathrm{patch}}$ in the voting pipeline introduces a non-trivial deployment burden, because the optimal $\theta_{\mathrm{patch}}$ is generator-dependent and must be tuned on a held-out set; a poor choice collapses the per-patch decision into either ``always fake'' or ``always real,'' destroying the patch-level signal regardless of how well $\tau$ is calibrated.

To isolate the AUROC induced purely by the choice of $\theta_{\mathrm{patch}}$ (and not by a particular $\tau$), we follow the threshold-free protocol used throughout this paper: for each $\theta_{\mathrm{patch}}$, we treat the fake-patch ratio $r(I)$ as a continuous score and sweep $\tau$ to obtain a ROC curve, then report AUROC and AP. This is the same metric used for MDMF (with the MMD score replacing $r(I)$), and it gives every voting configuration its best image-level operating point. Figure~\ref{fig:abla}(d) shows the resulting voting AUROC as $\theta_{\mathrm{patch}}$ is swept; here we provide the per-generator full results.

\begin{table*}[h]
\setlength{\tabcolsep}{3pt}
\caption{Detailed detection performance ($\%$) of patch-level hard voting at the eight patch thresholds $\theta_{\mathrm{patch}}\in\{0.03,0.05,0.08,0.10,0.15,0.20,0.25,0.30\}$ swept in Figure~\ref{fig:abla}(d), compared with MDMF on the ImageNet benchmark. Bold numbers indicate the best result on each generator across all configurations.
}
\label{tab:app:voting}
\resizebox{\textwidth}{!}{%
\begin{tabular}{@{}l|ccccccccccccccccccccccc@{}}
\toprule
&
  \multicolumn{2}{c}{ADM} &
  \multicolumn{2}{c}{ADMG} &
  \multicolumn{2}{c}{LDM} &
  \multicolumn{2}{c}{DiT} &
  \multicolumn{2}{c}{BigGAN} &
  \multicolumn{2}{c}{GigaGAN} &
  \multicolumn{2}{c}{StyleGAN XL} &
  \multicolumn{2}{c}{RQ-Transformer} &
  \multicolumn{2}{c}{Mask GIT} &
  \multicolumn{2}{c}{\multirow{-1}{*}{Average}} \\ \cmidrule(l){2-3} \cmidrule(l){4-5}\cmidrule(l){6-7}  \cmidrule(l){8-9}\cmidrule(l){10-11}\cmidrule(l){12-13}\cmidrule(l){14-15}\cmidrule(l){16-17}\cmidrule(l){18-19}
\multirow{-3}{*}{Method}  &
  AUROC & AP & AUROC & AP & AUROC & AP & AUROC & AP & AUROC & AP & AUROC & AP & AUROC & AP & AUROC & AP & AUROC & AP &
  AUROC ($\uparrow$) & AP ($\uparrow$) & \\ \midrule
Voting ($\theta_{\mathrm{patch}}=0.03$) & 85.35 & 78.69 & 83.01 & 76.67 & 86.24 & 88.49 & 82.22 & 86.32 & 88.99 & 81.95 & 88.63 & 89.86 & 88.43 & 89.84 & 88.61 & 89.86 & 88.80 & 89.96 & 86.70 & 85.74 \\
Voting ($\theta_{\mathrm{patch}}=0.05$) & 88.22 & 86.82 & 84.06 & 83.82 & 90.42 & 93.46 & 83.30 & 90.67 & 95.85 & 92.34 & 95.10 & 95.56 & 94.79 & 95.43 & 95.12 & 95.59 & 95.43 & 95.76 & 91.37 & 92.16 \\
Voting ($\theta_{\mathrm{patch}}=0.08$) & 90.94 & 89.71 & 87.43 & 86.21 & 92.24 & 94.64 & 86.69 & 91.62 & 97.71 & 96.50 & 96.64 & 97.11 & 96.25 & 96.93 & 96.68 & 97.16 & 97.08 & 97.36 & 93.52 & 94.14 \\
Voting ($\theta_{\mathrm{patch}}=0.10$) & 90.96 & 90.70 & 86.36 & 86.04 & 92.27 & 95.15 & 85.61 & 91.07 & 98.99 & 98.71 & 97.83 & 98.70 & 97.43 & 98.50 & 97.88 & 98.75 & 98.30 & 98.95 & 93.96 & 95.17 \\
Voting ($\theta_{\mathrm{patch}}=0.15$) & 91.35 & 91.34 & 87.41 & 87.32 & 92.65 & 95.41 & 86.68 & 92.04 & 98.81 & 98.67 & 97.56 & 98.16 & 97.13 & 97.95 & 97.61 & 98.20 & 98.06 & 98.42 & 94.14 & 95.28 \\
Voting ($\theta_{\mathrm{patch}}=0.20$) & 91.14 & 91.21 & 86.50 & 86.45 & 92.51 & 95.62 & 85.87 & 91.58 & 99.89 & 99.84 & 98.45 & 99.11 & 97.97 & 98.85 & 98.51 & 99.16 & 99.04 & 99.39 & 94.43 & 95.69 \\
Voting ($\theta_{\mathrm{patch}}=0.25$) & 88.17 & 88.19 & 82.33 & 82.23 & 89.66 & 93.56 & 82.12 & 88.70 & 99.89 & 99.86 & 97.87 & 98.71 & 97.23 & 98.34 & 97.98 & 98.78 & 98.69 & 99.21 & 92.66 & 94.18 \\
Voting ($\theta_{\mathrm{patch}}=0.30$) & 83.35 & 83.36 & 76.86 & 76.74 & 85.59 & 90.67 & 77.24 & 85.15 & 99.83 & 99.81 & 96.61 & 97.84 & 95.70 & 97.28 & 96.81 & 97.98 & 97.92 & 98.63 & 89.99 & 91.94 \\
\midrule
\textbf{MDMF (Ours)} & \textbf{92.56} & \textbf{93.57} & \textbf{88.86} & \textbf{90.16} & \textbf{94.63} & \textbf{97.35} & \textbf{88.89} & \textbf{94.48} & \textbf{99.93} & \textbf{99.94} & \textbf{98.99} & \textbf{99.52} & \textbf{98.76} & \textbf{99.41} & \textbf{98.84} & \textbf{99.46} & \textbf{99.40} & \textbf{99.72} & \textbf{95.65} & \textbf{97.07} \\
\bottomrule
\end{tabular}
}
\end{table*}

Table~\ref{tab:app:voting} reports the per-generator results across the eight patch thresholds spanning the regime explored in Figure~\ref{fig:abla}(d), from an aggressive low cutoff ($\theta_{\mathrm{patch}}{=}0.03$) at which nearly every patch is flagged fake to a stringent high cutoff ($\theta_{\mathrm{patch}}{=}0.30$) past the AUROC peak. Four observations emerge.
First, MDMF outperforms hard voting on \emph{every} generator under \emph{every} $\theta_{\mathrm{patch}}$ choice in both AUROC and AP. Even at the voting peak ($\theta_{\mathrm{patch}}{=}0.20$), MDMF still leads on each individual generator (e.g., $+1.42$ on ADM, $+2.36$ on ADMG, $+2.12$ on LDM, $+3.02$ on DiT-XL/2), and the per-generator advantage grows monotonically as $\theta_{\mathrm{patch}}$ moves away from the peak in either direction.
Second, the voting AUROC traces a clear single-peaked profile in $\theta_{\mathrm{patch}}$: rising from $86.70$ at $\theta_{\mathrm{patch}}{=}0.03$ through $91.37$ ($0.05$), $93.52$ ($0.08$), $93.96$ ($0.10$), $94.14$ ($0.15$), peaking at $94.43$ ($0.20$), and then decaying to $92.66$ ($0.25$) and $89.99$ ($0.30$). The full sweep shown in Figure~\ref{fig:abla}(d) confirms this single-peak shape and the absence of a flat plateau, indicating that voting performance is genuinely sensitive to the choice of $\theta_{\mathrm{patch}}$ rather than robust within a wide tolerance band.
Third, the diffusion generators (ADM, ADMG, LDM, DiT-XL/2) drive most of this volatility. Across the eight thresholds, their per-generator AUROC swings by $7$--$11$ points (e.g., DiT-XL/2 moves from $82.22$ at $\theta_{\mathrm{patch}}{=}0.03$ to $86.68$ at $0.15$ and back down to $77.24$ at $0.30$, while ADMG ranges from $76.86$ at $0.30$ to $87.43$ at $0.08$), whereas the GAN/AR generators (BigGAN, GigaGAN, StyleGAN-XL, RQ-Transformer, Mask-GIT) stay above $94\%$ AUROC across all but the most extreme low thresholds. This is consistent with the picture in which diffusion artifacts produce per-patch fake probabilities that concentrate near the decision boundary, so a small change in $\theta_{\mathrm{patch}}$ flips a large fraction of patches between fake and real and destabilizes the fake-ratio score; in contrast, MDMF's MMD score integrates evidence continuously across all patches, removing the per-patch decision boundary entirely.
Fourth, even the best voting configuration ($\theta_{\mathrm{patch}}{=}0.20$, AUROC $94.43$) trails MDMF by $+1.22$ AUROC on average, and the gap is most pronounced precisely on the four diffusion generators (collectively a $+2.23$ AUROC advantage at the voting peak, growing to $+10.48$ at $\theta_{\mathrm{patch}}{=}0.30$). Combined with the additional deployment cost of tuning a generator-dependent $\theta_{\mathrm{patch}}$, this confirms that distributional two-sample testing over PFS captures the patch-population signal more reliably than independent per-patch decisions, providing direct empirical support for adopting MMD over hard voting.

\subsection{Failure Case Analysis: Borderline Real Images}
\label{app:sec:failure}

\noindent\textbf{Setting.} Although MDMF reaches state-of-the-art detection performance across six benchmarks, no detector is perfect. To better understand where its residual error budget actually goes, we sort all $50{,}000$ ImageNet validation real images by the MDMF score and inspect the highest-scoring (most ``fake-looking'') tail. Figure~\ref{fig:app:failure-cases} reports four representative cases together with their MDMF heatmaps and per-patch grids. In each case the image is genuinely real (i.e., not produced by any generative model), yet MDMF nonetheless assigns it a high fake-side score.

\begin{figure}[h]
    \centering
    \includegraphics[width=\linewidth]{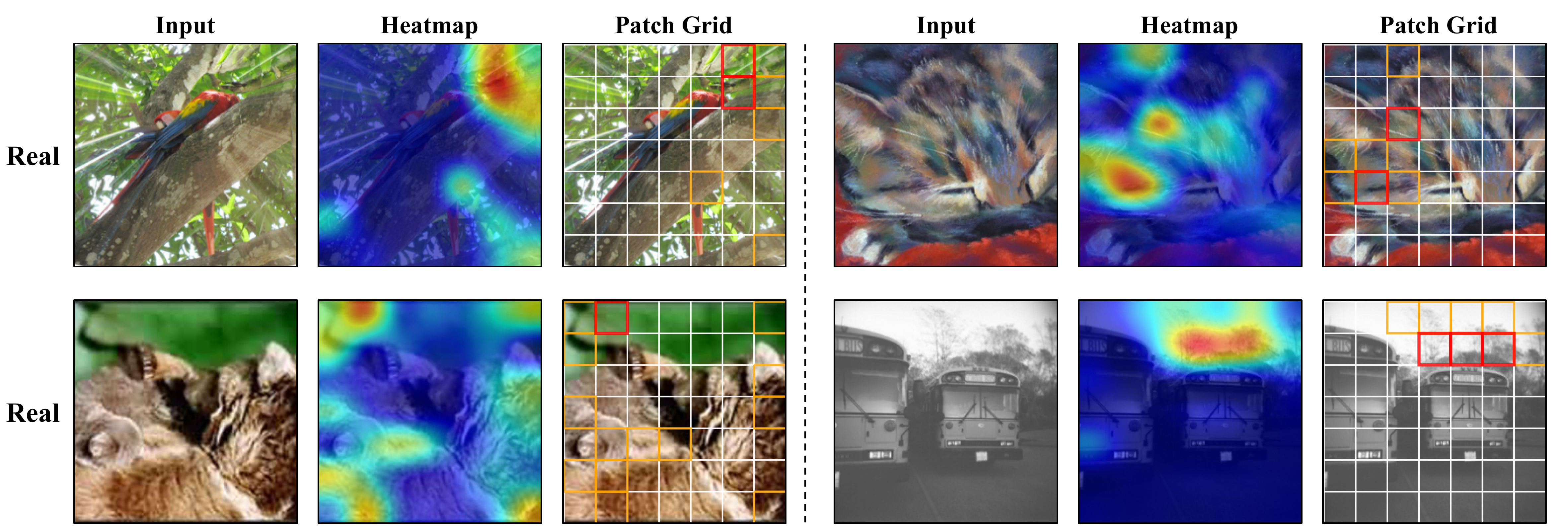}
    \caption{Failure cases on borderline real images from the ImageNet validation set. For each example we show the input, the MDMF heatmap, and the per-patch grid (red/orange = patches with the highest MDMF score). Warmer colors indicate higher
    predicted likelihood of being fake.
    }
    \label{fig:app:failure-cases}
\end{figure}

\noindent\textbf{What is borderline in each case.} The four images cover a small but representative slice of the photographic conditions that push real PFS distributions toward the generated side. (i) The \emph{macaw} (top-left) is captured against a heavily compressed canopy with strong chromatic noise on the leaves and bark; in particular, the upper-right corner contains a sharp branch silhouette against an over-exposed sky, where blocky compression edges replace the natural high-frequency texture present in the reference set. (ii) The \emph{cat} (top-right) is, in fact, an impressionist painting rather than a photograph, so coarse, painterly brushstrokes replace the fine fur texture of typical real photographs and yield unusually smooth patches. (iii) The \emph{brown bear} (bottom-left) is a soft-focus, low-resolution shot in which the fur and surrounding vegetation lose much of the natural high-frequency detail expected of sharp photographs, and the contrast between fur and grass is rendered as diffuse, low-detail patches. (iv) The \emph{school bus} (bottom-right) is a black-and-white film photograph: the absence of color, the visible film grain, and the soft optical blur on the bus body collectively produce patch statistics far from the color photographic prior of the reference bank. None of these images is synthetic, but each carries a strong photographic or stylistic post-processing characteristic that genuinely deviates from the clean color photograph distribution that MDMF's reference bank encodes.

\noindent\textbf{Where MDMF's attention falls and what this implies.} The patch grid in Figure~\ref{fig:app:failure-cases} highlights, in each case, exactly the regions whose local statistics deviate most from the reference distribution: the bright canopy-and-branch corner of the macaw, the central face of the painted cat, the high-contrast fur--vegetation boundary on the bear, and the textureless top of the bus body. In other words, when MDMF errs on these borderline reals it does so by faithfully detecting the same kind of local distributional shift that defines its operating principle---only here the shift is induced by photographic conditions (compression, painting, defocus, monochrome film) rather than by a generative model. We therefore view this as an interpretable, design-consistent failure mode: the score reflects \emph{how unusual the local distribution looks relative to clean real references}, regardless of whether the underlying cause is a generator or a real-world post-processing artifact. Two practical implications follow. First, the per-patch grid provides interpretable evidence: a misclassification can be traced to a specific image region rather than treated as a black-box error, an asset for downstream auditing as discussed in Appendix~\ref{app:sec:impact}. Second, this failure mode supports combining MDMF with complementary signals (e.g., provenance metadata or watermarking) in deployment, in line with the broader recommendation in Appendix~\ref{app:sec:limitations}.

\section{Detailed Visualizations}
\label{app:sec:visual}

In the main paper, we present qualitative visualizations on ADM (Figure~\ref{fig:exp:qualitative}/~\ref{fig:app:qualitative-1}).
Here we further report results on images generated by ADMG, LDM, and DiT (Figure~\ref{fig:app:qualitative-2}--\ref{fig:app:qualitative-4}).
Across these generators, we observe a consistent trend: the global pooling baseline mainly attends to semantically salient regions (e.g., object contours and high-contrast textures) with similar patterns on real and generated samples, suggesting limited sensitivity to sparse, localized artifacts.
In contrast, MDMF produces more localized activations on generated images and comparatively diffuse responses on real images, indicating that patch-wise PFS evidence induces a stronger distributional discrepancy that can be leveraged for robust detection.
Overall, these cross-model visualizations support the generalization of MDMF and corroborate our claim that localized forensic cues are suppressed by semantic-dominant global features but become salient under PFS-based distributional modeling.

\begin{figure*}[h]
  \small
  \centering
  \includegraphics[width=0.98\linewidth]{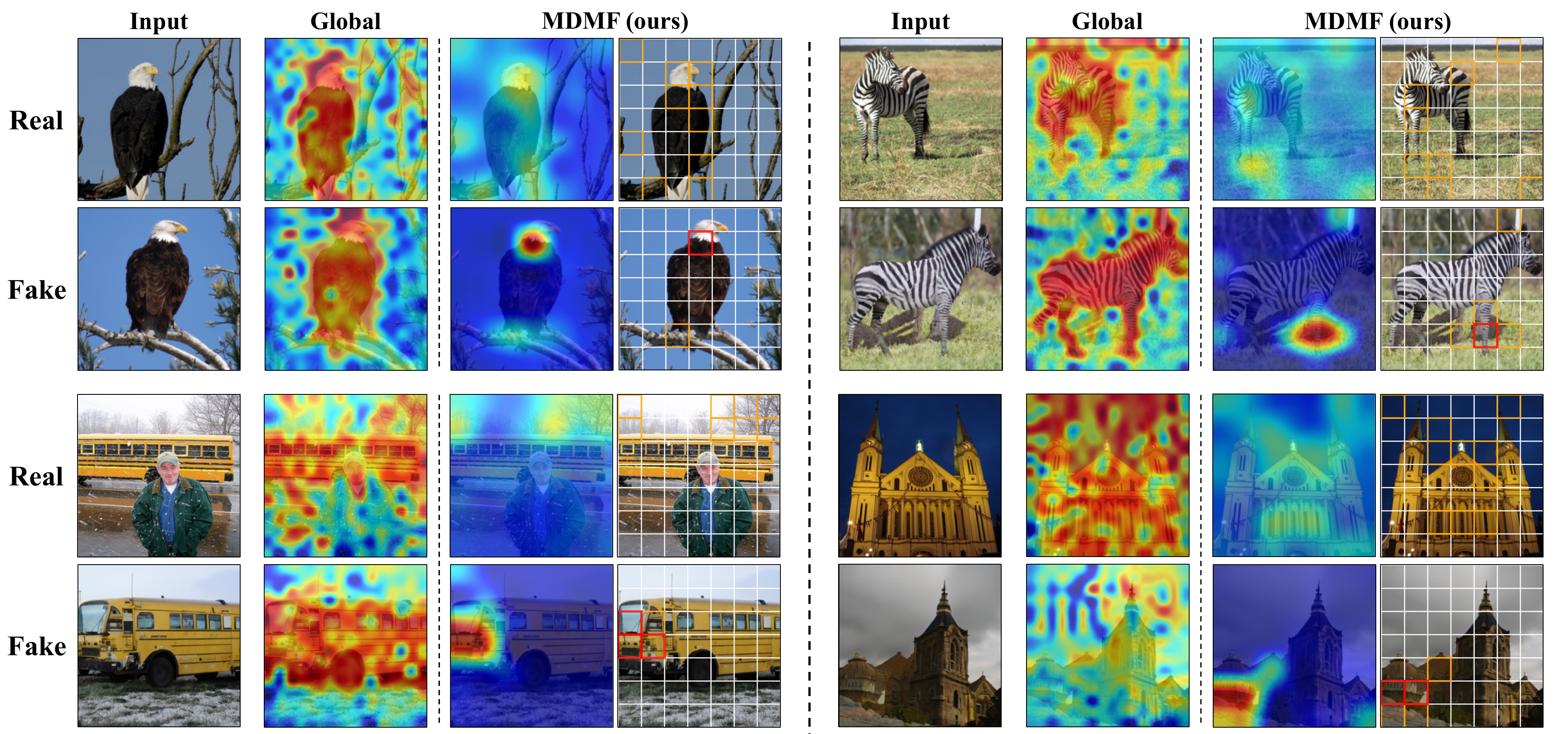}
  \caption{\small{\textbf{Qualitative visualization on ADM.} We compare real images and category-matched generated images, visualizing the responses of a global pooling baseline versus MDMF. Warmer colors indicate higher predicted likelihood of being fake.}} 
  \label{fig:app:qualitative-1}
\end{figure*}

\begin{figure*}[h]
  \small
  \centering
  \includegraphics[width=0.98\linewidth]{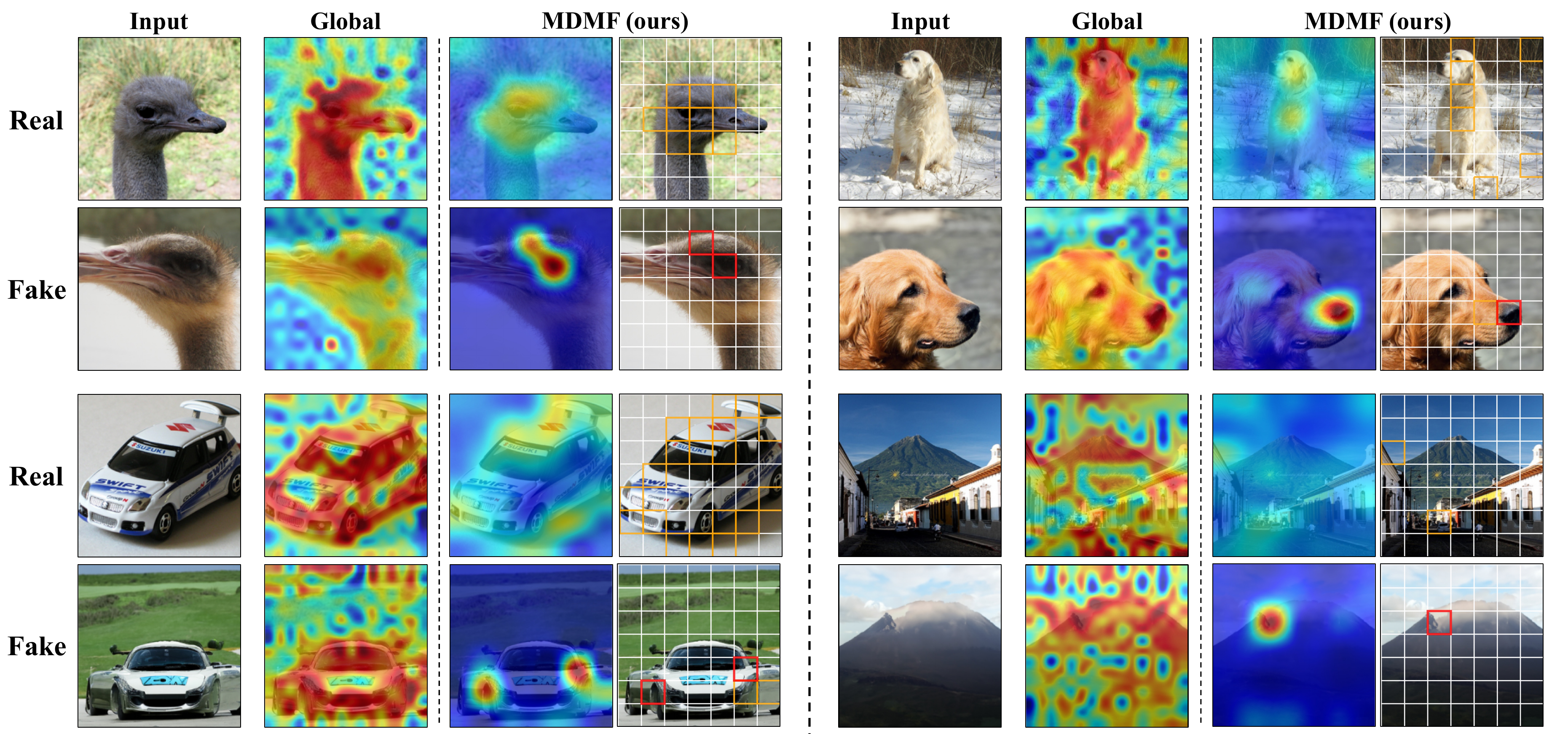}
  \caption{\small{\textbf{Additional qualitative visualization on ADMG.} We compare real images and category-matched generated images, visualizing the responses of a global pooling baseline versus MDMF. Warmer colors indicate higher predicted likelihood of being fake.}} 
  \label{fig:app:qualitative-2}
\end{figure*}

\begin{figure*}[h]
  \small
  \centering
  \includegraphics[width=0.98\linewidth]{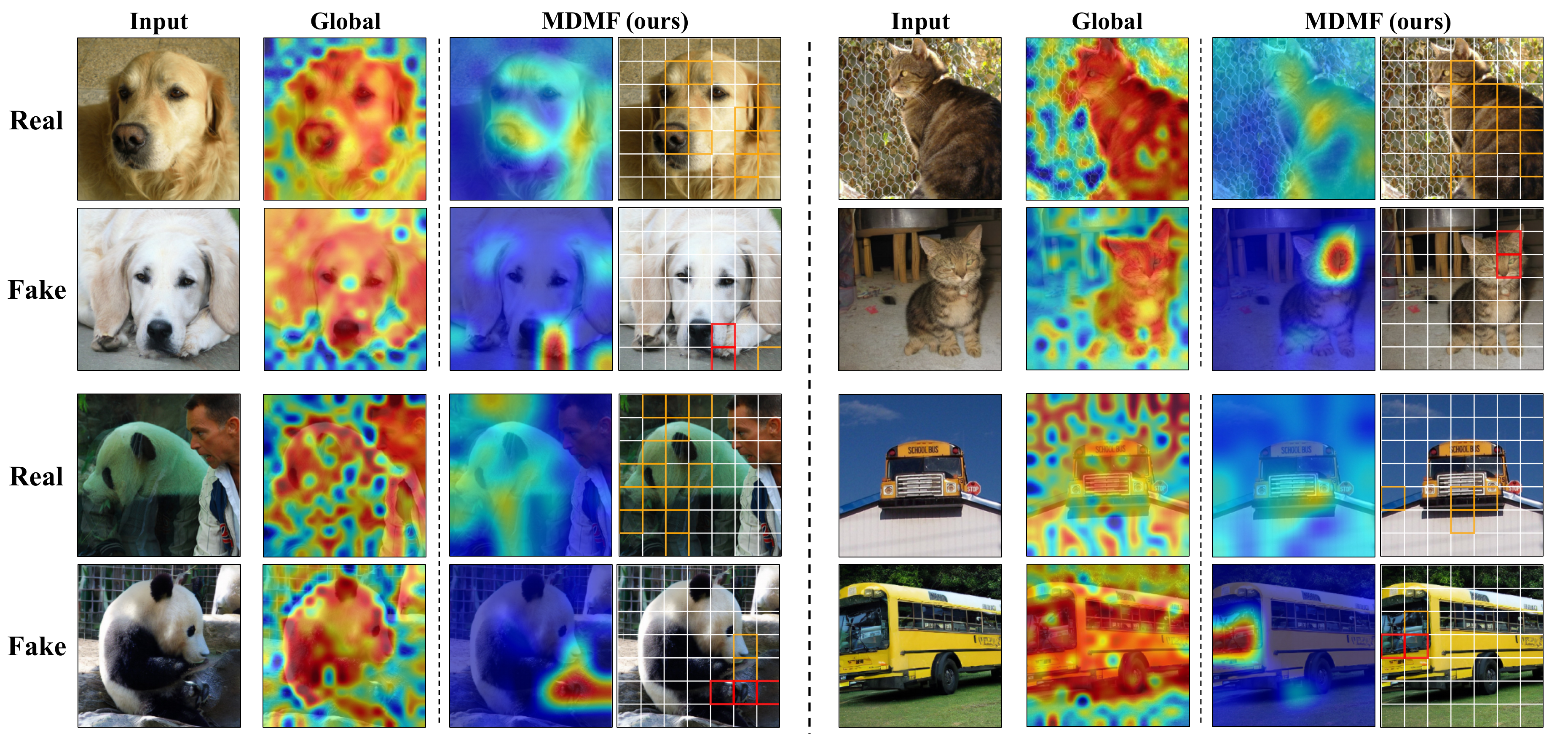}
  \caption{\small{\textbf{Additional qualitative visualization on LDM.} We compare real images and category-matched generated images, visualizing the responses of a global pooling baseline versus MDMF. Warmer colors indicate higher predicted likelihood of being fake.}} 
  \label{fig:app:qualitative-3}
\end{figure*}

\begin{figure*}[h]
  \small
  \centering
  \includegraphics[width=0.98\linewidth]{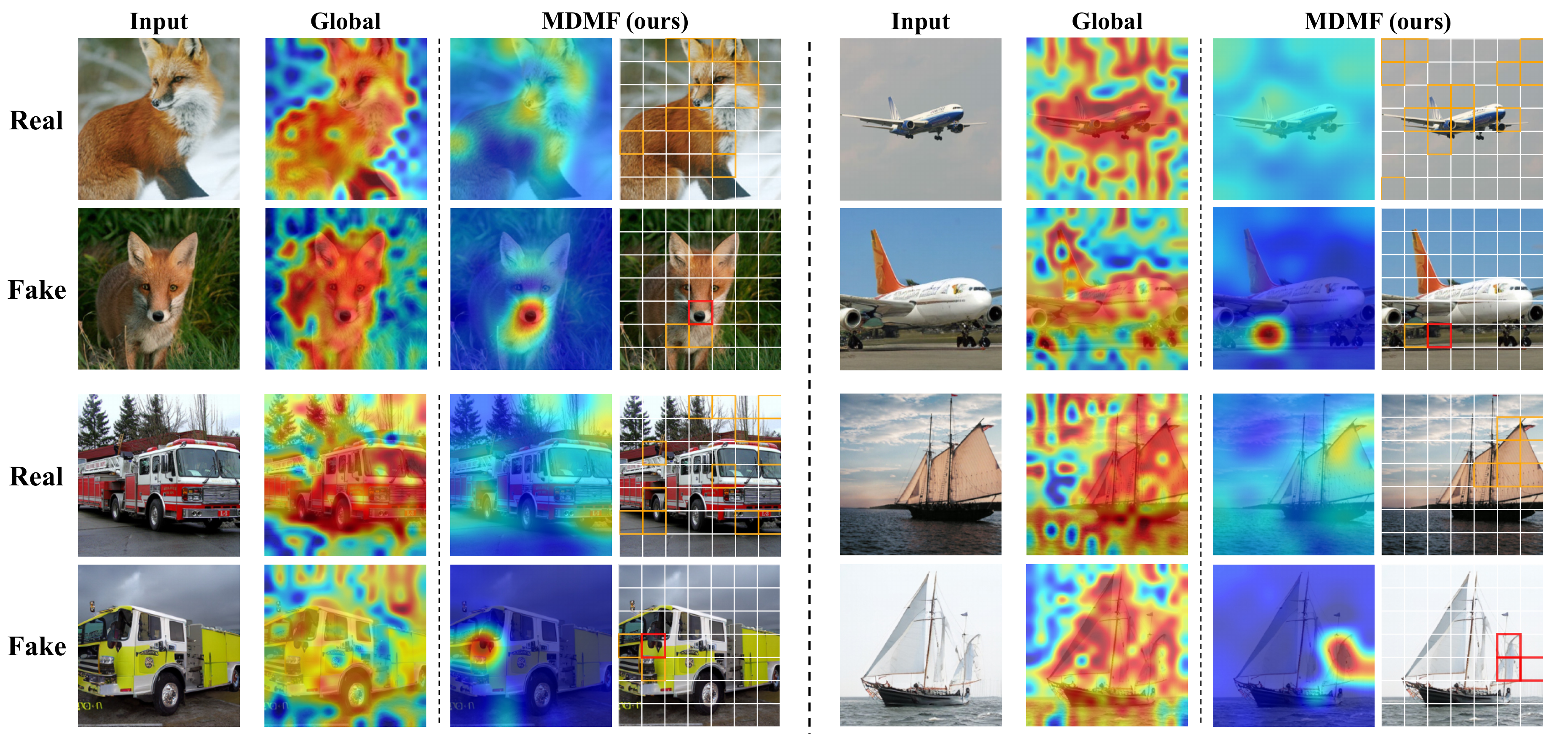}
  \caption{\small{\textbf{Additional qualitative visualization on DiT-XL/2.} We compare real images and category-matched generated images, visualizing the responses of a global pooling baseline versus MDMF. Warmer colors indicate higher predicted likelihood of being fake.}} 
  \label{fig:app:qualitative-4}
\end{figure*}

\clearpage

\section{Limitations and Discussion}
\label{app:sec:limitations}

We acknowledge two practical considerations of MDMF. First, MDMF estimates the MMD against a small reference set of real images, which introduces a lightweight operational dependency relative to fully feed-forward classifiers; we verify in Appendix~\ref{app:sec:ref} that performance is essentially stable from $1$k to $10$k references, and the use of an in-domain real-image reference at inference is shared with recent training-based detectors such as F-ConV~\citep{zhang2025detecting}, but deployment scenarios that require strict standalone inference (no real images available at test time) are out of our current scope. Second, although MDMF maintains higher AUROC and markedly gentler degradation than baselines under JPEG compression, Gaussian blur, and Gaussian noise (Figure~\ref{fig:abla}(c)), all evaluated detectors, including ours, still exhibit a non-trivial drop at the most severe perturbation levels (e.g., Gaussian blur $\sigma{=}5$); inspecting per-image scores, the dominant failure mode is real images carrying strong compression or denoising artifacts whose PFS distributions resemble those of generated samples, which is a shared challenge across forensic detectors rather than specific to MDMF. Despite these considerations, MDMF achieves state-of-the-art performance across six benchmarks (ImageNet, LSUN-Bedroom, GenImage, OpenSora, WildRF, and LDMFakeDetect), retains its advantage over the strongest training-based baselines under encoder-scale variation and post-processing perturbations, and consistently outperforms the best patch-level hard-voting alternative under a dense threshold sweep, indicating that the distributional two-sample-testing perspective offers a robust and principled framework for AI-generated image detection.

\section{Broader Impact}
\label{app:sec:impact}

The rapid proliferation of high-quality AI-generated images has substantially amplified the risk of visual disinformation, identity impersonation, and erosion of trust in digital media. By introducing a principled distributional perspective for AI-generated image detection, MDMF directly contributes to mitigating these harms: a more reliable detector enables platforms, fact-checkers, and end users to flag synthetic content with reduced false-negative rates, supporting the integrity of online discourse, journalism, and forensic investigation. The patch-level forensic signatures we learn also yield interpretable evidence (Figures~\ref{fig:exp:qualitative} and~\ref{fig:app:qualitative-1}--\ref{fig:app:qualitative-4}), which can be inspected and audited, in line with calls for transparent decision-making in AI-driven content moderation. We are aware of two potential negative effects worth noting. First, any forensic detector can become a target of adversarial attack: malicious actors with knowledge of MDMF's PFS-MMD pipeline could attempt to craft perturbations that evade detection; we therefore recommend that production deployments combine MDMF with complementary signals (e.g., provenance metadata or watermarking) and continual updates to keep pace with emerging generators. Second, false positives, i.e., real images mistaken for AI-generated, can adversely affect photographers and artists; deployers should expose calibrated confidence levels and provide a redress mechanism rather than treating MDMF outputs as final verdicts.


\end{document}